\documentclass{article}

\PassOptionsToPackage{comma}{natbib}

 \usepackage[final]{neurips_2022}

\usepackage[utf8]{inputenc} 
\usepackage[T1]{fontenc}    
\usepackage[hidelinks]{hyperref}       
\usepackage{multirow}
\usepackage{url}            
\usepackage{booktabs}       
\usepackage{amsfonts}       
\usepackage{nicefrac}       
\usepackage{microtype}      
\usepackage{xcolor}         

\usepackage{graphicx}       

\usepackage{dsfont}         

\usepackage{caption}

\usepackage{subcaption}
\usepackage{amsmath}
\usepackage{bm} 
\usepackage{amsthm} 
\usepackage{wrapfig}

\theoremstyle{plain}
\newtheorem{proposition}{Proposition}

\title{What Can the Neural Tangent Kernel Tell Us About Adversarial Robustness?}

\author{%
  Nikolaos Tsilivis \\
  Center for Data Science\\
  New York University\\
  \texttt{nt2231@nyu.edu} \\
   \And
   Julia Kempe \\
   Center for Data Science and\\
   Courant Institute of Mathematical Sciences\\
   New York University \\
   \texttt{kempe@nyu.edu} \\
}

\begin{document}

\maketitle

\begin{abstract} 
The adversarial vulnerability of neural nets, and subsequent techniques to create robust models have attracted significant attention; yet we still lack a full understanding of this phenomenon. Here, we study adversarial examples of trained neural networks through analytical tools afforded by recent theory advances connecting neural networks and kernel methods, namely the Neural Tangent Kernel (NTK), following a growing body of work that  leverages the NTK approximation to successfully analyze important
deep learning phenomena and design algorithms for new applications.
We show how NTKs allow to generate adversarial examples in a ``training-free'' fashion, and demonstrate that they transfer to fool their finite-width neural net counterparts in the ``lazy'' regime. We leverage this connection to provide an alternative view on robust and non-robust features, which have been suggested to underlie the adversarial brittleness of neural nets. Specifically, we define and study features induced by the eigendecomposition of the kernel to better understand the role of robust and non-robust features, the reliance on both for standard classification and the robustness-accuracy trade-off. We find that such features are surprisingly consistent across architectures, and that
robust features tend to correspond to the largest eigenvalues of the model, and thus are learned early during training. Our framework allows us to 
identify and visualize non-robust yet useful features.
%
Finally, we shed light on the robustness mechanism underlying adversarial training of neural nets used in practice: quantifying the evolution of the associated empirical NTK, we demonstrate that its dynamics falls much earlier into the ``lazy'' regime and
manifests a much stronger form of the well known bias to prioritize learning features within the top eigenspaces of the kernel, compared to standard training.

\end{abstract}

\section{Introduction}

Despite the tremendous success of deep neural networks in many computer vision and language modeling tasks, as well as in scientific discoveries, their properties and the reasons for their success are still poorly understood. Focusing on computer vision, a particularly surprising phenomenon evidencing that those machines drift away from how humans perform image recognition is the presence of \textit{adversarial examples}, images that are almost identical to the original ones, yet are misclassified by otherwise accurate models. 

Since their discovery \citep{Sze+14}, a vast amount of work has been devoted to understanding the sources of adversarial examples and explanations include, but are not limited to, the close to linear operating mode of neural nets \citep{GSS15}, the curse of dimensionality carried by the input space \citep{GSS15,Gab+19}, insufficient model capacity \citep{Tsi+19,Nakk19} or spurious correlations found in common datasets \citep{Ily+19}. In particular, one widespread 
viewpoint is that adversarial vulnerability is the result of a model's sensitivity to imperceptible yet well-generalizing features in the data, so called {\em useful non-robust} features, giving rise to a trade-off between accuracy and robustness \citep{Tsi+19,Zha+19}.
This gradual understanding has enabled the design of training algorithms, that provide convincing, yet partial, remedies to the problem; the most prominent of them being adversarial training and its many variants \citep{GSS15,Mad+18,robustbench20}. Yet we are far from a mature, unified theory of robustness that is powerful enough to universally guide engineering choices or defense mechanisms.


In this work, we aim to get a deeper understanding of adversarial robustness (or lack thereof) by focusing on the recently established connection of neural networks with kernel machines. Infinitely wide neural networks, trained via gradient descent with infinitesimal learning rate, provably become kernel machines with a data-independent, but architecture dependent kernel - its Neural Tangent Kernel (NTK) -  that remains constant during training \citep{JHG18,Lee+19,Aro+19b,Liu+20}. The analytical tools afforded by the rich theory of kernels have resulted in progress in understanding the optimization landscape and generalization capabilities of neural networks \citep{Du+19b,Aro+19a}, together with the discovery of interesting deep learning phenomena \citep{Fort+20,Jim+21}, while also inspiring practical advances in diverse areas of applications such as the design of better classifiers \citep{Sha+20}, efficient neural architecture search \citep{Chen+21}, low-dimensional tasks in graphics \citep{Tan+20} and dataset distillation \citep{Ngu+21}. While the NTK approximation is increasingly utlilized, even for finite width neural nets, little is known about the adversarial robustness properties of these infinitely wide models.

{\bf Our contribution:} Our work inscribes itself into the quest to 
leverage analytical tools afforded by kernel methods, in particular spectral analysis, to track properties of interest in the associated neural nets, in this case as they pertain to robustness. 
To this end, we first demonstrate that adversarial perturbations generated {\em analytically} with the NTK can successfully lead the associated trained wide neural networks (in the kernel-regime) to misclassify, thus allowing kernels to faithfully predict the lack of robustness of those trained neural networks. In other words, adversarial (non-) robustness transfers from kernels to networks; and adversarial perturbations generated via kernels resemble those generated by the corresponding trained networks. One implication of this transferability is that we can analytically devise adversarial examples that do not require access to the trained model and in particular its weights; instead these ``blind spots'' may be calculated a-priori, before training starts. 

\begin{figure}
    \centering
    \includegraphics[scale=0.09]{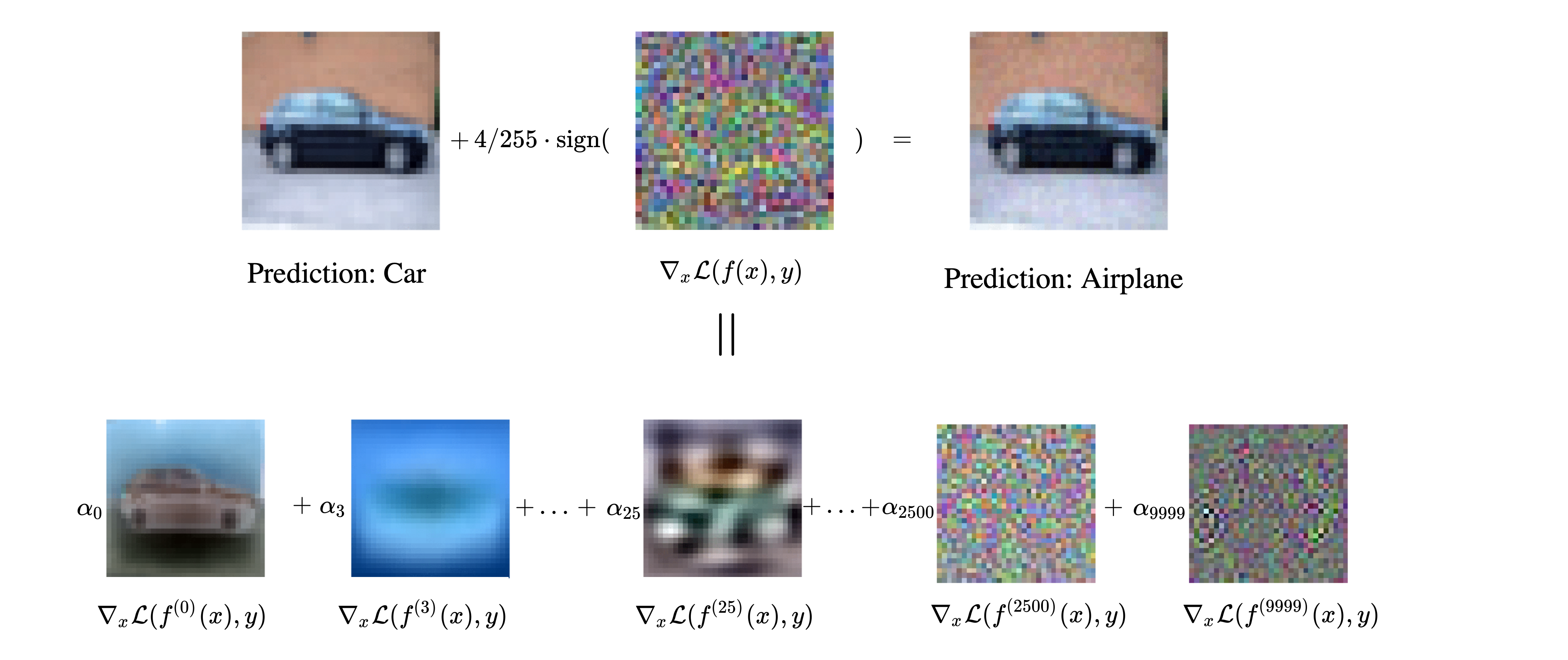}
    \caption{\textbf{Top}. Standard setup of an adversarial attack, where a barely perceivable perturbation 
    is added to an image to confuse an accurate classifier. \textbf{Bottom}. The correspondence between neural networks and kernel machines allows to visualize a decomposition of this perturbation, each part attributed to a different feature of the model. The first few features tend to be {\em robust}.}
    \label{fig:grad_decompose}
\end{figure}

A perhaps even more crucial implication of the NTK approach to robustness relates to the {\em understanding} of adversarial examples. Indeed, we show how the spectrum of the NTK provides an alternative way to define {\em features} of the model, to classify them according to their robustness and usefulness for correct predictions and visually inspect them via their contribution to the adversarial perturbation  (see Fig.~\ref{fig:grad_decompose}). This in turn  allows us to verify previously conjectured properties of standard classifiers; dependence on both \textit{robust} and \textit{non-robust} features in the data \citep{Tsi+19}, and tradeoff of accuracy and robustness during training. In particular we observe that features tend to be rather invariable across architectures, and that robust features tend to correspond to the {\em top} of the eigenspectrum (see Fig.~\ref{fig:features}), and as such are learned first by the corresponding wide nets \citep{Aro+19a,JHG18}. Moreover, we are able to visualize useful non-robust features of standard models (Fig.~\ref{fig:non_rob_feats}).
While this conceptual feature distinction has been highly influential in recent works that study the robustness of deep neural networks (see for example \citep{ZhLi20,KLR21,SMK21}), to the best of our knowledge, none of them has explicitly demonstrated the dependence of networks on such feature functions (except for simple linear models \citep{Goh19}). Rather, these works either reveal such features in some indirect fashion, or accept their existence as an assumption. Here, we show that Neural Tangent Kernel theory endows us with a natural definition of features through its eigen-decomposition and provides a way to \textit{visualise and inspect robust and non-robust features directly} on the function space of trained neural networks. 

\begin{figure}
    \centering
    \includegraphics[scale=0.3]{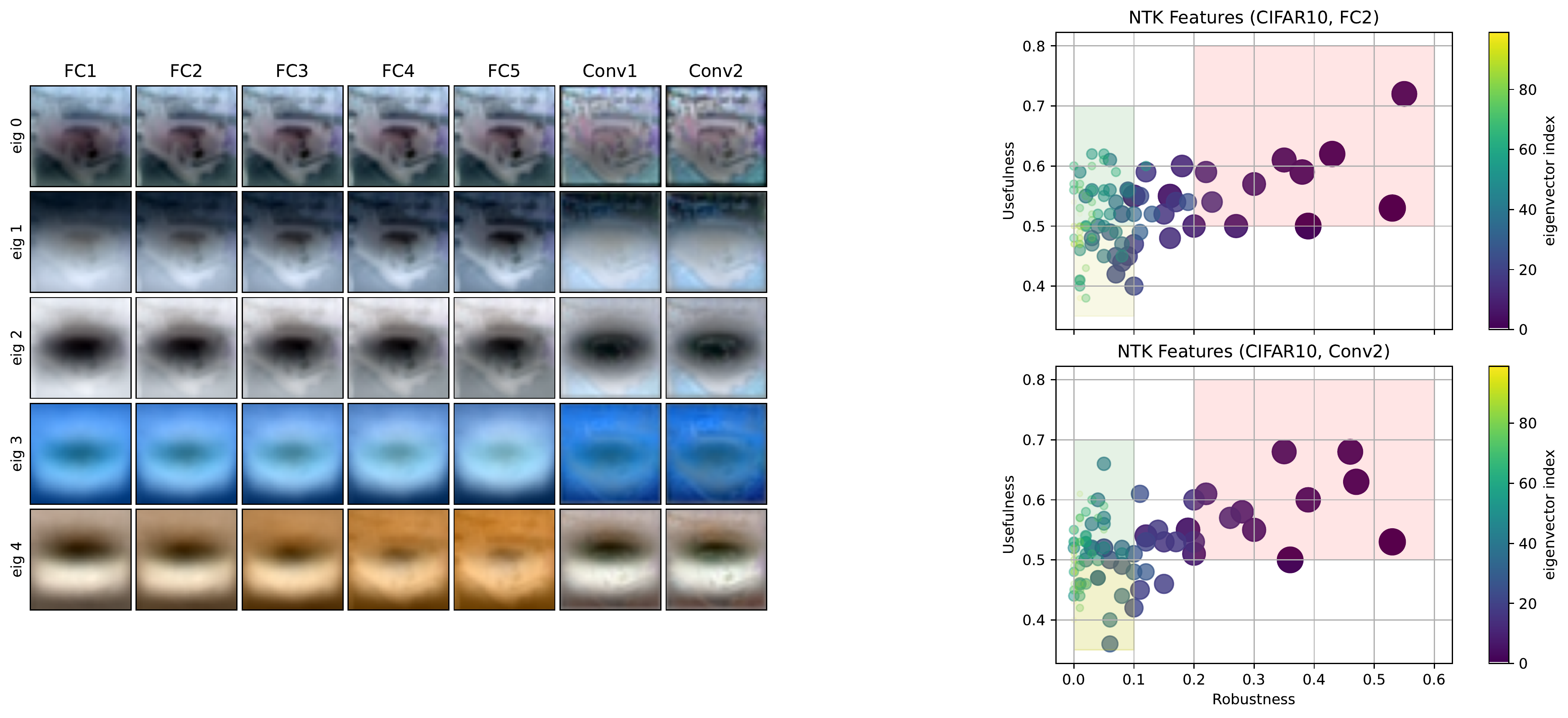}
    \caption{\textbf{Left}: Top 5 features for 7 different kernel architectures for a car image extracted from the CIFAR10 dataset when trained on car and plane images. \textbf{Right}: Features according to their robustness (x-axis) and usefulness (y-axis). 
Larger/darker bullets correspond to larger eigenvalues. {\em Useful} features have $>0.5$-usefulness; shaded boxes are meant to help visualize useful-robust regions.}

    \label{fig:features}
\end{figure}

Interestingly, this connection also enables us to empirically demonstrate that robust features of standard models alone are not enough for robust classification. Aiming to understand, then, what makes robust models robust, we track the {\em evolution} of the data-dependent {\em empirical} NTK during \textit{adversarial training} of neural networks used in practice.
Prior experimental work has found that networks with non-trivial width to depth ratio which are trained with large learning rates, depart from the NTK regime and fall in the so-called ``rich feature'' regime, where the NTK changes substantially during training \citep{Gei+19,Fort+20,Bar+21,Jim+21}. In our work, which to the best of our knowledge is the first to provide insights on how the kernel behaves during adversarial training, we find that the NTK evolves much faster compared to standard training, simultaneously both changing its features and assigning more importance to the more robust ones, giving direct insight into the mechanism at play during adversarial training (see Fig.~\ref{fig:polarrotation}).
In summary, the contributions of our work are the following:
\begin{itemize}
    \item We discuss how to generate adversarial examples for infinitely-wide neural networks via the NTK, and show that they transfer to fool their associated (finite width) nets in the appropriate regime, yielding a "training-free" attack without need to access model weights (Sec. \ref{blackbox_attack}).
    \item Using the spectrum of the NTK, we give an alternative definition of features, 
    providing a natural decomposition or perturbations into robust and non-robust parts \citep{Tsi+19,Ily+19} (Fig.~\ref{fig:grad_decompose}). We confirm that robust features overwhelmingly correspond to the top part of the eigenspectrum; hence they are learned early on in training. 
    We bolster previously conjectured hypotheses that prediction relies on both robust and non-robust features and that robustness is traded for accuracy during standard training. Further, we show that only utilizing the robust features of standard models is not sufficient for robust classification (Sec. \ref{ntk_feats}).
    
    
    \item We turn to finite-width neural nets with standard parameters to study the {\em dynamics} of their empirical NTK during \textit{adversarial training}. We show that the kernel rotates in a way that enables both new (robust) feature learning and that drastically increases of the importance (relative weight) of the robust features over the non-robust ones. We further highlight the structural differences of the kernel change during adversarial training versus standard training and observe that the kernel seems to enter the ``lazy'' regime much faster (Sec. \ref{sec_dynamics}). 
\end{itemize}


Collectively, our findings may help explain many phenomena present in the adversarial ML literature and further elucidate both the vulnerability of standard models and the robustness of adversarially trained ones. We provide code to visualize features induced by kernels, giving a unique and principled way to inspect features induced by standardly trained nets (available at \url{https://github.com/Tsili42/adv-ntk}).

{\bf Related work:} To the best of our knowledge the only prior work  that leverages NTK theory to derive perturbations in some adversarial setting is due to \cite{YuWu21}, yet with entirely different focus. It deals with what is coined {\em generalization attacks}: the process of altering the training data distribution to prevent models to generalise on clean data. 
\cite{Bai+21} study aspects of robust models through their linearized sub-networks, but do not leverage NTKs. 

\section{Preliminaries}
\label{prelim}

We introduce background material  and definitions important to our analysis. Here, we restrict ourselves to binary classification, to keep notation light. We defer the multiclass case, complete definitions and a more detailed discussion of prior work to the Appendix. 

\subsection{Adversarial Examples} Let $f$ be a classifier, $\mathbf{x}$ be an input (e.g. a natural image) and $y$ its label (e.g. the image class). Then, given that $f$ is an accurate classifier on $\mathbf{x}$, $\Tilde{\mathbf{x}}$ is an adversarial example \citep{Sze+14} for $f$ if 
\begin{enumerate}
\item[i)] the distance $ d(\mathbf{x},\tilde{\mathbf{x}})$ is small. Common choices in computer vision are the $\ell_p$ norms, especially the $\ell_\infty$ norm on which we focus henceforth, and 
\item[ii)] $f(\tilde{\mathbf{x}}) \neq y$. That is, the perturbed input is being misclassified.
\end{enumerate}
Given a loss function $\mathcal{L}$, such as cross-entropy, one can construct an adversarial example $\Tilde{\mathbf{x}} = \mathbf{x} + \bm{\eta}$ by finding the perturbation $\bm{\eta}$ that produces the maximal increase of the loss, solving
\begin{equation}\label{eq:adv1}
    \bm{\eta} = \arg \max_{\| \bm{\eta} \|_\infty \leq \epsilon} \mathcal{L}(f(\mathbf{x} + \bm{\eta}), y),
\end{equation}
for some $\epsilon > 0$ that quantifies the dissimilarity between the two examples. In general, this is a non-convex problem and one can resort to first order methods  \citep{GSS15}
\begin{equation}\label{eq:fgsm}
    \tilde{\mathbf{x}} = \mathbf{x} + \epsilon \cdot \mathrm{sign} \left( \nabla_{\mathbf{x}} \mathcal{L}(f(\mathbf{x}), y) \right),
\end{equation}
or iterative versions for solving it \citep{KGB17,Mad+18}.
The former method is usually called \textit{Fast Gradient Sign Method (FGSM)} and the latter \textit{Projected Gradient Descent (PGD)}. These methods are able to produce examples that are being misclassified by common neural networks with a probability that approaches 1 \citep{CaWa17}. Even more surprisingly, it has been observed that adversarial examples crafted to ``fool'' one machine learning model are consistently capable of ``fooling'' others \citep{PMG16,Pap+17}, a phenomenon that is known as the \textit{transferability} of adversarial examples. Finally, {\em adversarial training} refers to the alteration of the training procedure to include adversarial samples for teaching the model to be robust \citep{GSS15,Mad+18} and empirically holds as the strongest defense against adversarial examples \citep{Mad+18,Zha+19}.

\subsection{Robust and Non-Robust Features}
\label{ssec:feats}
Despite a vast amount of research, the reasons behind the existence of adversarial examples are not perfectly clear. A line of work has argued that a central reason is the presence of robust and non-robust features in the data that standard models learn to rely upon \citep{Tsi+19,Ily+19}. In particular it is conjectured that reliance on
{\em useful but non-robust} features during training is responsible for the brittleness of neural nets. Here, we slightly adapt the feature definitions of \citep{Ily+19}\footnote{We distinguish useful and robust features based on their accuracy as classifiers, not in terms of correlation with the labels as in \cite{Ily+19}, allowing a natural extension to the multi-class setting. For robustness, we consider any  accuracy bounded away from zero as robust, quantifying that an adversary cannot drive accuracy to zero entirely.}, and extend them to multi-class problems (see Appendix \ref{App:feats}).

Let $\mathcal{D}$ be the data generating distribution with $x \in \mathcal{X}$ and $y \in \{\pm 1\}$. We define a \textit{feature} as a function $\phi: \mathcal{X} \to \mathbb{R}$ and distinguish how they perform as classifiers. Fix $\rho, \gamma \geq 0$:

\begin{enumerate}
\item \textbf{$\rho$-Useful} feature:  A feature $\phi$ is called \textit{$\rho$-useful} if 
\begin{equation}\label{eq:useful}
\mathbb{E}_{x, y \sim \mathcal{D}}  \big[\mathds{1}_{\{\textrm{sign}[\phi(x)] = y \}}\big]  = \rho 
\end{equation}

\item \textbf{$\gamma$-Robust} feature: A feature $\phi$ is called \textit{$\gamma$-robust} if it remains useful under any perturbation inside a bounded ``ball'' $\mathcal{B}$, that is if
\begin{equation}\label{eq:robust}
\mathbb{E}_{x, y \sim \mathcal{D}} \big[\inf_{\delta \in \mathcal{B}}\mathds{1}_{\{ \textrm{sign}[\phi(x+\delta)] = y \}}\big] = \gamma 
\end{equation}
\end{enumerate}
In general, a feature adds predictive value if it gives an advantage above guessing the most likely label, i.e. $\rho > \max_{y' \in \{\pm1\}} \mathbb{E}_{x, y \sim \mathcal{D}} [\mathds{1}_{\{y'=y}\}]$, and we will speak of ``useful'' features in this case, omitting the $\rho$. We will call such a feature \textbf{useful, non-robust} if it is useful, but $\gamma$-robust only for $\gamma=0$ or very close to $0$, depending on context.

The vast majority of works imagines features as being induced by the {\em activations} of neurons in the net, most commonly those of the penultimate layer ({\em representation-layer} features), but the previous formal definitions are in no way restricted to activations, and we will show how to exploit them using the eigenspectrum of the NTK.
In particular, in Sec.~\ref{ntk_feats}, we demonstrate that the above framework agrees perfectly with features induced by the eigenspectrum of the NTK of a network, providing a natural way to decompose the predictions of the NTK into such feature functions. In particular we can identify  robust, useful, and, indeed, useful non-robust features.



\subsection{Neural Tangent Kernel}

Let $f: \mathbb{R}^d \to \mathbb{R}$ be a (scalar) neural network with a linear final layer parameterized by a set of weights $\mathbf{w} \in \mathbb{R}^p$ and $\{ \mathcal{X}, \mathcal{Y} \}$ be a dataset of size $n$, with $\mathcal{X} \in \mathbb{R}^{n \times d}$ and $\mathcal{Y} \in \{\pm 1\}^n$. Linearized training methods study the first order approximation
\begin{equation}\label{eq:first_order_NTK}
    f(\mathbf{x}; \mathbf{w}_{t+1}) = f(\mathbf{x}; \mathbf{w}_t) + \nabla_{\mathbf{w}} f(\mathbf{x}; \mathbf{w}_t)^\top (\mathbf{w}_{t+1} - \mathbf{w}_t).
\end{equation}
The network gradient $\nabla_{\mathbf{w}} f(\mathbf{x}; \mathbf{w}_t)$ induces a kernel function $\Theta_t: \mathbb{R}^d \times \mathbb{R}^d \to \mathbb{R}$, usually referred as the \textit{Neural Tangent Kernel (NTK)} of the model
\begin{equation}\label{eq:Gram}
    \Theta_t(\mathbf{x}, \mathbf{x}^\prime) = \nabla_{\mathbf{w}} f(\mathbf{x}; \mathbf{w}_t)^\top \nabla_{\mathbf{w}} f(\mathbf{x}^\prime; \mathbf{w}_t).
\end{equation}
This kernel describes the dynamics with infinitesimal learning rate (gradient flow).
In general, the tangent space spanned by the $\nabla_{\mathbf{w}} f(\mathbf{x}; \mathbf{w}_t)$ twists substantially during training, and learning with the Gram matrix of  Eq.~\eqref{eq:Gram} (empirical NTK) corresponds to training along an intermediate tangent plane.
Remarkably, however, in the infinite width limit with appropriate initialization and low learning rate, it has been shown that $f$ becomes a \textit{linear} function of the parameters \citep{JHG18,Liu+20}, and the NTK remains {\em constant} ($\Theta_t =\Theta_0=:\Theta$). Then, for learning with $\ell_2$ loss the training dynamics of infinitely wide networks admits a closed form solution corresponding to kernel regression \citep{JHG18,Lee+19,Aro+19b}
\begin{equation}\label{eq:kernel_prediction}
    f_t(\mathbf{x}) = \Theta(\mathbf{x}, \mathcal{X})^\top \Theta^{-1}(\mathcal{X}, \mathcal{X}) (I - e^{-\lambda \Theta(\mathcal{X}, \mathcal{X}) t}) \mathcal{Y},
\end{equation}
where $\mathbf{x} \in \mathbb{R}^d$ is any input (training or testing), $t$ denotes the time evolution of gradient descent, $\lambda$ is the (small) learning rate and, slightly abusing notation, $\Theta(\mathcal{X}, \mathcal{X}) \in \mathbb{R}^{n \times n}$ denotes the matrix containing the pairwise training values of the NTK, $\Theta(\mathcal{X}, \mathcal{X})_{ij} = \Theta(\mathbf{x}_i, \mathbf{x}_j)$, and similarly for $\Theta(\mathbf{x}, \mathcal{X}) \in \mathbb{R}^{n}$. To be precise, Eq.~\eqref{eq:kernel_prediction} gives the {\em mean} output of the network using a weight-independent kernel with variance depending on the initialization\footnote{For that reason, in the experiments, we often compare this with the centered prediction of the actual neural network, $f - f_0$, as is commonly done in similar studies \citep{COB19}.}. 



\section{Transfer Results in the Kernel Regime}
\label{blackbox_attack}

In this section, we show how to generate adversarial examples from NTKs and discuss their similarity to the ones generated by the actual networks. Note that for network results, we restrict ourselves to wide networks initialized in the ``lazy'' regime with small learning rates (the ``kernel regime'').

\subsection{Generation of Adversarial Examples for Infinitely Wide Neural Networks}

Adversarial examples arise in the context of {\em classification}, while the NTK learning process is described by a regression as in Eq.~\eqref{eq:kernel_prediction}. The arguably simplest way to align with the framework presented in Eq.~\eqref{eq:adv1} is to treat the outputs of the kernel similar to logits of a neural net, mapping them to a probability distribution via the sigmoid/softmax function and apply cross-entropy loss.

A simple calculation (see Appendix \ref{App:fgsm_ntk}, together with the generalization to the multi-class case) gives:

{\em The optimal one step adversarial example of a scalar, infinitely wide, neural network is given by}
    \begin{equation}\label{eq:ntk-fgsm-binary}
        \begin{split}
            \tilde{\mathbf{x}} & = \mathbf{x} - y \epsilon \cdot \mathrm{sign} \left( \nabla_{\mathbf{x}} f_t(\mathbf{x}) \right),
        \end{split}
    \end{equation}
    for $\| \tilde{\mathbf{x}} - \mathbf{x} \|_\infty \leq \varepsilon$, where $\nabla_{\mathbf{x}} f_t(\mathbf{x}) = \nabla_\mathbf{x} \Theta(\mathbf{x}, \mathcal{X})^\top \Theta^{-1}(\mathcal{X}, \mathcal{X}) (I - e^{-\lambda \Theta(\mathcal{X}, \mathcal{X}) t}) \mathcal{Y}$.

One can conceive other ways to generate adversarial perturbations for the kernel, either by changing the loss function (as previously done in neural networks (e.g. \citep{CaWa17})) or through a Taylor expansion around the test input, and we present such alternative derivations in Appendix \ref{App:fgsm_ntk}. However, in practice we observe little difference between that approach and the one presented here.




%

\subsection{Transfer Results and Kernel Attacks}

Predictions from NTK theory for infinitely wide neural networks have been used successfully for their large finite width counterparts, so it seems reasonable to conjecture that adversarial perturbations generated via the kernel as in Eq.~\eqref{eq:ntk-fgsm-binary} resemble those directly computed for the corresponding neural net as per Eq.~\eqref{eq:fgsm}. In particular, this would imply that adversarial perturbations derived from the NTK should not only fool the kernel machine itself, but also lead wide neural nets to misclassify.
\begin{wrapfigure}{r}{0.3\textwidth}
    \centering
    \includegraphics[width=0.3\textwidth]{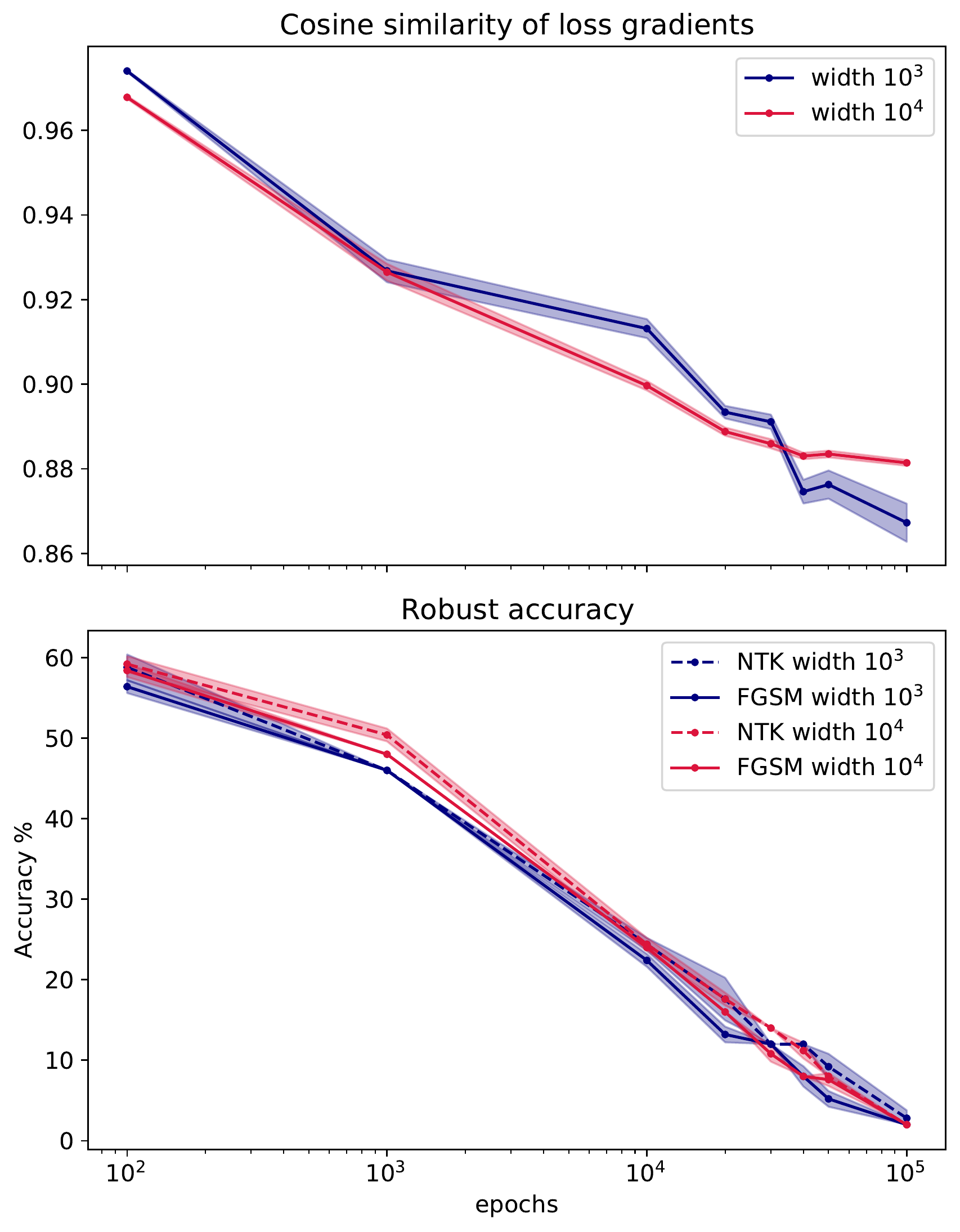}
    \caption{\textbf{Top}. Cosine similarity between the loss gradient of the neural net and of the NTK prediction for the same time point. \textbf{Bottom}. Robust accuracy of neural net against its own adversarial examples (solid) and corresponding NTK examples (dashed). CIFAR10, car vs plane.}
    \label{fig:finite_width_gradients}
\end{wrapfigure}
While similar transfer results in different contexts have been observed indirectly, via the {\em effects} of the perturbation on metrics like accuracy \citep{YuWu21,Ngu+21}, we aim to look deeper to compare perturbations {\em directly}.  High similarity would imply that {\em any} gradient based white-box attack on the neural net can be successfully mimicked by a ``black-box'' kernel derived attack.
 

\textbf{Setting}. To this end, we train multiple two-layer neural networks on image classifications tasks extracted from MNIST and CIFAR-10 and compare adversarial examples generated by Eqs.~\eqref{eq:fgsm} (attacking the neural network) and \eqref{eq:ntk-fgsm-binary} (attacking the kernel). The networks are trained with small learning rate and are sufficiently large, so lie close to the NTK regime.

We track cosine similarity between the gradients of the loss from the NTK predictions and the gradients from the actual neural net as training evolves. Then, we generate adversarial perturbations from both the neural net and the kernel machine, and test whether those produced by the latter can fool the former. Full experimental details can be found in Appendix \ref{App:bbox}.

\textbf{Results}. Our experiments confirm a very strong alignment of loss gradients from the neural nets and the NTK across the whole duration of training, as can be seen in Fig.~\ref{fig:finite_width_gradients} (top). Then, as expected, kernel-generated attacks produce a similar drop in accuracy throughout training as the networks ``own'' white-box attacks, eventually driving robust accuracy to  $0\%$, as seen in Fig.~\ref{fig:finite_width_gradients} (bottom). We reproduce these plots for MNIST in Appendix \ref{App:bbox}, leading to similar conclusions.



When concerned with security aspects of neural nets, adversarial attacks are mainly characterised as either \textit{white-box} or \textit{black-box} attacks \citep{Pap+17}. White box attacks assume full access to the neural network and in particular its weights; prominent examples include FGSM/PGD attacks. Black box attacks, on the other hand, can only {\em query} the model to try to infer the loss gradient, either through training separate surrogate models \citep{PMG16} or through carefully crafted input-output pairs fed to the target model \citep{Che+17,Ily+18,And+20}. NTK theory and the experiments of this section suggest a threat model in which the attacker does not require access to the model or its weights, nor training of a substitute model. For fixed architecture and training data, all the information required for the computation of Eq.~\eqref{eq:ntk-fgsm-binary} is available at initialization, making the ``NTK-attack'' akin to a ``training free'' substitution attack, and, at least in the kernel-regime for wide nets considered here, as effective as white-box attacks.

\section{NTK Eigenvectors Induce Robust and Non-Robust Features}
\label{ntk_feats}

This close connection between adversarial perturbations from the kernel and the corresponding neural net gives us the opportunity to bring to bear kernel tools on the study of adversarial robustness and its relation to features in a more direct fashion.
Several recent works leverage properties of the NTK, and specifically its spectrum, to study aspects of approximation and generalization in neural networks \citep{Aro+19a,Bas+19,BiMa19,Bas+20}. Here we show how the spectrum relates to robustness and helps to clarify the notion of robust/non-robust features.

We define {\em features} induced by the eigendecomposition of the Gram matrix $\Theta(\mathcal{X}, \mathcal{X}) = \sum_{i = 1}^n \lambda_i \mathbf{v}_i \mathbf{v}_i^\top$. We will be most interested in the {\em end} of training, when the model has access to all the features it can extract from the training data $\mathcal{X}$. As $t \to \infty$, Eq.~\eqref{eq:kernel_prediction} becomes $f_{\infty} (\mathbf{x}) = \Theta(\mathbf{x}, \mathcal{X})^\top \Theta(\mathcal{X}, \mathcal{X})^{-1} \mathcal{Y}$ and can be decomposed as $f_{\infty}(\mathbf{x}) = \Theta(\mathbf{x}, \mathcal{X})^\top \sum_{i = 1}^n \lambda_i^{-1} \mathbf{v}_i \mathbf{v}_i^\top \mathcal{Y} = \sum_{i = 1}^n f^{(i)} (\mathbf{x})$, where
\begin{equation}\label{eq:ntkfeat_definition}
    f^{(i)}: \mathbb{R}^d \to \mathbb{R}^k, \; f^{(i)} (\mathbf{x}) := \lambda_i^{-1} \Theta(\mathbf{x}, \mathcal{X})^\top \mathbf{v}_i \mathbf{v}_i^\top \mathcal{Y}.
\end{equation}


Each $f^{(i)}$ can be seen as a {\em unique feature} captured from the (training) data. Note that these functions map the input to the output space, thus matching the definitions of Sec.~\ref{ssec:feats}.
Also observe that all $f^{(i)}$'s jointly recover the original prediction of the model, while each one, intuitively, should contribute something different to it.

Importantly, these features induce a decomposition of the gradient of the loss into parts, each representing gradients of a unique feature as already advertised in Fig.~\ref{fig:grad_decompose}. The binary case is particularly elegant as it gives rise to a linear decomposition of the gradient as 
    \begin{equation}
        \nabla_\mathbf{x} \mathcal{L} (f_{\infty}(\mathbf{x}), y) = \sum_{i = 1}^n \alpha_i \nabla_\mathbf{x} \mathcal{L} (f^{(i)}(\mathbf{x}), y),
    \end{equation}
for some $\alpha_i$ depending on $\mathbf{x}$ and $y$ (see Appendix \ref{App:ntk_feats}).
But if $f^{(i)}$'s are features, how do they look like?

\paragraph{Feature properties of common architectures:}


\begin{wrapfigure}{r}{0.35\textwidth}
    \centering
    \includegraphics[width=0.35\textwidth]{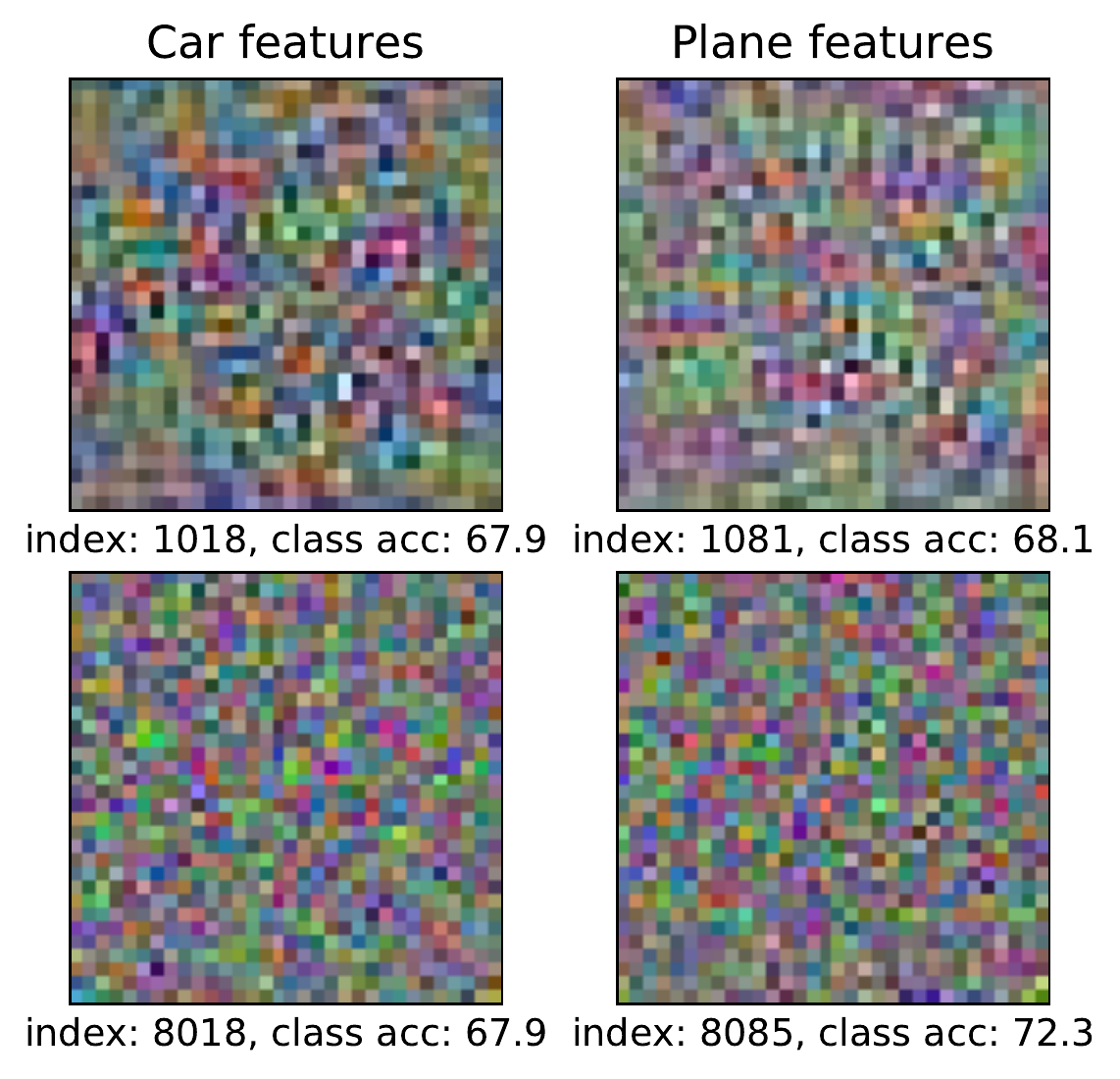}
    \caption{\label{fig:non_rob_feats}Non-robust, useful features earlier and later in the spectrum, for CIFAR10 car and plane.}
\end{wrapfigure}

With these definitions in place, we can now analyze the characteristics of features for commonly used architectures, leveraging their associated NTK. To be consistent with the previous section, we consider classification problems from MNIST (10 classes) and CIFAR-10 (car vs airplane). We compose the Gram matrices from the whole training dataset (50000 and 10000, respectively), and compute the different feature functions $f^{(i)}$ using the eigendecomposition of the matrix. We estimate the \textbf{usefulness} of a feature $f^{(i)}$ by measuring its accuracy on a hold-out validation set, and its \textbf{robustness} by perturbing each input of this set, using an FGSM attack on feature $f^{(i)}$. We consider several different Fully Connected and Convolutional Kernels, whose expressions are available through the Neural Tangents library \citep{Nov+20}, built on top of JAX \citep{Brad+18}. We summarize our findings on how these features behave:

    {\em Functions $f^{(i)}$ represent visually distinct features.} We visualise each feature $f^{(i)}$ by plotting its gradient with respect to $\mathbf{x}$. Fig.~\ref{fig:features} shows the gradient of the first 5 features for various architectures for a specific image from the CIFAR-10 dataset. We observe that features are fairly consistent across models, and they are interpretable: for example the 4th feature seems to represent the dominant color of an image, while the 5th one seems to be capturing horizontal edges.
    
    {\em Networks use both robust and non-robust features for prediction.} It has been speculated that neural networks trained in a standard (non adversarial) fashion rely on both robust and non-robust features. Our feature definition in Eq.~\eqref{eq:ntkfeat_definition} shows that this is indeed the case. The NTK of common neural networks consists of both robust features that match human expectations, such as the ones depicted in Fig.~\ref{fig:features}, but also on features that are predictive of the true label, while not being robust to adversarial perturbations of the input (Fig.~\ref{fig:non_rob_feats}). Fig.~\ref{fig:features} depicts the first 100 features of a fully connected and a convolutional tangent kernel in Usefulness-Robustness space. The upper left region of the plots shows a large amount of useful, yet non-robust features. These features seem random to human observers.
    
    {\em Robustness lies at the top.} We observe in Fig.~\ref{fig:features} that features corresponding to the top eigenvectors tend to be robust. This is consistent among different models and between the two datasets (see Appendix \ref{App:ntk_feats}). Since these eigenvectors are the ones fitted first during training \citep{Aro+19a,JHG18}, it is no wonder that the loss gradient evolves from coherence to noise, as observed in Fig.~\ref{fig:grads}. This also explains the apparent trade-off between robustness and accuracy of neural networks as training progresses: useful, robust features are fitted first, followed by useful, but non-robust ones. This ties in well with both empirical findings \citep{Rah+19} and theoretical case studies \citep{Bas+19,BiMa19,Bas+20} that demonstrate that low frequency \textit{functions} are fitted first during training and provide favorable generalization properties and we would associate robust features with these low-frequency  parts (in function space).
    
    {\em Robust features alone are not enough. } In light of these findings, it might be reasonable to conjecture that
we could obtain robust models by retaining the robust features of the prediction, while discarding the non-robust ones. The spectral approach gives a principled way to disentangle features and create kernel machines keeping only the robust ones. Our results show that in general it is not possible to obtain non-trivial performance without compromising robustness in this fashion, strengthening the case for the necessity of data augmentation in the form of adversarial training (see Appendix \ref{ssec:notenough}).

\section{Kernel Dynamics during Adversarial Training}\label{sec_dynamics}


Given the apparent necessity for adversarial training to produce robust models, how does it achieve this goal? To shed some light on this fundamental question, we depart from the ``lazy'' NTK regime and study the evolution of the NTK of adversarially trained models.
For a neural network trained with gradient descent, as the learning rate $\eta \rightarrow 0$, the continuous time dynamics can be written as
\begin{equation}
    \frac{\partial w}{\partial t} = - \eta \nabla_w \mathcal{L} = -\eta \nabla_w f^\top \frac{\partial \mathcal{L}}{\partial f} \,\,\,\,\,\textnormal{and}\,\,\,\,\, \frac{\partial f}{\partial t} = - \eta \underbrace{\nabla_w f \nabla_w f^\top}_{\Theta_t} \frac{\partial \mathcal{L}}{\partial f}.
\end{equation}
In the NTK regime, this kernel $\Theta_t$ remains fixed at its initial value. However, outside this regime, it has been demonstrated, both empirically \citep{Gei+19,Fort+20,Bar+21,Jim+21} and theoretically \citep{ABP21}, that $\Theta_t$ is not constant during training, and is changing as the weights move. In adversarial training, moreover, there is the additional effect that at each weight update, the data changes as well. For that reason, understanding the dynamics of adversarial training requires tracking the evolution of a kernel 
    $\Theta_t (\mathcal{X}_t, \mathcal{X}_t)$,
where $\mathcal{X}_t$ denotes the current (mini) batch of training data. Notice that the tangent vector $\nabla_w f(\mathcal{X}_t)$ is still describing the instantaneous change of $f$ on the current batch of data, thus $\Theta_t (\mathcal{X}_t, \mathcal{X}_t)$ is informative of the local geometry of the function space, justifying its value as a quantity to be measured during adversarial training.

We train a deep convolutional architecture on CIFAR-10 (multiclass) with standard (SGD) and adversarial training using PGD with an $\ell_\infty$ constraint. Full implementations details and accuracy curves can be found in Appendix \ref{App:empirical_ntk}, together with the reproduction of the same experiment on MNIST, where the observations are similar. We track the following quantities during training:

\begin{figure}
\begin{subfigure}[b]{0.47\textwidth}
    \centering
    \includegraphics[width=0.99\textwidth]{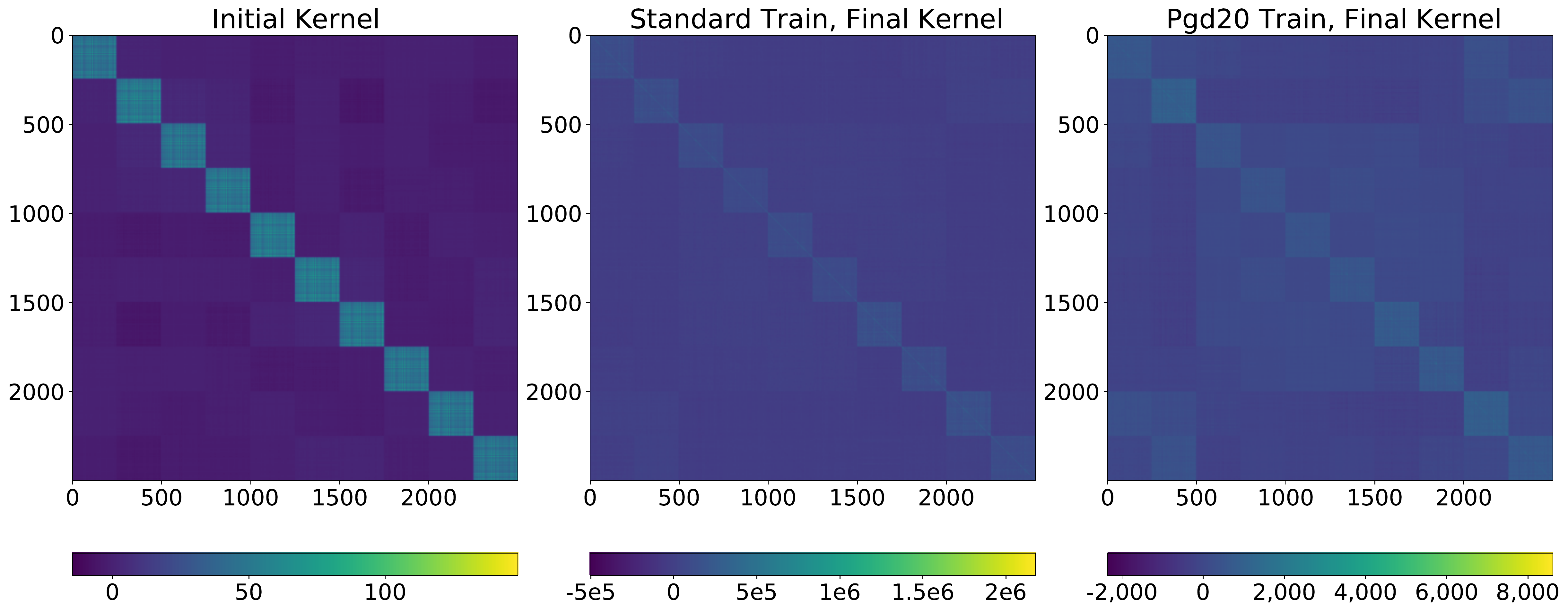}

    \label{fig:kernel_screenshot}
    \end{subfigure}
    \hspace{5mm}
\begin{subfigure}[b]{0.47\textwidth}
    \centering
    \includegraphics[width=0.99\textwidth]{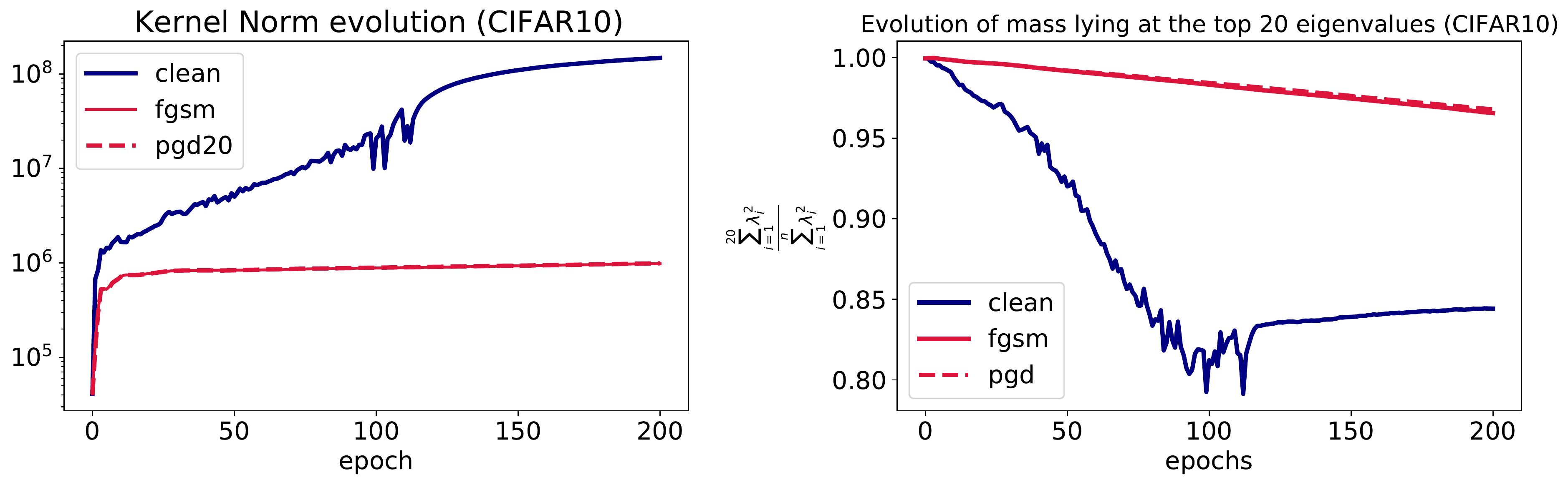}
    \label{fig:norm_and_mass}
\end{subfigure}
\caption{\textbf{Left}: Kernel Matrices for a mini batch of size 256. Left to Right: Kernel at initialization, Kernel after standard training, Kernel after adversarial training (20 PGD steps). The standard kernel grows significantly more than the adversarial one.  \textbf{Right}: (a) Kernel Frobenius norm evolution, and (b) concentration on the top 20 eigenvalues during standard and adversarial training. Setting: CIFAR10, $\ell_\infty = 8 /255$.}
\label{fig:kernels}
\end{figure}


\textbf{Kernel distance.} We compare two kernels using a \textit{scale invariant distance}, which quantifies the relative rotation between them, as used in other works studying NTK dynamics (e.g. \cite{Fort+20}):
\begin{equation}\label{eq:rot}
d(\Theta_i, \Theta_j) = 1 - \frac{\mathrm{Tr}(\Theta_i \Theta_j^\top)}{\sqrt{\mathrm{Tr}(\Theta_i \Theta_i^\top)} \sqrt{\mathrm{Tr}(\Theta_j \Theta_j^\top)}}.
\end{equation}
\textbf{Polar dynamics}. Zooming in on the change that the initial kernel undergoes, we define a \textit{polar space} on which we measure the movement of the kernel:
\begin{equation}\label{eq:polar}
   \begin{split}
        r_t = \frac{\| \Theta_t - \Theta_0 \|_F}{\| \Theta_f - \Theta_0 \|_F}, \; \; \;
        \theta_t = \arccos\left(1 - d(\Theta_t, \Theta_0)\right),
    \end{split}
\end{equation}
where $\Theta_0, \Theta_f$ are the initial and final kernel, respectively.
Fig.~\ref{fig:polarrotation} presents a heatmap of kernel distances at different time steps for both standard and adversarial training, as well as both training trajectories in polar space.

\begin{figure}
\begin{subfigure}[t]{0.7\textwidth}
           \centering
        \includegraphics[width=\linewidth]{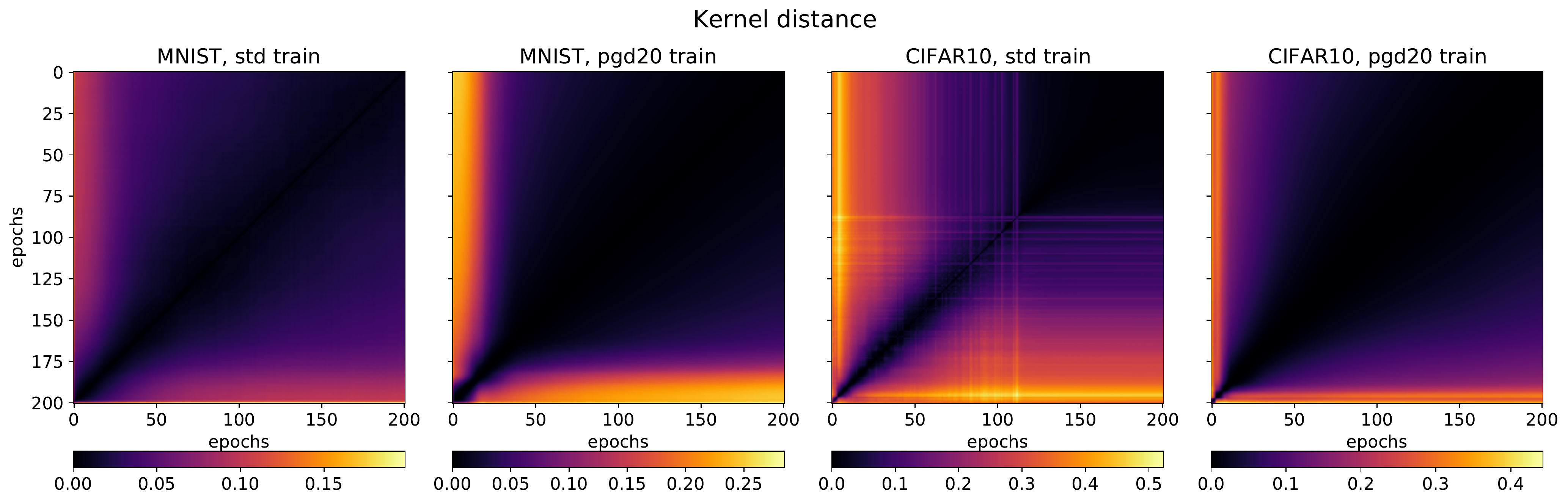}
        \label{fig:kernelrotations}
        \end{subfigure}
        \hfill        
\begin{subfigure}[t]{0.29\textwidth}
\includegraphics[scale=0.18]{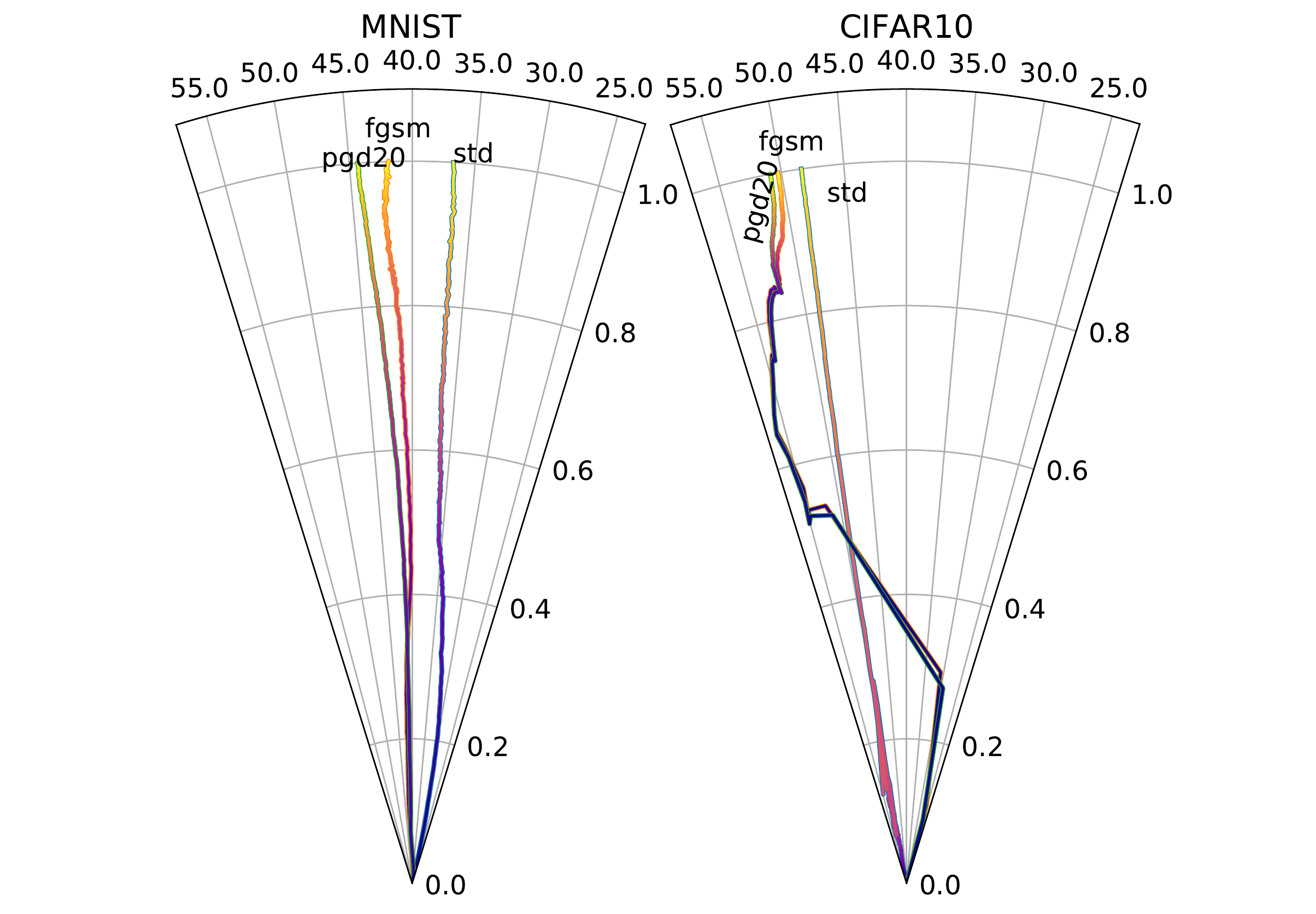}
\vfill
\end{subfigure}
\caption{\textbf{Left:} Rotation (Eq.~\eqref{eq:rot}) of the empirical NTK during standard, and adversarial training. Left to right: MNIST, standard, MNIST adversarial, CIFAR standard, CIFAR adversarial. \textbf{Right:} Kernel trajectories in polar space (Eq.~\eqref{eq:polar}) for MNIST (left) and CIFAR10 (right). Darker colors indicate earlier epochs.}\label{fig:polarrotation}
\end{figure}

\textbf{Concentration on subspaces.} 
To quantify weight concentration on the top region of the spectrum, we track the (normalized) Frobenius norm of subspaces as $\sum_{i = 1}^p \lambda_i^2  /\sum_{i = 1}^n \lambda_i^2$, for various cut-offs $p$, where we have indexed the eigenvalues from largest to smallest. Fig.~\ref{fig:kernels} depicts concentration on the top 20 eigenvalues during training.

Our findings show that similar to what has been reported in prior work \citep{Fort+20},
 the kernel rotates significantly in the beginning of training and then slows down for both standard and adversarial training. However, in the latter case, this second phase begins a lot earlier.
As Fig.~\ref{fig:polarrotation} illuminates, the kernel moves a greater distance than when performing standard training, but after a few epochs it stops both rotating and expanding; note that this is not the case for standard training where the kernel increases its magnitude substantially later in training, and in fact grows to have a norm orders of magnitude larger than during adversarial training (see Fig.~\ref{fig:kernels}).
In hindsight, this behavior is perhaps not surprising, as each element of the kernel measures similarity between data points, and a robust machine should be more conservative when estimating similarity. The observation that during adversarial training the kernel becomes relatively static relatively fast might indicate that {\em linear} dynamics govern the later phase of adversarial training. It has been observed in previous works \citep{Gei+19,Fort+20,Jim+21} that linearization after a few initial epochs of rapid rotation often closely matches performance of full network training. Our results indicate that a similar phenomenon occurs even under the data shift of adversarial training (see Appendix \ref{App:linear_advtrain} for a study of linearized adversarial training), opening avenues to design robust machines more efficiently.

Moreover, endowed with the knowledge that at least for kernels trained with static data robust features lie at the top, we study polar dynamics of the top space only (see Fig.~\ref{fig:top_space_dynamics}) to observe that there is substantial rotation in this space, suggesting that robust features are learned early on not only during standard, but in particular during adversarial training. Even more interestingly, Fig.~\ref{fig:kernels} demonstrates that not only the robust features change, but their relative weight as measured by the concentration on the top-20 space is increasing simultaneously relative to standard training as well, and remains large; in fact, significantly larger than during standard training.
As each eigenvalue weights the importance of the corresponding feature on the final prediction, this implies that the kernel ``learns'' to depend more on the most robust features.

Put together, these findings reveal different kernel dynamics during standard and adversarial training: the kernel rotates much faster, expands much less and becomes ``lazy'' much earlier than during standard training. Fully understanding the properties of converged adversarial kernels remains an important avenue for future work, that might allow to design faster algorithms for robust classification.

\section{Final Remarks}

We have studied adversarial robustness through the lens of the NTK across multiple architectures and data sets  both in the idealized NTK regime and the ``rich feature'' regime. When connecting the spectrum of the kernel with fundamental properties characterizing robustness our phenomenological study reveals a universal picture of the emergence of robust and non-robust features and their role during training.
There are certain limitations and unexplored themes in our work; Sec.~\ref{blackbox_attack} argues that transferable attacks from the NTK may be as effective as white-box attacks, but this warrants an in-depth study across architectures, kernels and data sets (which has not been the main focus of this work).
Sec.~\ref{ntk_feats} visualises features for fairly simple models, since the computation of kernel derivatives is a costly procedure. It would be interesting to use our framework to visualise features from more complicated architectures. Finally, our work in Sec.~\ref{sec_dynamics} invites more research on the kernel at the end of adversarial training, similar to what has been done for standard models \citep{Long21}.

We hope that our viewpoint can motivate further theoretical understanding of adversarial phenomena (such as transferability) and the design of better and/or faster adversarial learning algorithms, by further analyzing the kernels from robust deep neural networks.

\section*{Acknowledgements}

The authors would like to thank Jingtong Su, Alberto Bietti, Yunzhen Feng, and Artem Vysogorets for fruitful discussions and feedback in various stages of this work. NT thanks Dimitris Tsipras for a helpful discussion in the beginning of this project. 
The authors would like to acknowledge support through the National Science Foundation under NSF Award 1922658. This work was supported in part through the NYU IT High Performance Computing resources, services, and staff expertise.

\bibliography{refs-no-url.bib}
\appendix


 

\section*{Checklist}

The checklist follows the references.  Please
read the checklist guidelines carefully for information on how to answer these
questions.  For each question, change the default \answerTODO{} to \answerYes{},
\answerNo{}, or \answerNA{}.  You are strongly encouraged to include a {\bf
justification to your answer}, either by referencing the appropriate section of
your paper or providing a brief inline description.  For example:
\begin{itemize}
  \item Did you include the license to the code and datasets? \answerYes{See Sec.~\ref{ntk_feats} and Appendix.}
  \item Did you include the license to the code and datasets? \answerNo{The code and the data are proprietary.}
  \item Did you include the license to the code and datasets? \answerNA{}
\end{itemize}
Please do not modify the questions and only use the provided macros for your
answers.  Note that the Checklist section does not count towards the page
limit.  In your paper, please delete this instructions block and only keep the
Checklist section heading above along with the questions/answers below.

\begin{enumerate}

\item For all authors...
\begin{enumerate}
  \item Do the main claims made in the abstract and introduction accurately reflect the paper's contributions and scope?
    \answerYes{}{}
  \item Did you describe the limitations of your work?
    \answerYes{}{}{}
  \item Did you discuss any potential negative societal impacts of your work?
    \answerYes{Our work sheds light properties of adversarial examples to make mahcine learning models more reliable in the long run.}
  \item Have you read the ethics review guidelines and ensured that your paper conforms to them?
    \answerYes{}
\end{enumerate}

\item If you are including theoretical results...
\begin{enumerate}
  \item Did you state the full set of assumptions of all theoretical results?
    \answerYes{}
        \item Did you include complete proofs of all theoretical results?
    \answerYes{}
\end{enumerate}

\item If you ran experiments...
\begin{enumerate}
  \item Did you include the code, data, and instructions needed to reproduce the main experimental results (either in the supplemental material or as a URL)?
    \answerYes{}
  \item Did you specify all the training details (e.g., data splits, hyperparameters, how they were chosen)?
    \answerYes{}
        \item Did you report error bars (e.g., with respect to the random seed after running experiments multiple times)?
    \answerYes{}
        \item Did you include the total amount of compute and the type of resources used (e.g., type of GPUs, internal cluster, or cloud provider)?
    \answerYes{}
\end{enumerate}

\item If you are using existing assets (e.g., code, data, models) or curating/releasing new assets...
\begin{enumerate}
  \item If your work uses existing assets, did you cite the creators?
    \answerYes{}
  \item Did you mention the license of the assets?
    \answerYes{}
  \item Did you include any new assets either in the supplemental material or as a URL?
    \answerYes{}
  \item Did you discuss whether and how consent was obtained from people whose data you're using/curating?
    \answerNA{}
  \item Did you discuss whether the data you are using/curating contains personally identifiable information or offensive content?
    \answerNA{}
\end{enumerate}

\item If you used crowdsourcing or conducted research with human subjects...
\begin{enumerate}
  \item Did you include the full text of instructions given to participants and screenshots, if applicable?
    \answerNA{}
  \item Did you describe any potential participant risks, with links to Institutional Review Board (IRB) approvals, if applicable?
    \answerNA{}
  \item Did you include the estimated hourly wage paid to participants and the total amount spent on participant compensation?
    \answerNA{}
\end{enumerate}

\end{enumerate}

\appendix


 

\section{Robust and Non-Robust features}\label{App:feats}

The idea that data features are to be blamed for the adversarial weakness of machine learning models was proposed in \citep{Ily+19,Tsi+19}.
In particular, \cite{Ily+19} show that training with adversarially perturbed images labeled with the ``wrong'' label yields classifiers with non-trivial test performance (``learning from non-robust features only''), while, in a dual experiment, they demonstrate that standard training with ``robustified'' data (data that presumably are ``denoised'' from non-robust features) produces a classifier with non-trivial \textit{robust} accuracy (``relies only on robust features''). Motivated by these observations, the authors propose a model of robust/non-robust features that are hidden in the data, and whose presence determines the eventual robustness of models.
To accompany the definitions of Sec. 2.2, we extend them for multiclass classification, since Sec. 4 introduces our NTK feature framework for both binary and multiclass problems.

Let $\mathcal{D}$ be the data generating distribution, with $x \in \mathcal{X}$ (input space) and $y \in \{1, \ldots, k\}$ (action space). We define features $\phi: \mathcal{X} \to \mathbb{R}^k$ as functions from the input to the action space, and categorize them as follows, according to their performance as classifiers. Fix $\rho, \gamma \geq 0$:

\begin{enumerate}
\item \textbf{$\rho$-Useful} feature:  A feature $\phi$ is called \textit{$\rho$-useful} if
\begin{equation}\label{eq:usefulapp}
\mathbb{E}_{x, y \sim \mathcal{D}}\big[  \mathds{1}_{\{\arg \max_{i \in [k]} \phi_i(x) = y\} } \big] = \rho
\end{equation}

\item \textbf{$\gamma$-Robust} feature: A feature $\phi$ is called \textit{$\gamma$-robust} if it is predictive of the true label under any perturbation inside a bounded ``ball'' $\mathcal{B}$, that is if
\begin{equation}\label{eq:robustapp}
\mathbb{E}_{x, y \sim \mathcal{D}}\big[ \inf_{\delta \in \mathcal{B}} \mathds{1}_{ \{\arg \max_{i \in [k]} \phi_i(x+\delta) = y \} } \big] = \gamma 
\end{equation}

\item \textbf{Useful, non-robust} feature:  A feature is called \textbf{useful, non-robust} if it confers an advantage above guessing the most likely label, i.e.  $\exists \rho > \max_{i \in [k]} \mathbb{E}_{x, y \sim \mathcal{D}}[\mathds{1}_{\{i=y\}}] $, but is $\gamma$-robust only for $\gamma \approx 0$ (within some precision).

\end{enumerate}

The above framework was introduced by \citep{Ily+19,Tsi+19}, and we have slightly adapted it in terms of accuracy as classifiers derived from features. \cite{Goh19} showed how such feature functions arise in a simple linear model, and proposed two mechanisms to construct useful, non-robust features. In \citep{ZhLi20}, the authors view the weights of neural networks as features, and show that adversarial training ``purifies/robustifies'' them.

\section{Derivation of Adversarial Perturbations for Kernel Regression}
\label{App:fgsm_ntk}

In this section, we derive expressions for adversarial attacks on Neural Tangent Kernels presented in the main paper, as well as additional derivations obtained from first-order expansions around the input. 

\subsection{Adversarial Perturbations from Cross-Entropy Loss}

We first derive the expression in Eq. (8) of the paper. Let $\mathbf{x} \in \mathbb{R}^d$ be an input to the NTK prediction 
\begin{equation}\label{eq:ntk_App}
    f_t(\mathbf{x}) = \Theta(\mathbf{x}, \mathcal{X})^\top \Theta^{-1}(\mathcal{X}, \mathcal{X}) (I - e^{-\lambda \Theta(\mathcal{X}, \mathcal{X}) t}) \mathcal{Y},
\end{equation}
where $(\mathcal{X}, \mathcal{Y})$ is a dataset of size $n$.
We consider the binary and the multiclass case separately.

In the \textbf{binary} case, where $y \in \{\pm 1\}$, we feed expression Eq.~\eqref{eq:ntk_App} to a sigmoid $\sigma (x) = (1 + e^{-x})^{-1}$ and maximize the cross entropy loss between the output and the true label:
\begin{equation}
    \mathcal{L} (\mathbf{x}, y) = - \hat{y} \log \left(\sigma(f_t(\mathbf{x}))\right) - (1-\hat{y}) \log \left(1 - \sigma(f_t(\mathbf{x})) \right),
\end{equation}
where we set $\hat{y}=\frac{y+1}{2}$ to lie in $\{0, 1\}$. We compute the gradient of the loss with respect to $\mathbf{x}$:
\begin{equation}\label{eq:ntk-fgsm}
    \begin{split}
        \nabla_\mathbf{x} \mathcal{L} (\sigma(f_t(\mathbf{x})), y) & = -  \frac{\hat{y}}{\sigma(f_t(\mathbf{x}))} \nabla_\mathbf{x}\sigma(f_t(\mathbf{x})) +  \frac{(1-\hat{y})}{1 - \sigma(f_t(\mathbf{x}))} \nabla_\mathbf{x} \sigma(f_t(\mathbf{x})) \\
        & = \frac{\sigma(f_t(\mathbf{x})) - \hat{y}}{\sigma(f_t(\mathbf{x}))(1 - \sigma(f_t(\mathbf{x})))} \nabla_\mathbf{x} \sigma(f_t(\mathbf{x})) 
         = \left(\sigma(f_t(\mathbf{x})) - \hat{y} \right)  \nabla_{\mathbf{x}} f_t(\mathbf{x}).
    \end{split}
\end{equation}
So the optimal one-step attack, under an $\ell_\infty$ adversary, reduces to computing perturbation
\begin{equation}
    \bm{\eta} = \epsilon \cdot \mathrm{sign}\left(\left(\sigma(f_t(\mathbf{x})) - \hat{y} \right)  \nabla_{\mathbf{x}} f_t(\mathbf{x})\right) = - \epsilon y \cdot \mathrm{sign}\left( \nabla_{\mathbf{x}} f_t(\mathbf{x}) \right),
\end{equation}
since $\sigma(u) \in (0, 1)$ for all $u \in \mathbb{R}$.


In the case of a \textbf{k-class} classification problem with one hot labels $\mathcal{Y} \in \mathbb{R}^{n \times k}$, we can express the cross entropy loss between the NTK predictions Eq.~\eqref{eq:ntk_App} and the labels as:
\begin{equation}
    \begin{split}
        \mathcal{L} (\mathbf{x}, y) = - \log \frac{e^{f_{t, y}(\mathbf{x})}}{\sum_{r = 1}^k e^{f_{t, r}(\mathbf{x})}} = - f_{t, y}(\mathbf{x}) + \log \sum_{r = 1}^k e^{f_{t, r}(\mathbf{x})},
    \end{split}
\end{equation}
where $f_{t, r}$ denotes the $r$-th output of Eq.~\eqref{eq:ntk_App}. Computing the loss gradient as before yields the optimal perturbation $\bm{\eta}$,

\begin{equation}\label{eq:ntk-fgsm-multi}
     \bm{\eta} = \epsilon \cdot \mathrm{sign} \left( - \nabla_\mathbf{x} f_{t, y}(\mathbf{x}) + \frac{\sum_{r = 1}^k e^{f_{t, r}(\mathbf{x})} \nabla_\mathbf{x} f_{r, y}(\mathbf{x})}{\sum_{r = 1}^k e^{f_{r, t}(x)}} \right).
\end{equation}

The above calculations allow us to speed up the computation of the attacks in the case of NTKs with closed form expression, since the gradient
\begin{equation}
    \nabla_\mathbf{x} f_{t, r}(\mathbf{x}) = D \Theta(\mathbf{x}, \mathcal{X})^\top \Theta^{-1}(\mathcal{X}, \mathcal{X}) (I - e^{-\lambda \Theta(\mathcal{X}, \mathcal{X}) t}) \mathcal{Y}_{:, r}, 
\end{equation}
with D being  the Jacobian of $\Theta$ wrt to $\mathbf{x}$, can be pre-computed, without the need for auto-differentiation tools. We leverage this in the experiments of Sec. 3.

\subsection{Alternative Approaches to Generate Perturbations}

One can derive other perturbation variants by changing the loss function from cross-entropy to other functions studied in the literature in this context (e.g. \citep{CaWa17}). Alternatively, we can study the output $f_t(x)$ on a test input $x$ directly to devise strategies to most efficiently perturb it, using a Taylor expansion around the input, leading to a linear expression (shown here for scalar kernels):
\begin{equation}\label{eq:prediction_onadv_simplified}
    f(\mathbf{x} + \bm{\eta}) \approx f(\mathbf{x}) + \bm{\eta}^T \mathbf{z},
\end{equation}
for some $z \in \mathbb{R}^d$ that depends on the training data and the NTK kernel only.

\paragraph{Binary case:} Suppose we would like to evaluate a model described by Eq.~(7) at the end of training,
\begin{equation}\label{eq:kernel_pred}
    f_{\infty} (\mathbf{x}) = \Theta(\mathbf{x}, \mathcal{X})^\top \Theta(\mathcal{X}, \mathcal{X})^{-1} \mathcal{Y}
\end{equation}

on \textit{slightly} perturbed variations of the original {\em training} data. Then, slightly abusing  notation, we set, $\Tilde{\mathcal{X}} = \mathcal{X} + \bm{\epsilon}$, that is $\Tilde{\mathbf{x}}_i = \mathbf{x}_i + \bm{\eta}_i$ for all $\mathbf{x}_i \in \mathcal{X}$ for small, but unknown, perturbations $\bm{\eta}_i$. By taking a first-order Taylor expansion in the perturbation, we can write the $ij$-th element of $\Theta(\Tilde{\mathcal{X}}, \mathcal{X})$ as follows:
\begin{equation}\label{eq:taylor_exp}
    \begin{split}
        \Theta(\Tilde{\mathbf{x}}_i, \mathbf{x}_j) = \Theta(\mathbf{x}_i + \bm{\eta}_i, \mathbf{x}_j)
        \approx \Theta(\mathbf{x}_i, \mathbf{x}_j) + \nabla_{\mathbf{x}_i} \Theta(\mathbf{x}_i, \mathbf{x}_j)^T \bm{\eta}_i.
    \end{split}
\end{equation}
For each row $\underbrace{\Theta_{i,:}(\Tilde{\mathcal{X}}, \mathcal{X})}_{\in \mathbb{R}^{1 \times n}}$ we obtain:
\begin{equation}\label{eq:theta_row}
    \Theta_{i,:}(\Tilde{\mathcal{X}}, \mathcal{X})^T = \Theta_{i,:}(\mathcal{X}, \mathcal{X})^T + \underbrace{\begin{pmatrix}
        \nabla_{\mathbf{x}_i} \Theta(\mathbf{x}_i, \mathbf{x}_1)^T \\
        \nabla_{\mathbf{x}_i} \Theta(\mathbf{x}_i, \mathbf{x}_2)^T \\
        \vdots \\
        \nabla_{\mathbf{x}_i} \Theta(\mathbf{x}_i, \mathbf{x}_n)^T
    \end{pmatrix}}_{\mathbf{A}_i \in \mathbb{R}^{n \times d}}
    \bm{\eta}_i.
\end{equation}
Hence, $\Theta(\Tilde{\mathcal{X}}, \mathcal{X})$ can be written as $\Theta(\mathcal{X}, \mathcal{X}) + \bm{\Delta}$ for a perturbation matrix $\bm{\Delta}$, with $i$-th row $\bm{\Delta}_{i,:} = \bm{\eta}_i^T \mathbf{A}_i^T$. Substituting into Eq.~(\ref{eq:kernel_pred}), we get:
\begin{equation}\label{eq:total_pred}
    \begin{split}
        f(\Tilde{\mathcal{X}}) = (\Theta(\mathcal{X}, \mathcal{X}) + \bm{\Delta}) \Theta (\mathcal{X}, \mathcal{X})^{-1} \mathcal{Y} 
         = \mathcal{Y} + \bm{\Delta} \Theta(\mathcal{X}, \mathcal{X})^{-1} \mathcal{Y}.
    \end{split}
\end{equation}
Thus, the output of the model on $\Tilde{\mathbf{x}}_i$ is:
\begin{equation}\label{eq:sample_pred}
    \begin{split}
        f(\Tilde{\mathbf{x}}_i) & = y_i + \bm{\Delta}_i \Theta(\mathcal{X}, \mathcal{X})^{-1} \mathcal{Y} \\
        & = y_i + \bm{\eta}_i^T \mathbf{A}_i^T  \Theta (\mathcal{X}, \mathcal{X})^{-1} \mathcal{Y} =: y_i + \bm{\eta}_i^T \mathbf{z}_i,
    \end{split}
\end{equation}
leading to the linear expression advertised in Eq.~(\ref{eq:prediction_onadv_simplified}).
The adversarial perturbation $\bm{\eta}_i$ changes the output by $\bm{\eta}_i^T \mathbf{z_i}=\bm{\eta}_i^T \mathbf{A}_i^T \Theta(\mathcal{X}, \mathcal{X})^{-1} \mathcal{Y}$, an expression which allows us to {\em compute} the adversarial perturbation to maximally change the output within the desired constraints on $\bm{\eta_i}$. 

Since Eq.~(\ref{eq:kernel_pred}) describes regression models with LSE ($L_2$-loss), while adversarial examples typically are studied for classification models, we use thresholding (i.e. taking the sign of the output in the case of binary  $\{-1,1\}$ classification tasks) or by outputting the maximum prediction (in the case of multiclass problems) to turn Eq.~(\ref{eq:kernel_pred}) into a classifier.

Inspecting Eq.~(\ref{eq:sample_pred}), maximal ``confusion" of the classification model is achieved by aligning $\bm{\eta}_i$ with $-y_i\mathbf{z_i}$ (directed towards the decision boundary).
In case of the commonly used $\ell_\infty$ restriction, i.e. $\| \bm{\eta}_i \|_\infty \leq \epsilon$, the optimal adversarial perturbation is given by:
%
\begin{equation}\label{eq:opt_perturb}
    \bm{\eta}_i = -\epsilon y_i \cdot \mathrm{sign}(\mathbf{A}_i^T  \Theta (\mathcal{X}, \mathcal{X})^{-1} \mathcal{Y}).
\end{equation}
%
 The computation of this optimal adversarial perturbation requires an expression for the NTK and its gradient with respect to the training data. For models where an {\em analytical} expression of the NTK is available, only access to the labeled training data is necessary (as presented, for instance, in Sec.~\ref{two-layer}). In more complicated models or those that deviate from the assumptions for Eq.~(\ref{eq:kernel_pred}) one can compute an {\em empirical} kernel by sampling over kernels at initialization over a few instances and obtain the matrices $\mathbf{A_i}$ with autodifferentiation tools.


Eq.~(\ref{eq:sample_pred}) has been derived for perturbations of the {\em training} data. 
Consider now the case when we evaluate Eq.~(\ref{eq:kernel_pred}) on perturbations of unseen {\em test} data, that is on $\Tilde{\mathcal{X}} + \bm{\epsilon}$. Then, Eq.~(\ref{eq:total_pred}) becomes:
\begin{equation}\label{eq:total_pred_test}
    \begin{split}
        f(\Tilde{\mathcal{X}} + \bm{\epsilon})  = (\Theta(\Tilde{\mathcal{X}}, \mathcal{X}) + \bm{\Delta})  \Theta(\mathcal{X}, \mathcal{X})^{-1} \mathcal{Y} 
        = f(\Tilde{\mathcal{X}}) + \bm{\Delta}  \Theta(\mathcal{X}, \mathcal{X})^{-1} \mathcal{Y}.
    \end{split}
\end{equation}
Again, solely the second term depends on the perturbation, so we proceed by choosing a maximally perturbing direction as before. The only difference lies in the matrix $\bm{\Delta}$ that now depends on the test set $\Tilde{\mathcal{X}}$
\begin{equation}
    \bm{\Delta}_{i, :} = \bm{\eta}_i^T \begin{pmatrix}
        \nabla_{\Tilde{\mathbf{x}}_i} \Theta(\Tilde{\mathbf{x}}_i, \mathbf{x}_1)^T \\
        \nabla_{\Tilde{\mathbf{x}}_i} \Theta(\Tilde{\mathbf{x}}_i, \mathbf{x}_2)^T \\
        \vdots \\
        \nabla_{\Tilde{\mathbf{x}}_i} \Theta(\Tilde{\mathbf{x}}_i, \mathbf{x}_n)^T
    \end{pmatrix}^T
\end{equation}

In practice, an adversary can calculate the NTK $\Theta(\mathcal{X}, \mathcal{X})$ offline and calculate the optimal perturbation on a new test input $\Tilde{\mathbf{x}}_i$ by computing the corresponding row of the matrix $\bm{\Delta}$. Importantly, no information on the test data {\em labels} is needed. 

\paragraph{Multiclass case:}

We adapt the derivations of the binary case to the setting where the output dimension is larger than one in the underlying regression setting (see below), resulting in a multiclass classifier. This leads to the multi-dimensional analogue of the linear Eq.~(\ref{eq:prediction_onadv_simplified}) for $f(x) \in \mathbb{R}^{k}$, $y\in \mathbb{R}^{k}$:
\begin{equation}\label{eq:multi-linear}
    f(\mathbf{x}_i + \bm{\eta}_i) = \mathbf{y}_i + \begin{pmatrix}
     \bm{\eta}_i^T \mathbf{z}_1 \\
     \bm{\eta}_i^T \mathbf{z}_2 \\
    \vdots \\
     \bm{\eta}_i^T \mathbf{z}_k
    \end{pmatrix}.
\end{equation}
Again, the $z \in \mathbb{R}^d$ can be computed from the NTK and its derivative as well as the training data labels. Exactly analogous considerations as in the binary case allow to adapt this expression to perturbations of the {\em test} data.

At this point we have a choice of how to adversarially perturb the classifier to achieve the largest effect on the network output. We present the two most obvious methods.

{\em Max-of-$\ell_1$ perturbation: } Similar in spirit to traditional approaches in adversarial attacks (\cite{CaWa17}) we choose $\bm{\eta_i}$ such as to most efficiently decrease the correct response $r^*= \mathrm{arg} \max_j (\mathbf{y}_i)_j$ while maximally increasing one of the false responses $r \neq r^*$.
The solution is given by:
\begin{equation}
        \bm{\eta}_i = \varepsilon \cdot \mathrm{sign}(\mathrm{arg} \max_{r = 1, r \neq r^\star}^k \| \mathbf{z}_r - \mathbf{z}_{r^\star} \|_1). 
\end{equation}

 It is obtained by solving $$\bm{\eta}_i = \mathrm{arg} \max_{\| \bm{\eta}_i \|_\infty \leq \epsilon} \max_{r = 1, r \neq r^\star}^k f_r(\mathbf{x}_i + \bm{\eta}_i) - f_{r^\star}(\mathbf{x}_i + \bm{\eta}_i).$$ Then 
\begin{equation}
    \begin{split}
        \bm{\eta}_i & = \mathrm{arg} \max_{\| \bm{\eta}_i \|_\infty \leq \epsilon} \max_{r = 1, r \neq r^\star}^k \bm{\eta}_i^T \mathbf{z}_r - \bm{\eta}_i^T \mathbf{z}_{r^\star} \\ 
        & = \mathrm{arg} \max_{\| \bm{\eta}_i \|_\infty \leq \epsilon} \max_{r = 1, r \neq r^\star}^k \bm{\eta}_i^T (\mathbf{z}_r - \mathbf{z}_{r^\star}) \\
        & = \varepsilon \cdot \mathrm{sign}(\mathrm{arg} \max_{r = 1, r \neq r^\star}^k \| \mathbf{z}_r - \mathbf{z}_{r^\star} \|_1). 
    \end{split}
\end{equation}
%

%

{\em Sum-of-$\Delta z$ perturbation:}
For one-hot vectors $\mathbf{y}_i$ we could, instead, maximize the cross-entropy between the labels and the new outputs, thus choosing to produce a maximally mixed output. If $r^*$ is the correct label, this yields
\begin{equation}\label{eq:sum_delta_z}
\bm{\eta}_i = \varepsilon \cdot \mathrm{sign}(\sum_{r \neq r^*}^n (\mathbf{z}_r - \mathbf{z}_{r^\star})).
\end{equation}
derived as follows
\begin{equation}
    \begin{split}
        L_{ce} (f(\mathbf{x}_i + \bm{\eta}), \mathbf{y}_i) & = -\sum_{r = 1}^k y_i^{(r)} \log \left(\frac{e^{y_i^{(r)} + \bm{\eta}_i^T \mathbf{z}_r}}{\sum_{r^\prime = 1}^k e^{y_i^{(r^\prime)} + \bm{\eta}_i^T \mathbf{z}_r^\prime}} \right) \\
        & = - \log \left(\frac{e^{1+\bm{\eta}_i^T \mathbf{z}_{r^\star}}}{\sum_{r = 1, r \neq r^\star}^k e^{\bm{\eta}_i^T \mathbf{z}_r} + e^{1+\bm{\eta}_i^T \mathbf{z}_{r^\star}}}\right) \\
        & = \log \left(\sum_{r \neq r^\star} e^{\bm{\eta_i}^T (\mathbf{z}_r - \mathbf{z}_{r^\star}) - 1} + 1\right).
    \end{split}
\end{equation}
Maximizing this cross entropy amounts to maximizing
$$\sum_{r \neq r^\star} e^{\bm{\eta_i}^T (\mathbf{z}_r - \mathbf{z}_{r^\star})}.$$
For small perturbations we can develop the exponential to first order\footnote{The resulting expression for the maximum also holds when developing to second order.}, which leads to finding the maximum of 
$$\bm{\eta_i}^T \sum_{r \neq r^\star}(\mathbf{z}_r - \mathbf{z}_{r^\star}),$$

yielding Eq.~\eqref{eq:sum_delta_z}.


{\em Derivation of Eq.~\eqref{eq:multi-linear}:}
While we remain with $\mathcal{X} \in \mathbb{R}^{n \times d}$ as in the binary case, the other quantities change as $\mathcal{Y} \in \mathbb{R}^{nk}$, $f(\mathcal{X}) \in \mathbb{R}^{nk}$ and $\Theta(\mathcal{X}, \mathcal{X}) \in \mathbb{R}^{nk \times nk}$, i.e. for each data pair $(\mathbf{x}_i, \mathbf{x}_j)$ we have $\Theta(\mathbf{x}_i, \mathbf{x}_j) \in \mathbb{R}^{k \times k}$. Let $\Theta_{lm}(\mathbf{x}_i, \mathbf{x}_j)$ denote the entry of $\Theta(\mathbf{x}_i, \mathbf{x}_j)$ that corresponds to the $l$-th and the $m$-th output of the model (evaluated at $\mathbf{x}_i$ and $\mathbf{x}_j$). Then, with similar reasoning that led to Eq.~(\ref{eq:taylor_exp}) we now obtain:
\begin{equation}
    \Theta_{lm} (\mathbf{x}_i + \bm{\eta}, \mathbf{x}_j) \approx \Theta_{lm}(\mathbf{x}_i, \mathbf{x}_j) + \nabla_{\mathbf{x}_i} \Theta_{lm}^T(\mathbf{x}_i, \mathbf{x}_j) \bm{\eta}.
\end{equation}
For the prediction of the model on the whole dataset, we have:
\begin{equation}
    f(\mathbf{\Tilde{X}}) = \mathcal{Y} + \underbrace{\bm{\Delta}}_{\in \mathbb{R}^{nk \times nk}} (\Theta(\mathbf{X}, \mathbf{X}))^{-1} \mathcal{Y},
\end{equation}
which for a given sample $\mathbf{x}_i$ gives:
\begin{equation}
    f(\mathbf{x}_i + \bm{\eta}_i) = \underbrace{\mathbf{y}_i}_{\in \mathbb{R}^k} + \bm{\Delta}_{ik:(i+1)k, :}
    \left ( \Theta(\mathbf{X}, \mathbf{X}) \right )^{-1} \mathcal{Y},
\end{equation}
where $\bm{\Delta}_{(i-1)k:ik, :}$ is equal to
\begin{equation}
    \underbrace{\begin{pmatrix}
        \overbrace{\begin{pmatrix} \nabla_{\mathbf{x}_i} \Theta_{11}(\mathbf{x}_i, \mathbf{x}_1) & \ldots & \nabla_{\mathbf{x}_i} \Theta_{1k}(\mathbf{x}_i, \mathbf{x}_1) & \nabla_{\mathbf{x}_i} \Theta_{11}(\mathbf{x}_i, \mathbf{x}_2) & \ldots & \nabla_{\mathbf{x}_i} \Theta_{1k}(\mathbf{x}_i, \mathbf{x}_n) \end{pmatrix}^T}^{\in \mathbb{R}^{nk \times d}} \bm{\eta}_i \\
        \begin{pmatrix} \nabla_{\mathbf{x}_i} \Theta_{21}(\mathbf{x}_i, \mathbf{x}_1) & \ldots & \nabla_{\mathbf{x}_i} \Theta_{2k}(\mathbf{x}_i, \mathbf{x}_1) & \nabla_{\mathbf{x}_i} \Theta_{21}(\mathbf{x}_i, \mathbf{x}_2) & \ldots & \nabla_{\mathbf{x}_i} \Theta_{2k}(\mathbf{x}_i, \mathbf{x}_n) \end{pmatrix}^T \bm{\eta}_i \\
        \vdots \\
        \begin{pmatrix} \nabla_{\mathbf{x}_i} \Theta_{k1}(\mathbf{x}_i, \mathbf{x}_1) & \ldots & \nabla_{\mathbf{x}_i} \Theta_{kk}(\mathbf{x}_i, \mathbf{x}_1) & \nabla_{\mathbf{x}_i} \Theta_{k1}(\mathbf{x}_i, \mathbf{x}_2) & \ldots & \nabla_{\mathbf{x}_i} \Theta_{kk}(\mathbf{x}_i, \mathbf{x}_n) \end{pmatrix}^T \bm{\eta}_i
    \end{pmatrix}}_{\in \mathbb{R}^{k \times nk}}.
\end{equation}

\section{Transfer Results for Wide Two-Layer Networks}\label{two-layer}
\label{App:bbox}

In this section, we present additional experimental details for Sec. 3.2 and show the results of the experiments on MNIST. We train two-layer neural networks of the form
\begin{equation}\label{eq:model}
     f(\mathbf{x}) = \frac{1}{\sqrt{m}} \mathbf{A} \max(\mathbf{W} \mathbf{x}, 0), \; \; \mathbf{A} \in \{ \pm 1 \}^{m \times k}, \mathbf{W} \sim \mathcal{N}(0, 0.01^2 I_{m \times d}),
\end{equation}
where the first layer is initialized with the normal distribution, the second layer is frozen to its initial random values in $\{\pm1\}$, and $m$ denotes the width of the network. The NTK of this architecture is given by
\begin{equation}\label{eq:H_inf}
    \Theta(\mathbf{x}_i, \mathbf{x}_j) = (\frac{1}{2} - \frac{\arccos{(\frac{\mathbf{x}_i^\top\mathbf{x}_j}{\| \mathbf{x}_i \| \| \mathbf{x}_j \|}})}{2\pi}) \mathbf{x}_i^\top\mathbf{x}_j.
\end{equation}

 
We choose this family of models in order to be consistent with early works that analyzed training and generalization properties of neural networks in the NTK regime \citep{Aro+19a,Du+19a}. We perform experiments on image classification on MNIST and on a binary task extracted from CIFAR-10 (car vs airplane). We train the networks in a regression fashion, minimizing the $\ell_2$ loss between the predictions and one-hot vectors, using full-batch gradient descent on the entire dataset (full training data for MNIST and 5K images for each of car and airplane in binary CIFAR). We keep the learning rate fixed to $10^{-2}$ and vary the width of the network in $\{10^3, 10^4\}$. We train 3 networks for each dataset until convergence ($10^5$ epochs), each initialized with a different random seed. When we measure quantities from the neural net, we subtract the initial prediction $f_0$, since the NTK expression Eq.~\eqref{eq:ntk_App} does not take  the initialization of the network into account. When attacking the models ($\ell_\infty$ attacks), we use perturbation budget $\epsilon=0.3$ for MNIST and $\epsilon=8/255$ for CIFAR-10. The experiments are performed with PyTorch \citep{Pas+19}.
 
For each model,  we calculate the loss gradients with respect to the input during training, and compare them to those derived for the NTK in Eqs.~\eqref{eq:ntk-fgsm} and \eqref{eq:ntk-fgsm-multi} for the binary and the multiclass task, respectively, using cosine similarity:
\begin{equation}
    \frac{\nabla_x \mathcal{L}(f_{t}, y)^\top \nabla_x \mathcal{L}(f - f_0, y)}{\| \nabla_x \mathcal{L}(f_{t}, y) \|_2 \|\nabla_x \mathcal{L}(f - f_0, y)\|_2},
\end{equation}
where $f_t$ is the NTK prediction defined in Eq.~\eqref{eq:ntk_App}, $f$ denotes the output of the neural net and $f_0$ is the initial prediction of the neural net (prior to training).
In order to match the time-scales, we manually align the networks on epoch = $10^3$ with a time-point for the NTK, and based on this number, we match the rest of the epochs assuming linear dependence (as theory predicts - Eq.~\eqref{eq:ntk_App}). Fig. \ref{fig:mnist_ntkvSnn} shows cosine similarity of loss gradients and robust accuracy of the network (evaluated against its own adversarial examples, and those from the NTK) for MNIST. Fig. \ref{fig:grads} illustrates the similarity of loss gradients of neural nets and their NTKs for 3 different epochs.

\begin{figure}[h]
    \centering
    \begin{subfigure}[b]{0.47\textwidth}
        \centering
        \includegraphics[scale=0.25]{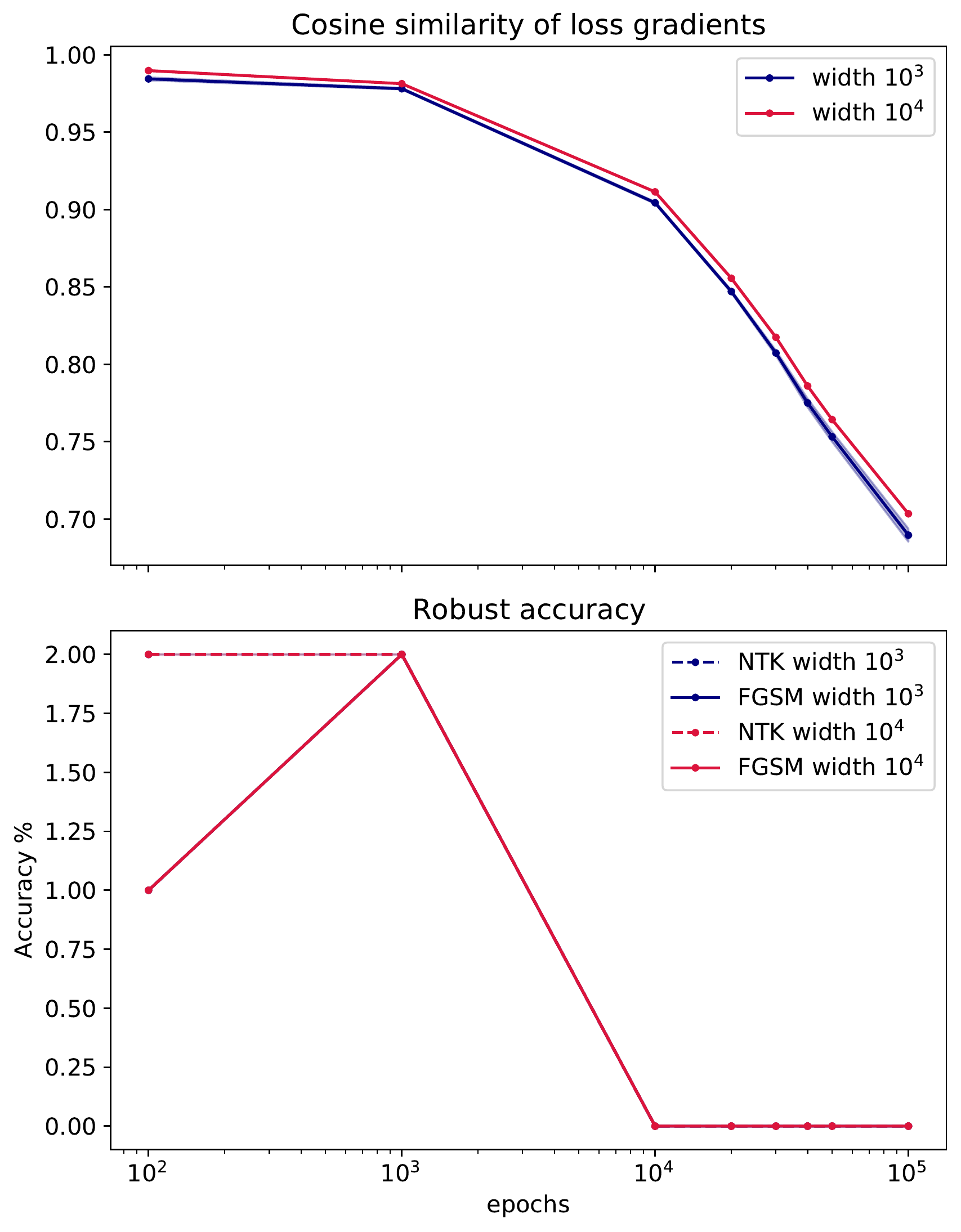}
        \caption{Comparison of NTK and neural net derived quantities for digit recognition (MNIST) during training.}
        \label{fig:mnist_ntkvSnn}
    \end{subfigure}
    \hfill
    \begin{subfigure}[b]{0.5\textwidth}
        \centering
        \includegraphics[scale=0.3]{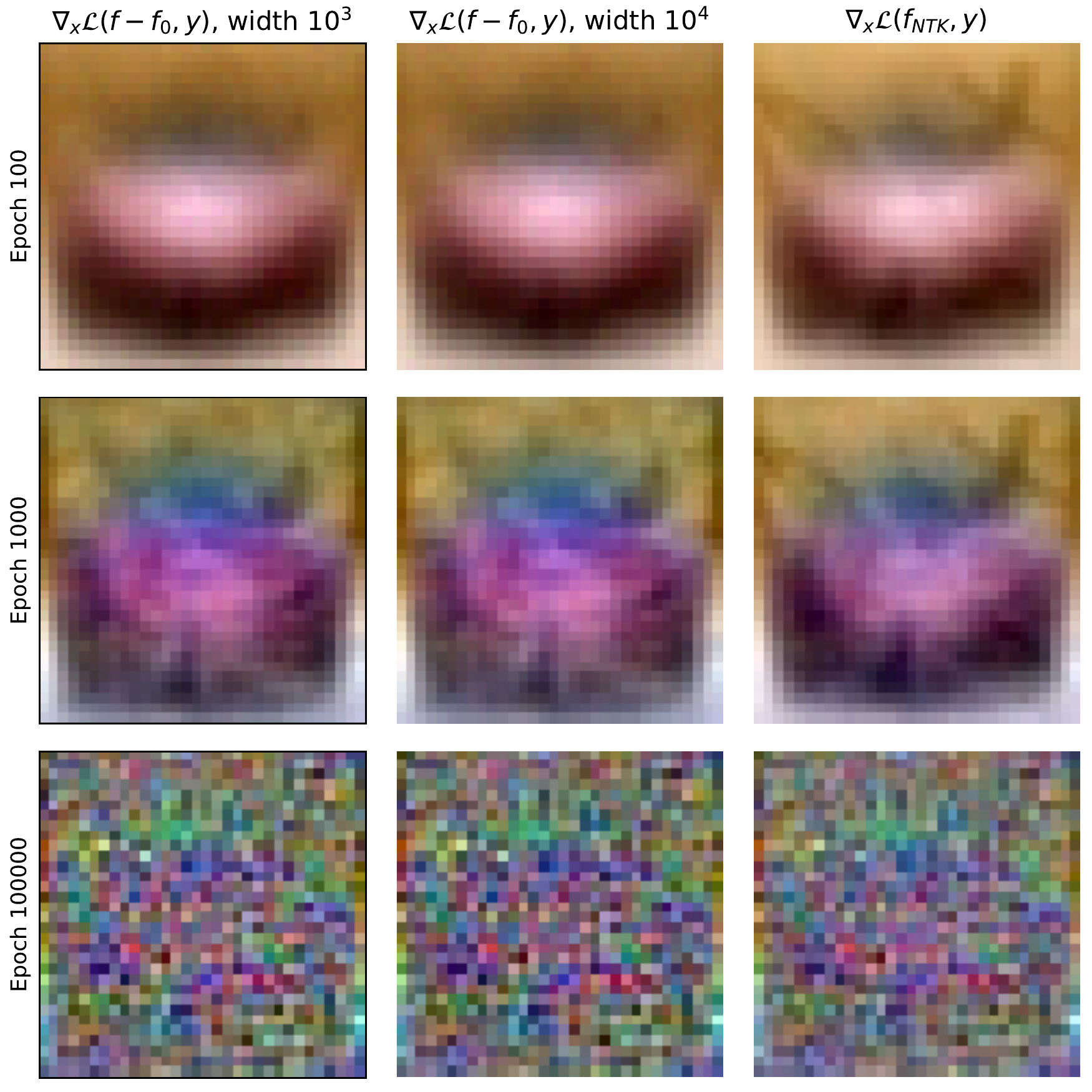}
        \caption{Loss gradients.}
        \label{fig:grads}
    \end{subfigure}
    \caption{Visualizing the similarity of NTK and NN adversarial quantities. (a) \textbf{Top}. Cosine similarity between the loss gradient of the neural net and of the NTK prediction for the same time point (MNIST). \textbf{Bottom}. Robust accuracy of  neural net against its own adversarial examples (solid) and corresponding NTK examples (dashed) for MNIST. Blue and red lines overlap in the second plot, and the effect of the random seed is insignificant. (b) Illustration of the similarity of loss gradients for NTK (\textbf{right} column) and neural nets of width $10^3$ (\textbf{left}) and $10^4$(\textbf{middle}) for a specific image extracted from CIFAR-10. Columns show gradients for different epochs ($10^2, 10^3, 10^4$, respectively).}
    \label{}
\end{figure}


Notice the very small discrepancy between the loss gradients of different networks (initialized with different random seeds) in Fig. \ref{fig:mnist_ntkvSnn}.
They are all centered around the loss gradient of the NTK, a manifestation of transferability of adversarial examples, at least for models with the same architecture. The NTK framework might possibly provide a wider explanation of this phenomenon, also across architectures. For instance, for fully connected kernels, the NTK expression for kernels of depth $l$ is a relatively simple function of expressions for depth $l-1$ \citep{JHG18,BiMa19} which could explain transferability across architectures of varying depth.

\section{NTK Features: Additional Details}
\label{App:ntk_feats}

In this section, we present additional material for Sec. 4; we show derivations that are missing from the main text, and complement the plots by showing the same information for more architectures and datasets.

\subsection{Loss Gradient Decomposition}
First, recall our definitions of features from Sec. 4. Let $\mathcal{X}, \mathcal{Y}$ be a dataset, where $\mathcal{X} \in \mathbb{R}^{n \times d}$ and $\mathcal{Y} \in \{ \pm 1\}^n$ (binary classification). Then, kernel regression on this dataset gives predictions of the form
$f_{\infty} (\mathbf{x}) = \Theta(\mathbf{x}, \mathcal{X})^\top \Theta(\mathcal{X}, \mathcal{X})^{-1} \mathcal{Y}$. Given, the eigendecomposition of the Gram Matrix $\Theta(\mathcal{X}, \mathcal{X})$, we can decompose the prediction as follows
\begin{equation}
	f_{\infty}(\mathbf{x}) = \Theta(\mathbf{x}, \mathcal{X})^\top \sum_{i = 1}^n \lambda_i^{-1} \mathbf{v}_i \mathbf{v}_i^\top \mathcal{Y} = \sum_{i = 1}^n f^{(i)} (\mathbf{x}),
\end{equation}
where $f^{(i)}: \mathbb{R}^d \to \mathbb{R}^k, f^{(i)} (\mathbf{x}) = \lambda_i^{-1} \Theta(\mathbf{x}, \mathcal{X})^\top \mathbf{v}_i \mathbf{v}_i^\top \mathcal{Y}$. Notably, this means that the gradient of the cross entropy loss can be also understood as a composition of gradients coming from these features, as the following proposition shows.

\begin{figure}[h]
    \centering
    \includegraphics[scale=0.3]{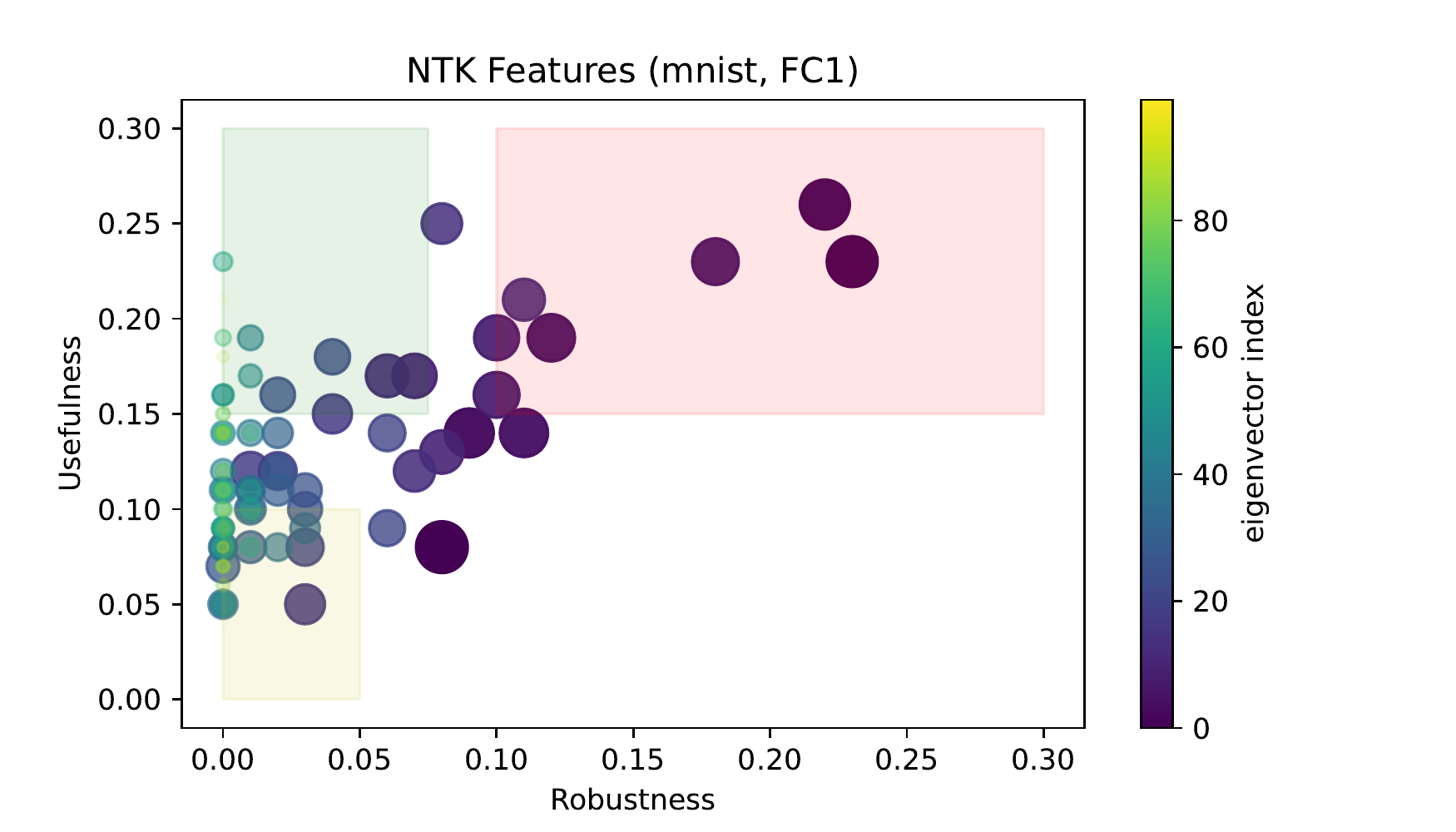}
    \includegraphics[scale=0.3]{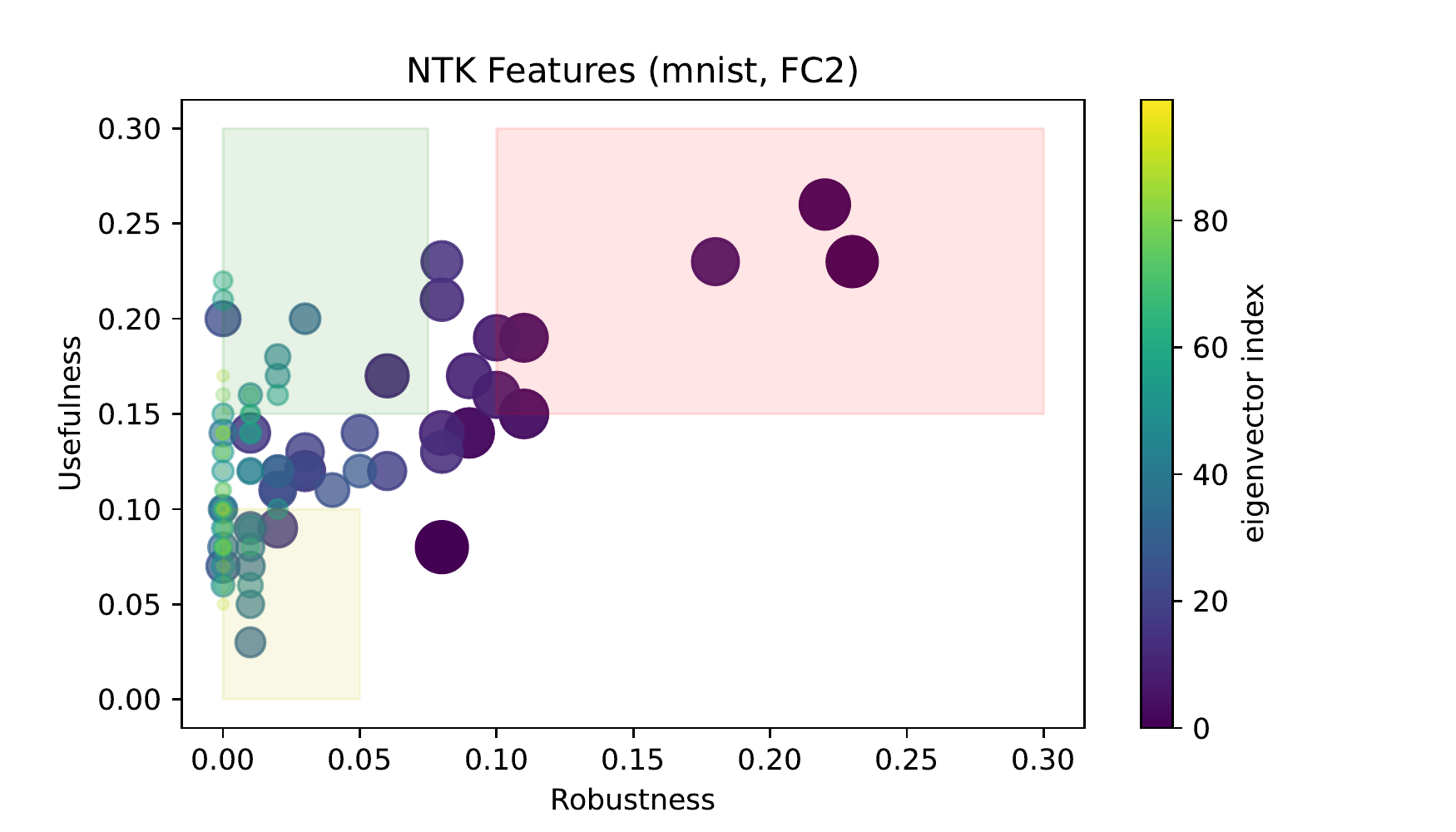}
    \includegraphics[scale=0.3]{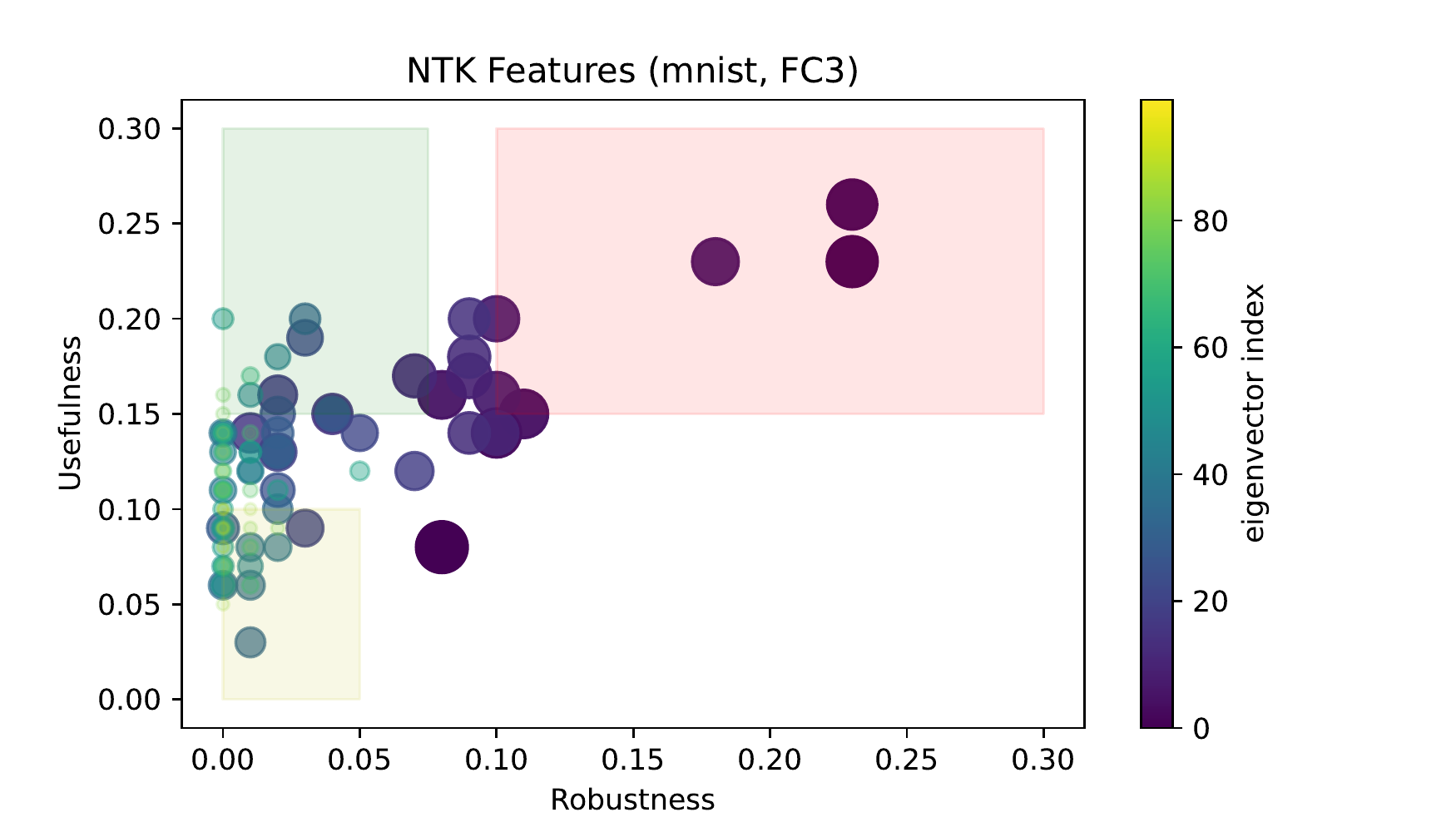}
    \includegraphics[scale=0.3]{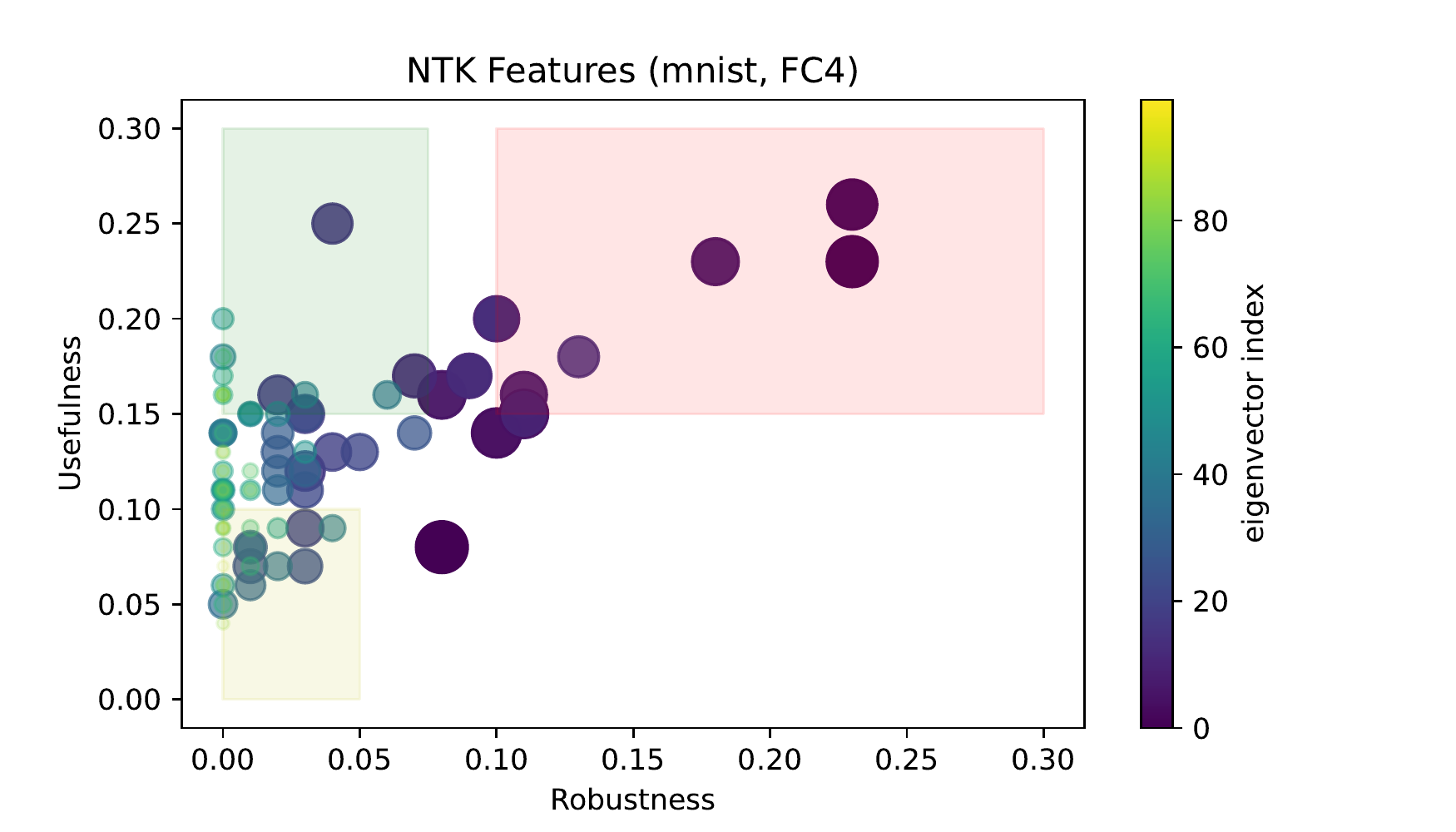}
    \includegraphics[scale=0.3]{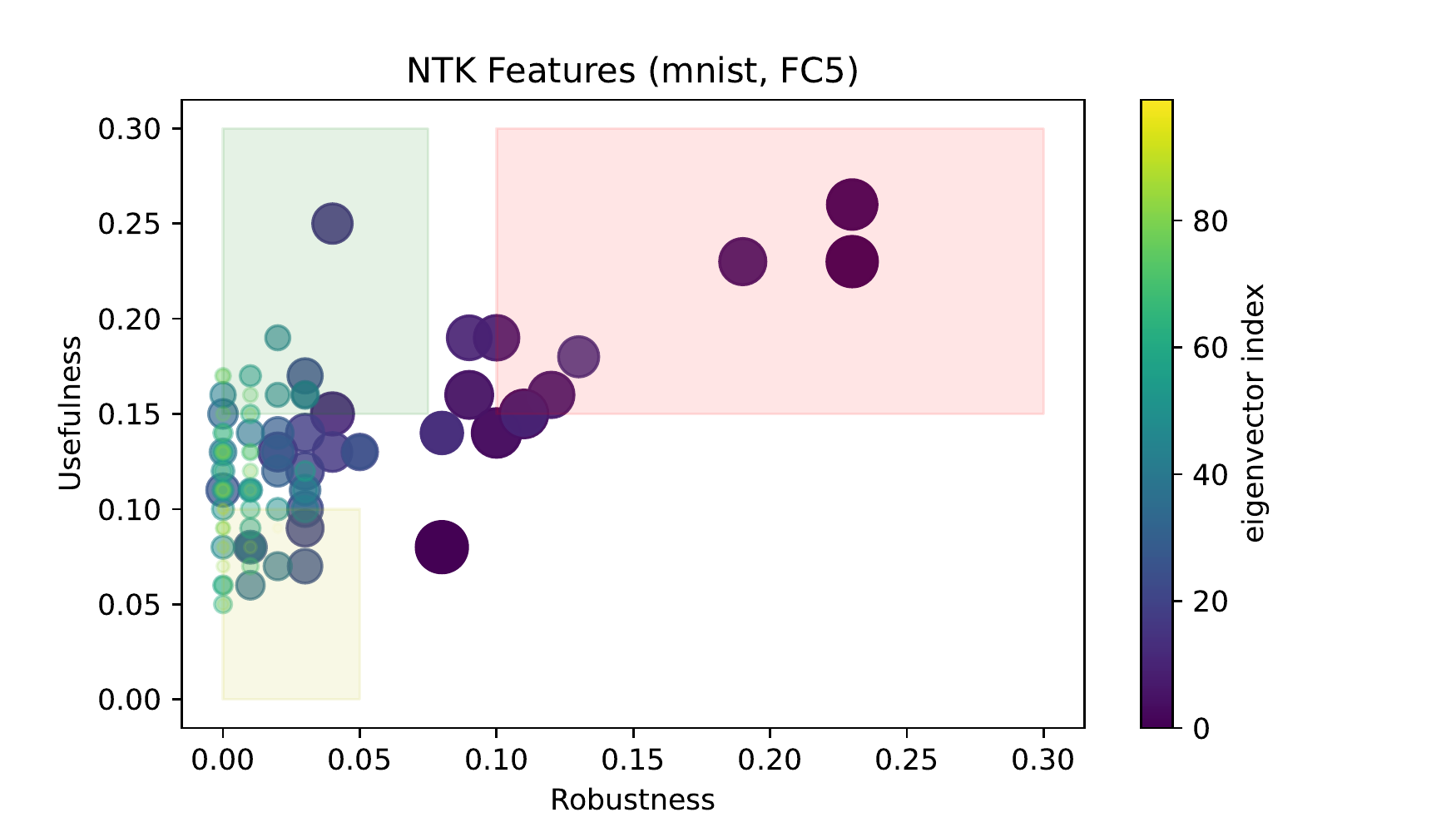}
    \caption{Robustness Usefulness space for various kernels, MNIST multiclass. The axes lie in $[0, 1]$. "Useful" features have usefulness above 0.1 (the random guessing probability for our balanced data set). The colored red, green and yellow boxes are arbitrary, meant to visually distinguish useful-robust from other features.}
    \label{fig:rob_vs_usef_mnist}
\end{figure}

\begin{proposition}
The loss gradient of $f_{\infty}$ can be decomposed as follows:
\begin{equation}
    \nabla_\mathbf{x} \mathcal{L} (f_{\infty}(\mathbf{x}), y) = \sum_{i = 1}^n \alpha_i \nabla_\mathbf{x} \mathcal{L} (f^{(i)}(\mathbf{x}), y),
\end{equation}
where $\alpha_i$ is a quantity that depends on $\mathbf{x}, y$.
\end{proposition}

\begin{proof}
    Recall from Eq.~\eqref{eq:ntk-fgsm}, that $\nabla_{\mathbf{x}} \mathcal{L} (f (\mathbf{x}), y) = \left(\sigma (f(x)) - \frac{y+1}{2}\right) \nabla_{\mathbf{x}} f(\mathbf{x})$. Then, we have
    \begin{equation}
        \begin{split}
            \nabla_{\mathbf{x}} \mathcal{L} (f_{\infty}(\mathbf{x}), y) & = \left(\sigma (f_{\infty}(x)) - \frac{y+1}{2}\right) \nabla_{\mathbf{x}} f_{\infty}(\mathbf{x}) \\
            & = \left(\sigma (f_{\infty}(x)) - \frac{y+1}{2}\right) \nabla_{\mathbf{x}} \sum_{i = 1}^n f^{(i)}(\mathbf{x}) \\
            & = \left(\sigma (f_{\infty}(x)) - \frac{y+1}{2}\right)  \sum_{i = 1}^n \nabla_{\mathbf{x}} f^{(i)}(\mathbf{x}) \\
            & = \left(\sigma (f_{\infty}(x)) - \frac{y+1}{2}\right)  \sum_{i = 1}^n \frac{1}{\left(\sigma (f^{(i)}(x)) - \frac{y+1}{2}\right)} \nabla_{\mathbf{x}} \mathcal{L} (f^{(i)}(\mathbf{x}), y) \\
            & = \sum_{i = 1}^n \underbrace{\frac{\left(\sigma (f_{\infty}(x)) - \frac{y+1}{2}\right)}{\left(\sigma (f^{(i)}(x)) - \frac{y+1}{2}\right)}}_{\alpha_i} \nabla_{\mathbf{x}} \mathcal{L} (f^{(i)}(\mathbf{x}), y).
        \end{split}
    \end{equation}
\end{proof}

\subsection{Additional Plots}

Complementing Fig.~\ref{fig:features} in the main text, we show (the first 100) NTK features in Robustness - Usefulness space defined in Sec. 4 for a larger number of architectures for both MNIST and CIFAR in Fig. \ref{fig:rob_vs_usef_mnist} and \ref{fig:rob_vs_usef_cifar10}. We use available analytical NTK expressions for standard FC\{1,2,3,4,5\} and CONV\{1,2\} architectures in the NTK regime to evaluate and decompose kernels on a subset of 10K MNIST training images and 10K binary CIFAR images - 5K cars and 5K airplanes. We note that within a dataset, the plots do not change much between architectures, speaking to the universal nature of these kernel-induced features.

\begin{figure}[h]
    \centering
    \includegraphics[scale=0.3]{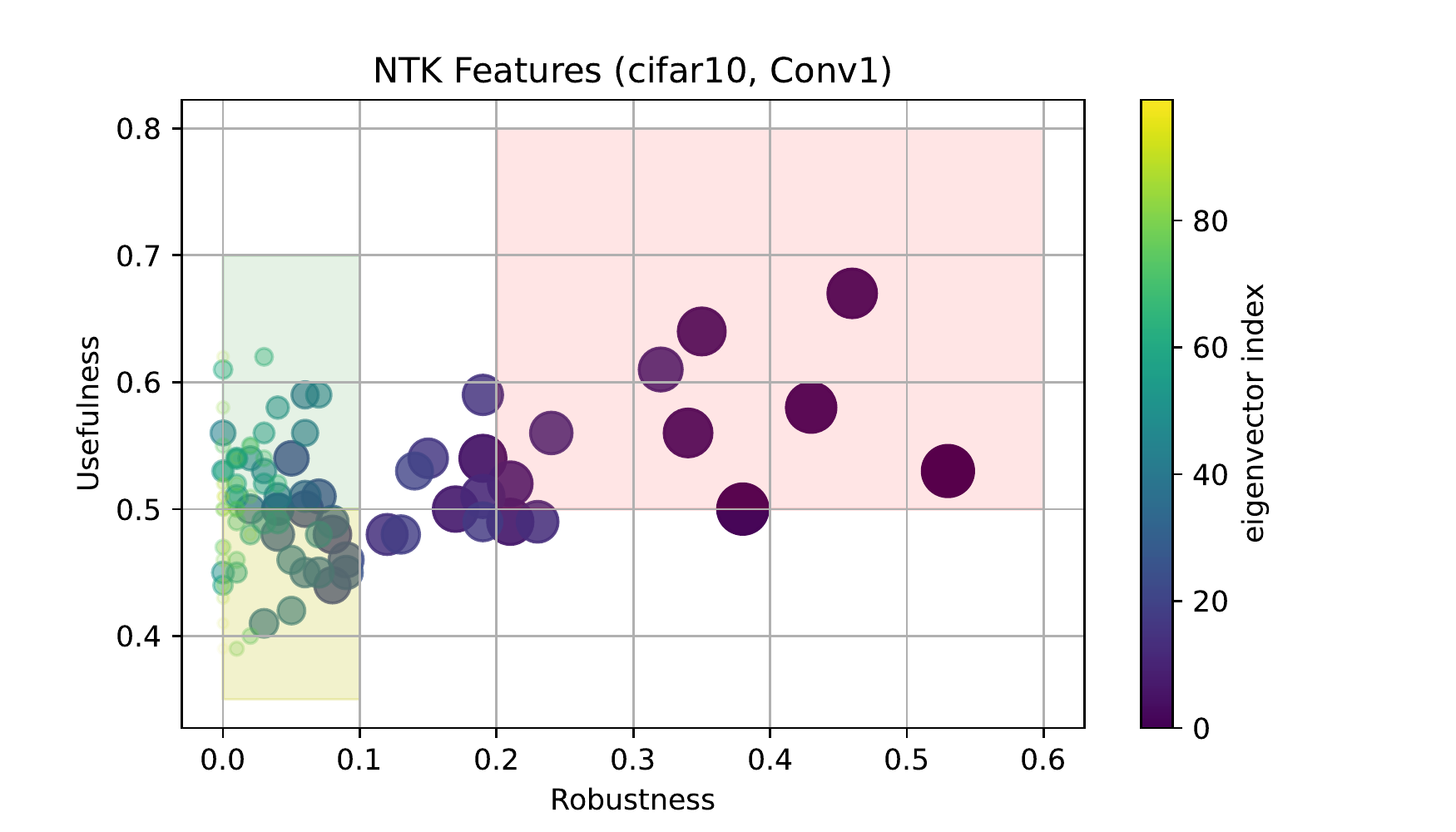}
    \includegraphics[scale=0.3]{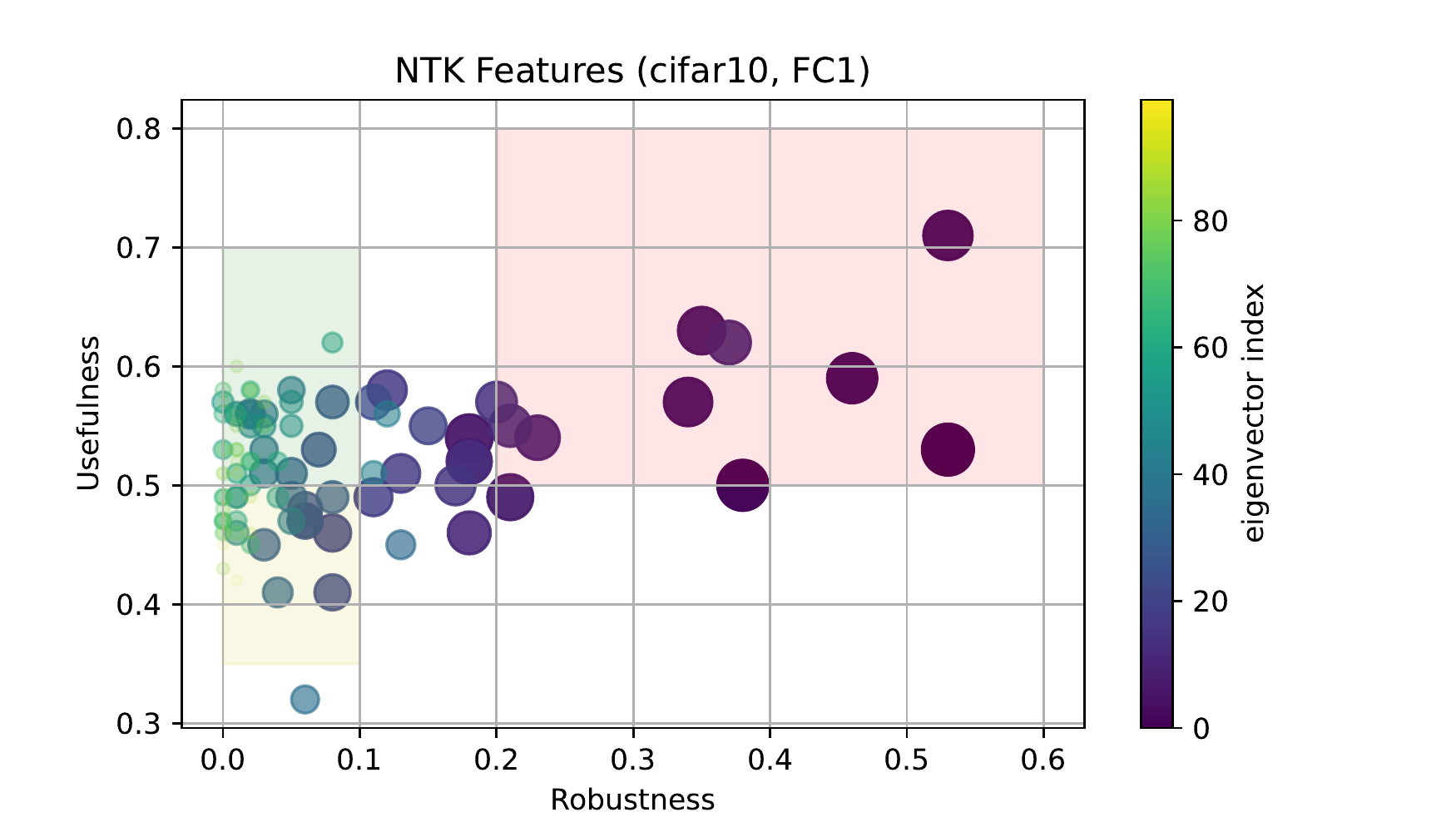}
    \includegraphics[scale=0.3]{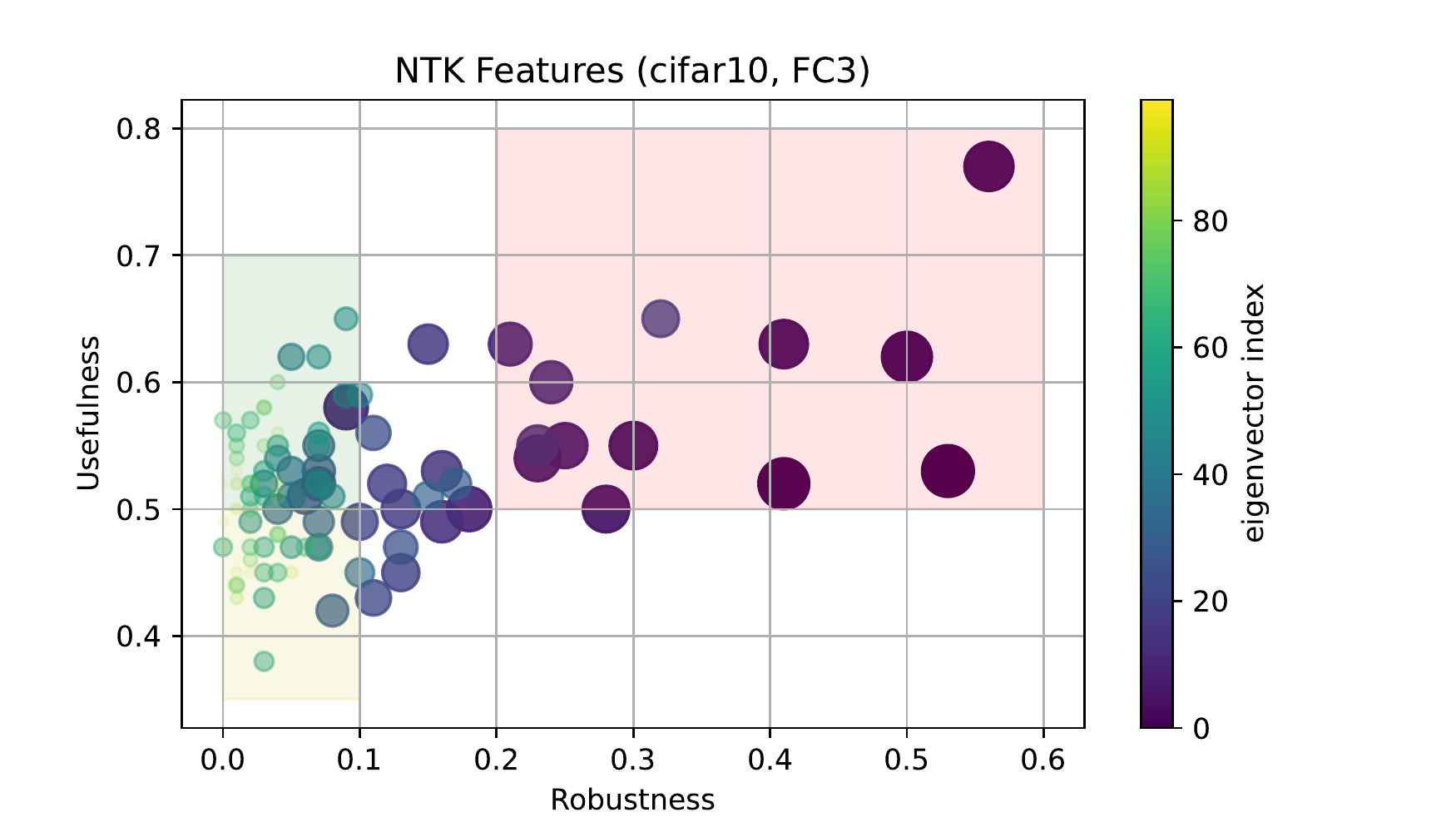}
    \includegraphics[scale=0.3]{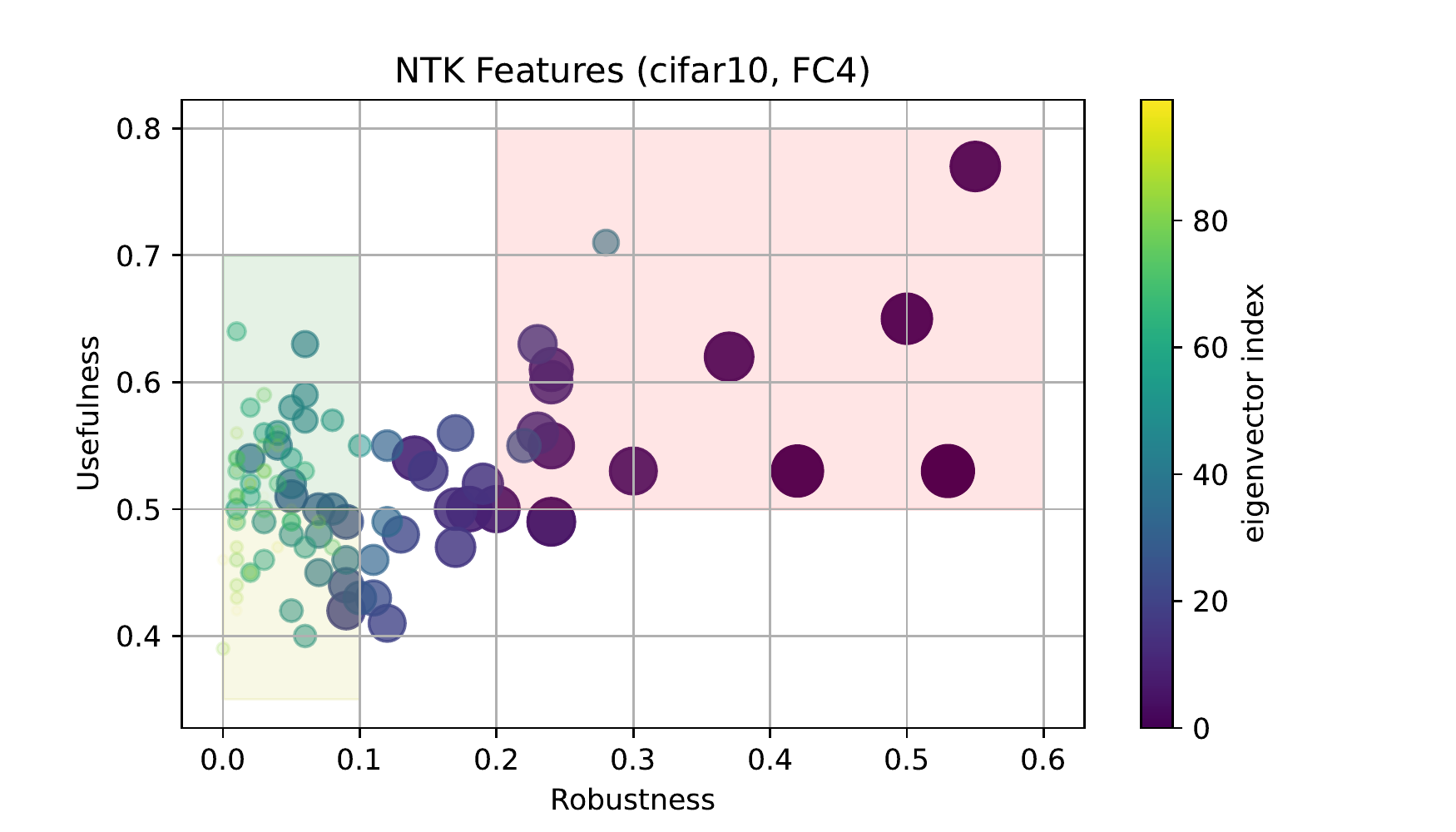}
    \includegraphics[scale=0.3]{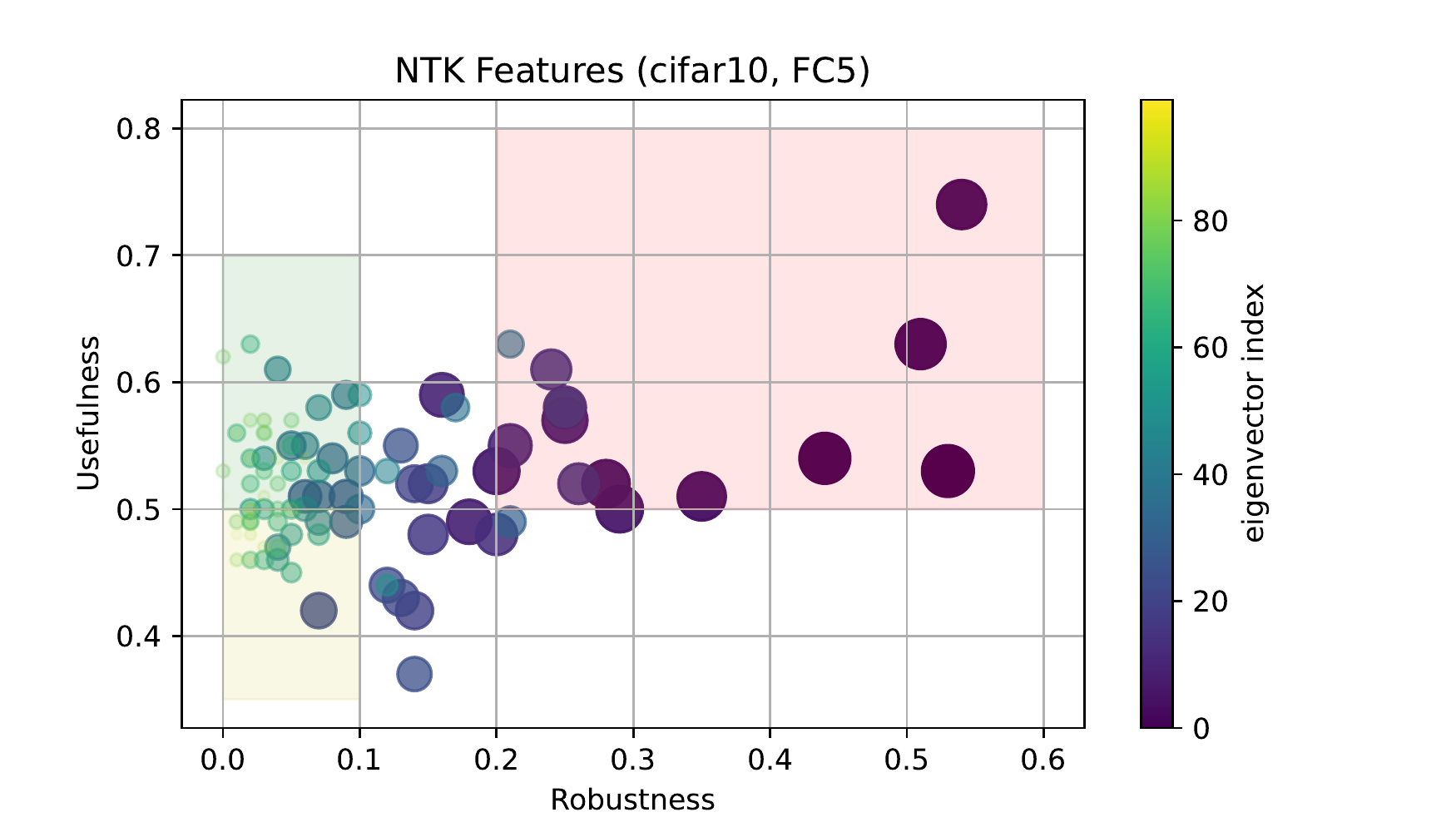}
    \caption{Robustness Usefulness space for various kernels, CIFAR-10 car vs plane. The axes lie in $[0, 1]$. Fig. \ref{fig:features} in the main text shows FC2 and CONV 2. "Useful" features have usefulness $>0.5$ (the random guessing probability for a binary balanced data set). The colored red, green and yellow boxes are arbitrary, meant to visually distinguish useful-robust from other features.}
    \label{fig:rob_vs_usef_cifar10}
\end{figure}

\begin{figure}[h]
    \centering
    \includegraphics[scale=0.3]{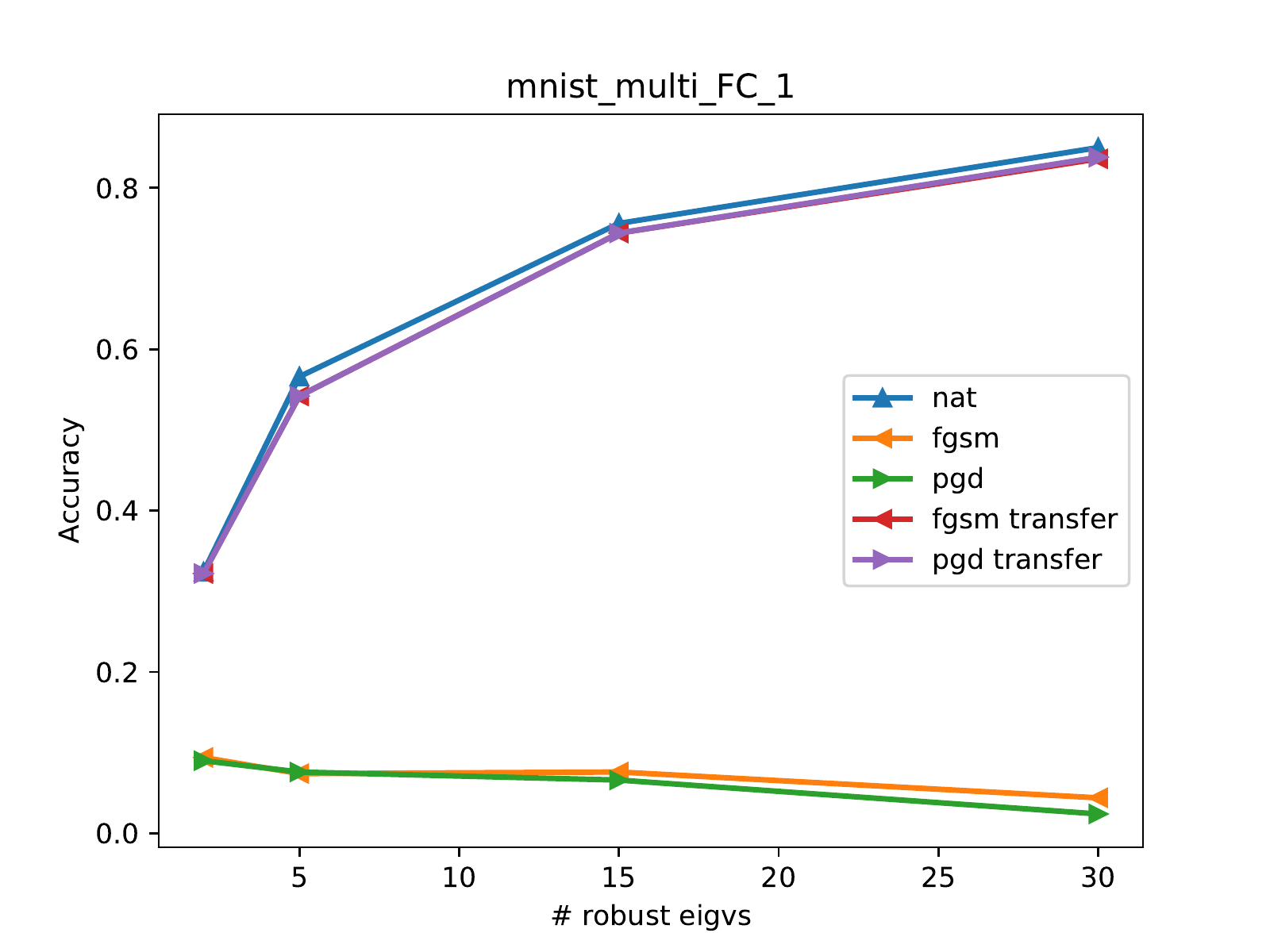}
    \includegraphics[scale=0.3]{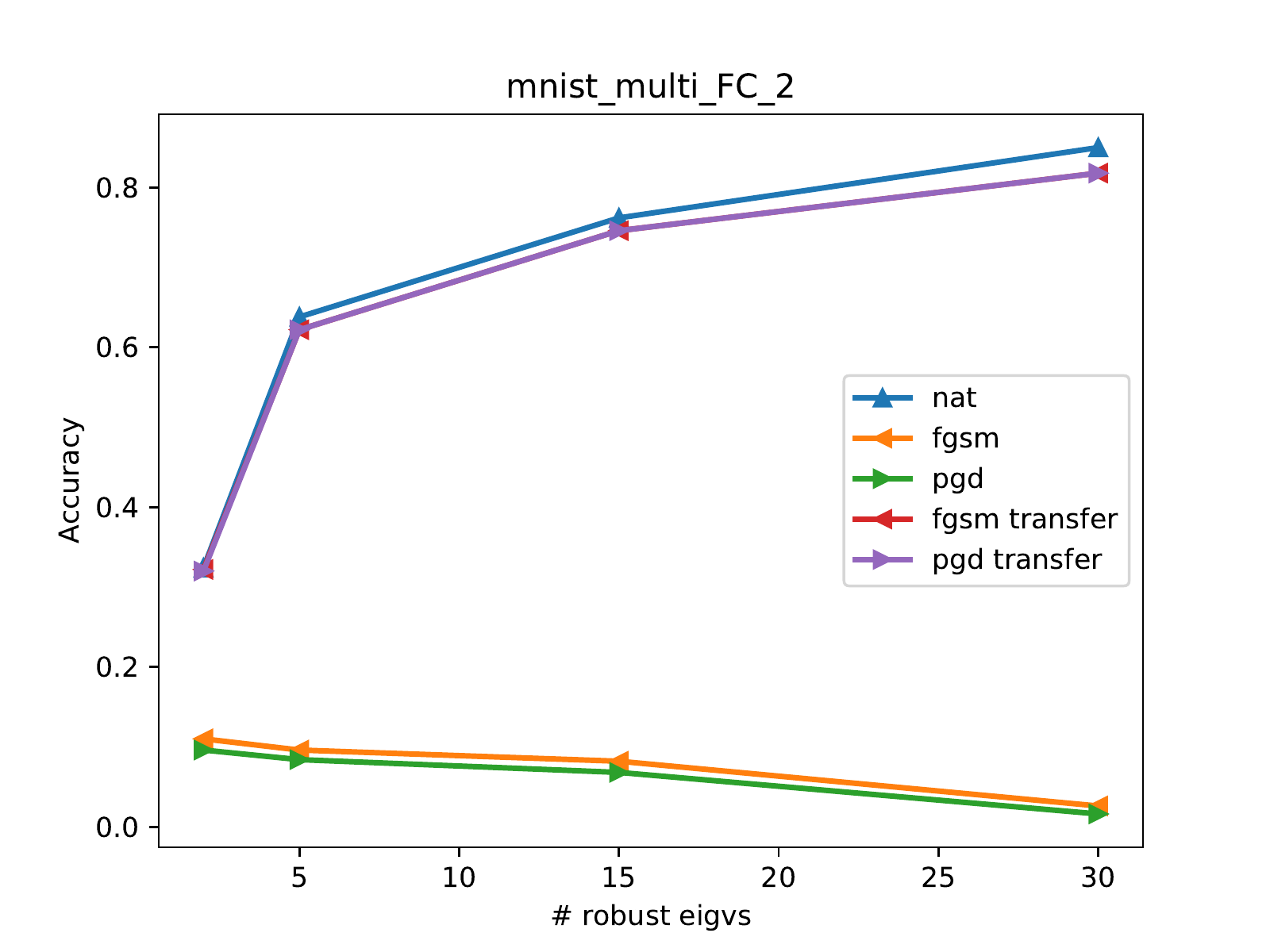}
    \includegraphics[scale=0.3]{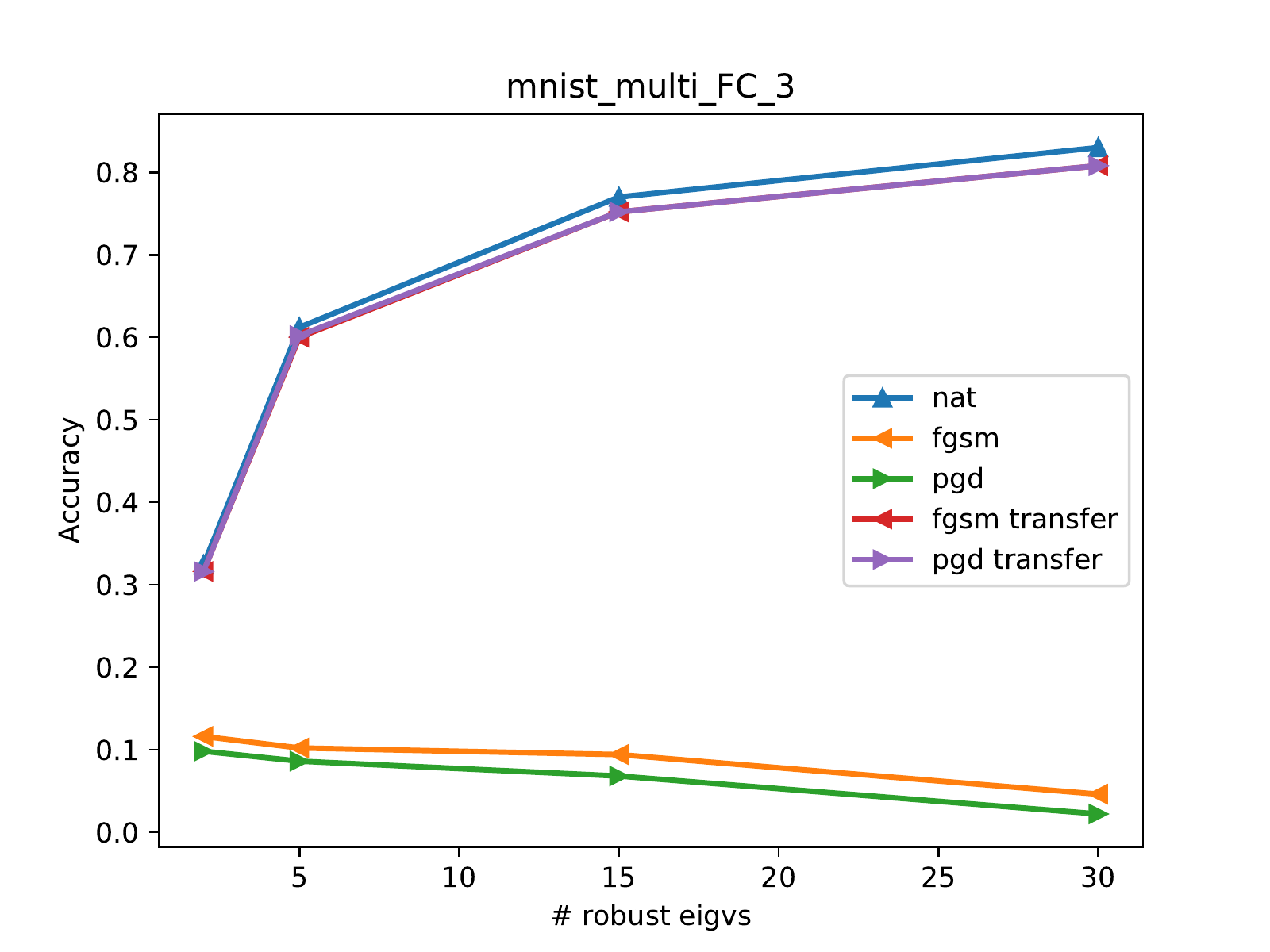}
    \includegraphics[scale=0.3]{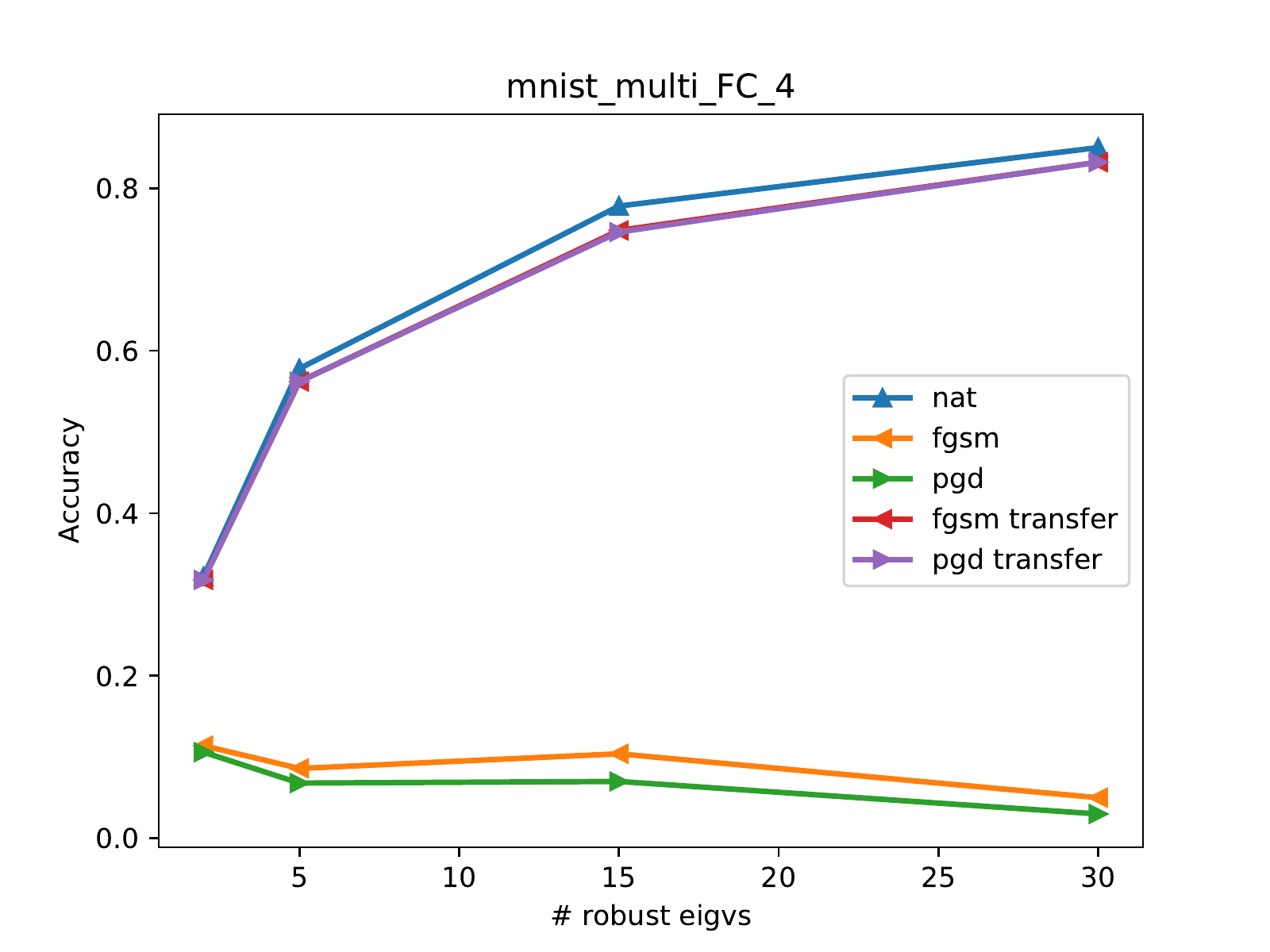}
    \includegraphics[scale=0.3]{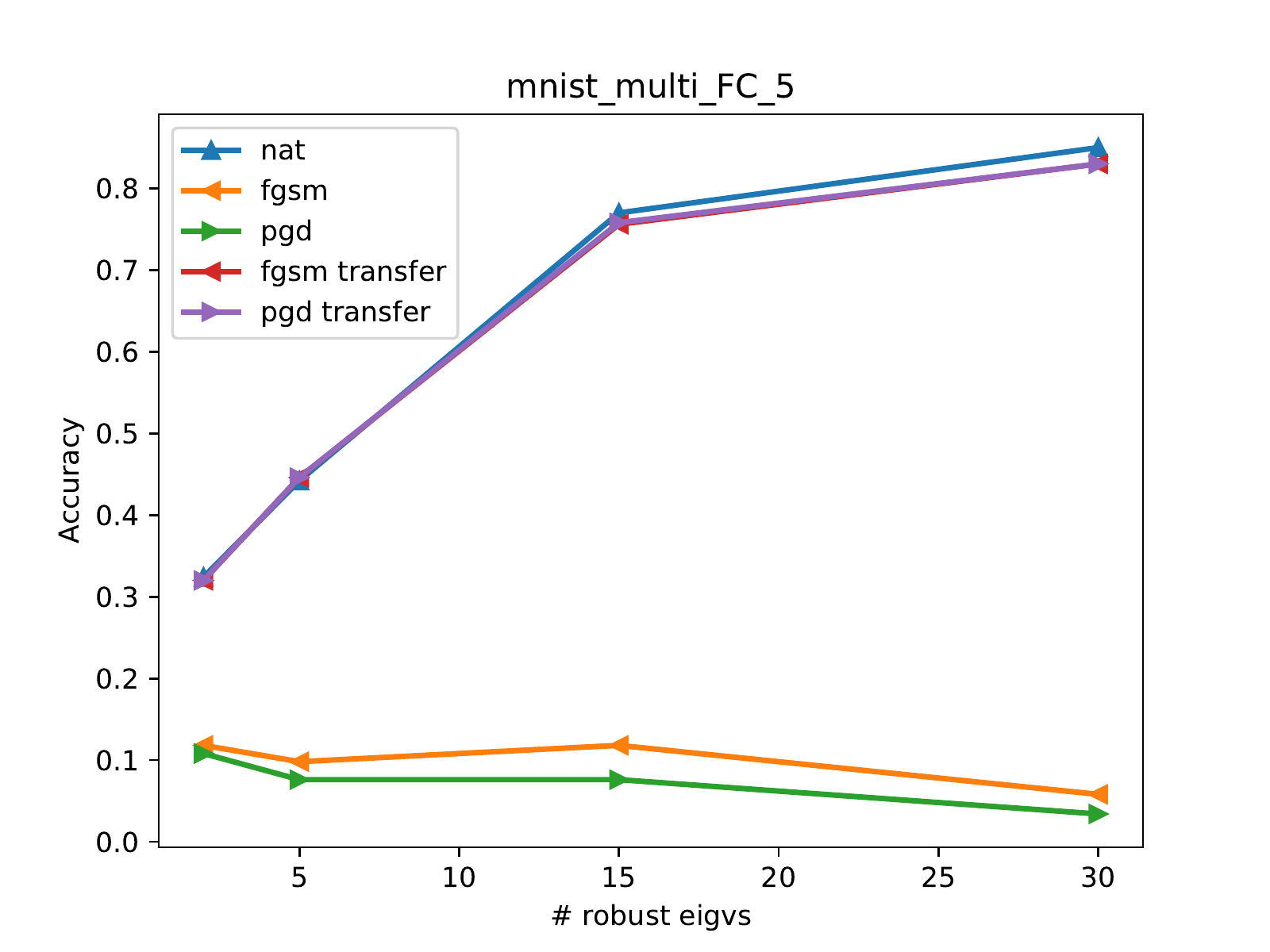}
    \caption{Robustness of kernel when keeping a few of the top robust features (MNIST). The Gram matrix is computed using 10k images from the training set. Blue lines show clean accuracy, red and purple (lines are overlapping) show accuracy against FGSM and PGD10 examples generated using the full kernel machines (consisting of all the features) and orange and green show the resulting robustness of the new model (FGSM and PGD10, respectively). Accuracy on the $y$-axis lies in $[0, 1]$.}
    \label{fig:mnist_ntk_robfeats}
\end{figure}

\subsection{Robust Features Alone are not Enough}
\label{ssec:notenough}

Feature definitions outlined in Sec. 4 open an avenue to use traditional feature selection methods to search for robust models. In particular, here we rank the features of an NTK based on their robustness on a validation set (accuracy against adversarial examples computed from the same feature - setting: FGSM with $\epsilon=0.3$ for MNIST or $\epsilon=8/255$ for CIFAR-10). Specifically, we test and rank each "one-feature kernel" function $f^{(i)} (\mathbf{x}) := \lambda_i^{-1} \Theta(\mathbf{x}, \mathcal{X})^\top \mathbf{v}_i \mathbf{v}_i^\top \mathcal{Y}$. Given this ranking, we construct a sequence of new kernels $\Theta^\prime_r(\mathcal{X},\mathcal{X})$ by progressively aggregating the $r$ most robust features with their original eigenvalues. This gives rise to kernel machines of the form $f^\prime_r (\mathbf{x}) = \Theta(\mathbf{x}, \mathcal{X})^\top \Theta_r^\prime(\mathcal{X}, \mathcal{X})^{-1} \mathcal{Y}$, where $r$ indicates the number of top robust features kept. We present the results of this approach in Figures \ref{fig:mnist_ntk_robfeats} (MNIST) and \ref{fig:cifar10_ntk_robfeats} (CIFAR-10), where we plot clean accuracy as well as robust accuracy against perturbation from the kernel $f^\prime_r$ itself as well as against "transfer" perturbations from the original (full) kernel.

On the binary classification task, some robustness can be garnered by keeping the most robust features and there seems to exist a sweet spot where the robustness is maximized (this seems to be consistent across other models as well).
On multiclass MNIST, however, despite the relative simplicity of the dataset, we are not able to obtain non-trivial performance without compromising robustness. We conclude that it is unlikely that robust features (of standard models) alone are sufficient for robust classification, and the burden of some data augmentation, like in the form of adversarial training, seems necessary, at least for the models considered in our experiments. 

\begin{figure}[h]
    \centering
    \includegraphics[scale=0.3]{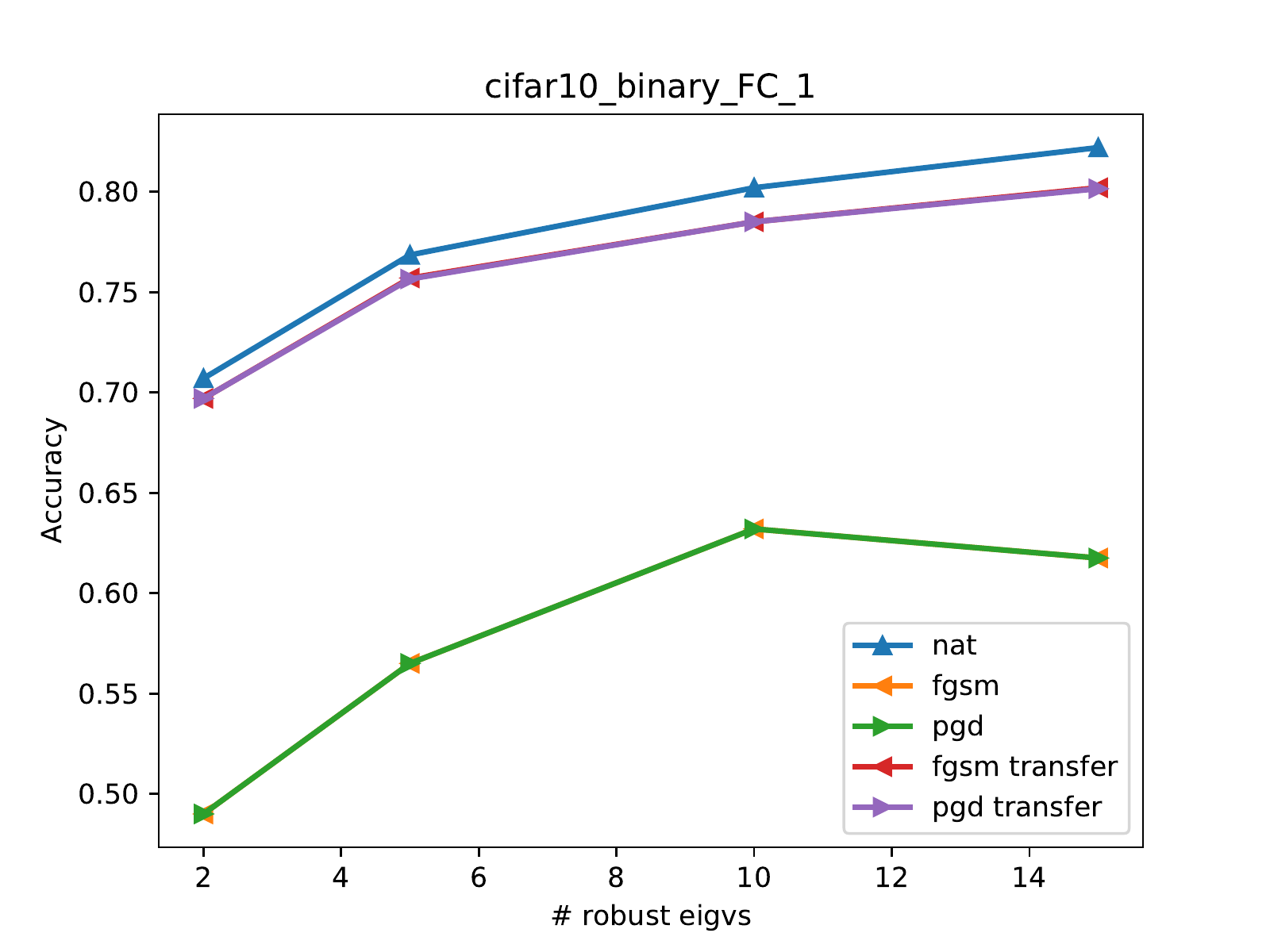}
    \includegraphics[scale=0.3]{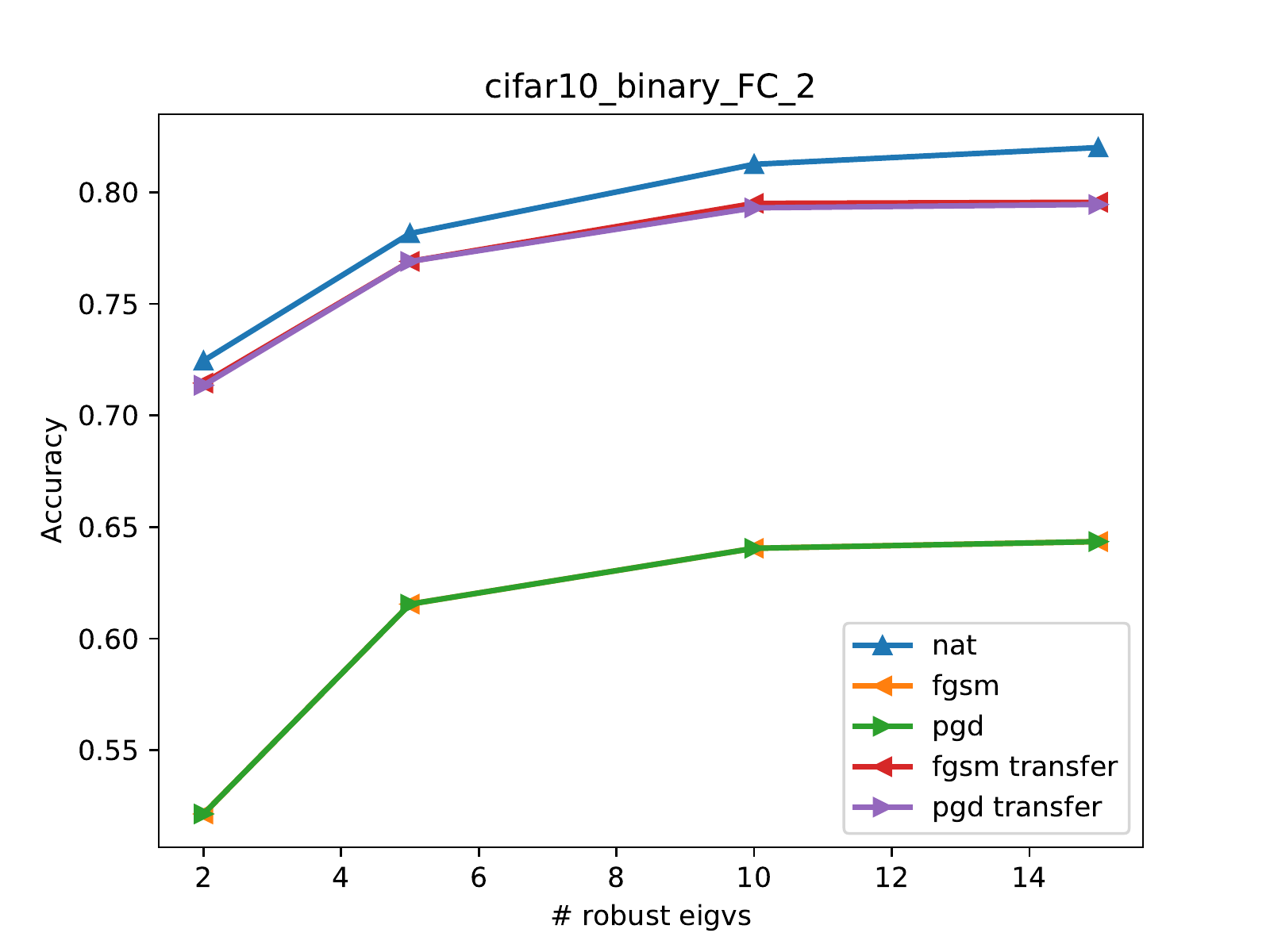}
    \includegraphics[scale=0.3]{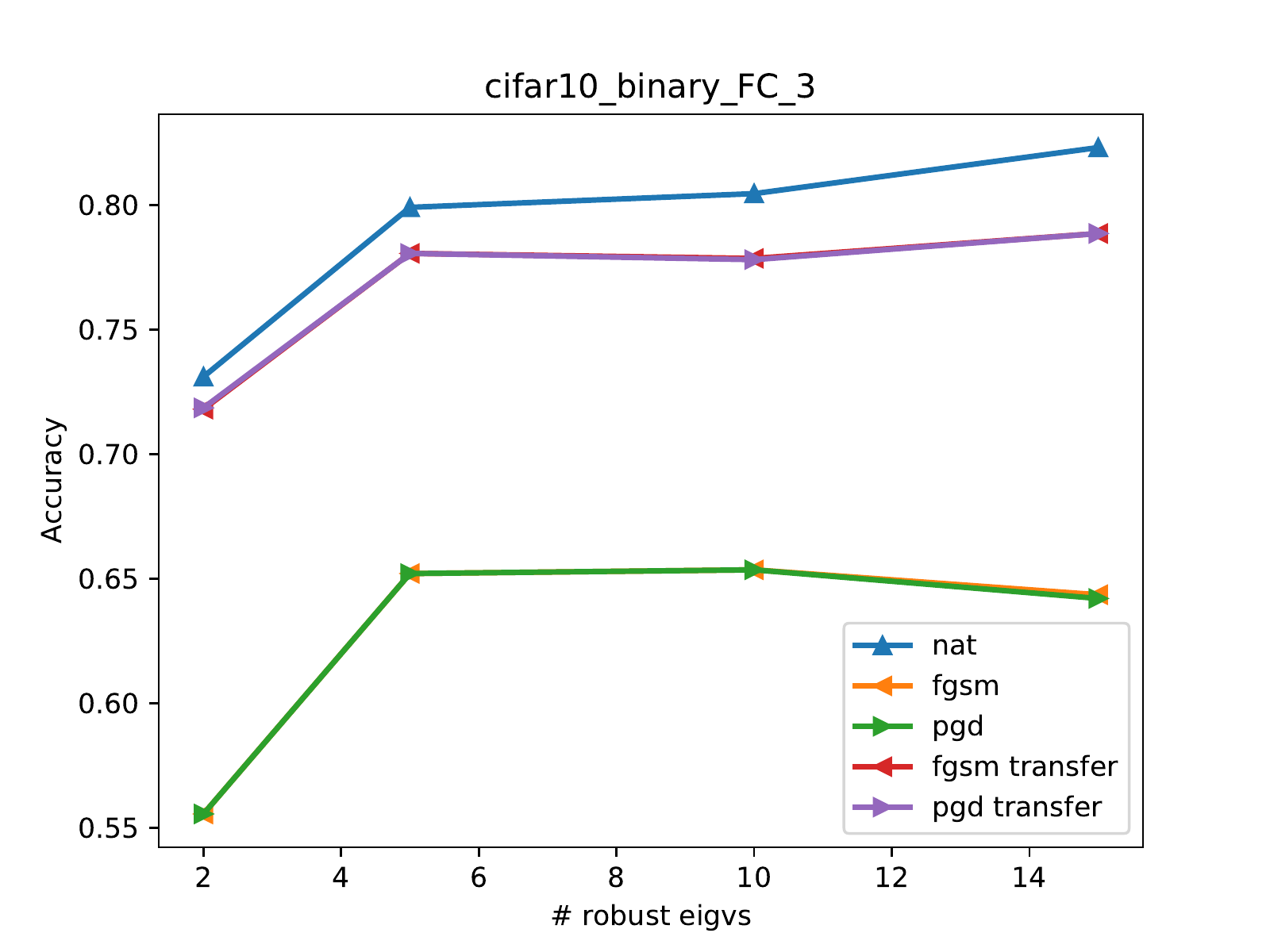}
    \includegraphics[scale=0.3]{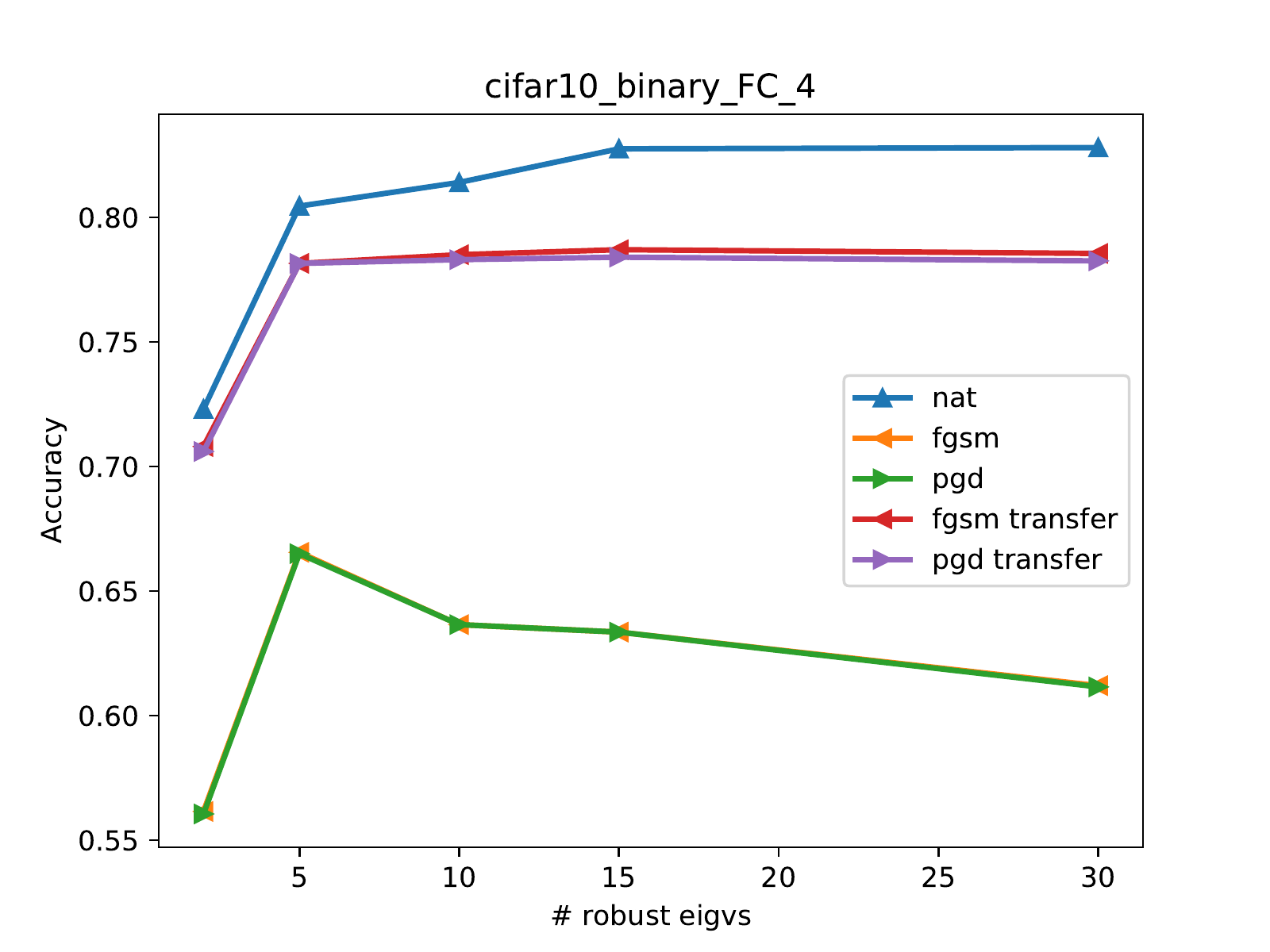}
    \includegraphics[scale=0.3]{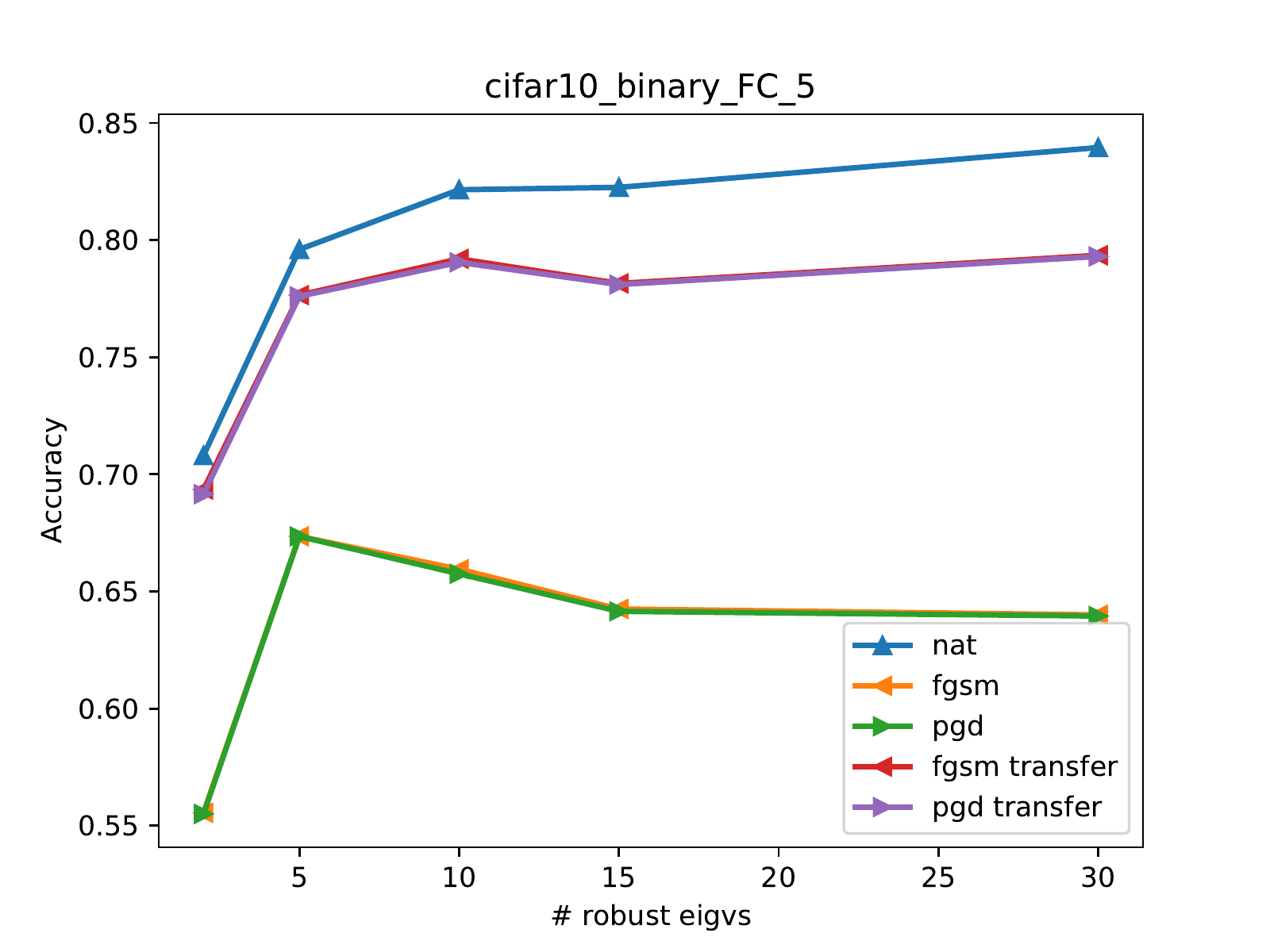}
    \caption{Robustness of kernel keeping a few of the robust features (CIFAR car vs plane). The Gram matrix is computed using all 10K images from the training set. Blue lines show clean accuracy, red and purple (lines are overlapping) show accuracy against FGSM and PGD10 examples generated using the full kernel machines (consisting of all the features) and orange and green show the resulting robustness of the new model (FGSM and PGD10, respectively). Accuracy on the $y$-axis lies in $[0, 1]$.}
    \label{fig:cifar10_ntk_robfeats}
\end{figure}



\section{Experimental Details for the Kernel Dynamics Section}
\label{App:empirical_ntk}

Here we provide the details of our experiments in Sec. 5, where we compare standard and adversarial training by tracking several kernel quantities. 

For experiments with MNIST, we use a simple convolutional architecture with 3 layers. The first 2 layers compute a convolution (with a 3$\times$3 kernel), followed by a ReLU and then by an average pooling layer (of kernel size 2$\times$2 and stride 2). The 3rd layer is fully-connected with a ReLU non-linearity, followed by a linear prediction layer with 10 outputs. The layers have width 32, 64 and 256, respectively.

For CIFAR-10, we use a deeper architecture consisting of 6 layers. Layers 1 and 2, 3 and 4, 5 and 6 are fully convolutional with 32, 64 and 128 channels, respectively, and a kernel of size 3$\times$3. There is a max pooling operation after layer 2, and average pooling after the final layer, followed by a linear prediction layer. Both pooling operations use a kernel of size 2$\times$2 and stride 2.

We use a fixed learning rate of $10^{-2}$ for all experiments and no weight decay. We do not use any data augmentation, since we are interested in analyzing the behavior of kernels, rather than obtaining the best possible results. Stochastic gradient descent is used in all cases, with a batch size of 300 for MNIST and 250 for CIFAR-10. The kernels quantities are tracked for the same (first) batch during training. For adversarial training, we either used FGSM or PGD (for generating the adversarial examples) with 20 steps against $\ell_\infty$ adversaries. The maximum perturbation size is set to $\epsilon=0.3$ and $\epsilon=8/255$ (for MNIST and CIFAR-10, respectively), and in the case of PGD training we use an attack step size of $\alpha=0.1$ and $\alpha=2/255$, respectively. Experiments were run with JAX \citep{Brad+18}, and empirical NTKs were computed using the Neural Tangents Library \citep{Nov+20}. Neural nets were trained using Flax \citep{Hee+20} and the JaxOpt library \citep{Blo+21}, adapting code available from the JaxOpt repository. This code snippet was licensed under the Apache License, Version 2.0.

\begin{figure}[h]
    \centering
    \includegraphics[scale=0.23]{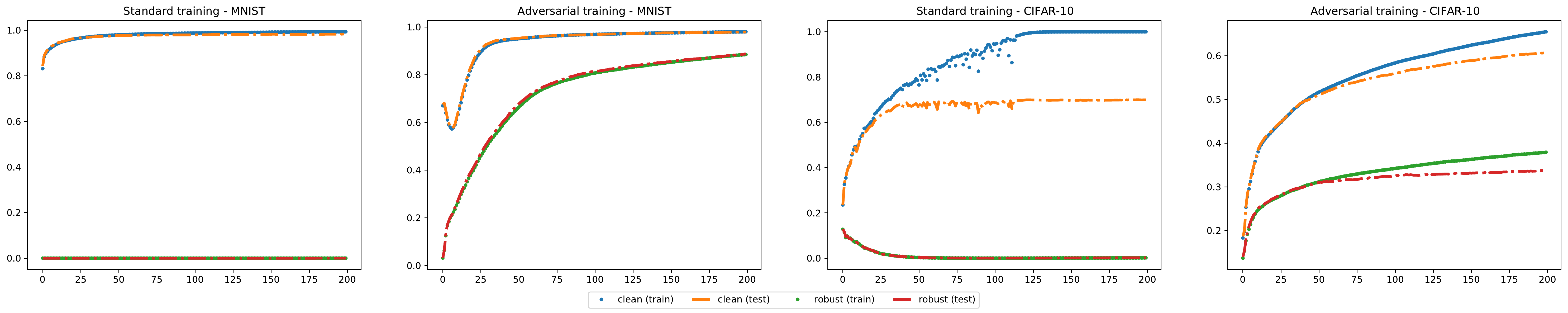}
    \caption{Training curves for networks trained in Sec. 5. From left to right: Standard training on MNIST, Adversarial (PGD-20) training on MNIST, Standard training on CIFAR-10, Adversarial (PGD-20) training on CIFAR-10. For each of the 4 settings, we show train/test accuracy on clean and on adversarially perturbed (PGD-20) data.}
    \label{fig:train_curves}
\end{figure}

Models were trained for 200 epochs. Fig. \ref{fig:train_curves} summarizes the performance of the networks during training. In Fig. \ref{fig:mass_extra}, we show how norm concentration evolves during training - similar to the plots for CIFAR-10 in Fig. \ref{fig:kernels}, but for MNIST and for two choices of eigenvalue index cut-off.

\begin{figure}[h]
    \centering
    \includegraphics[scale=0.3]{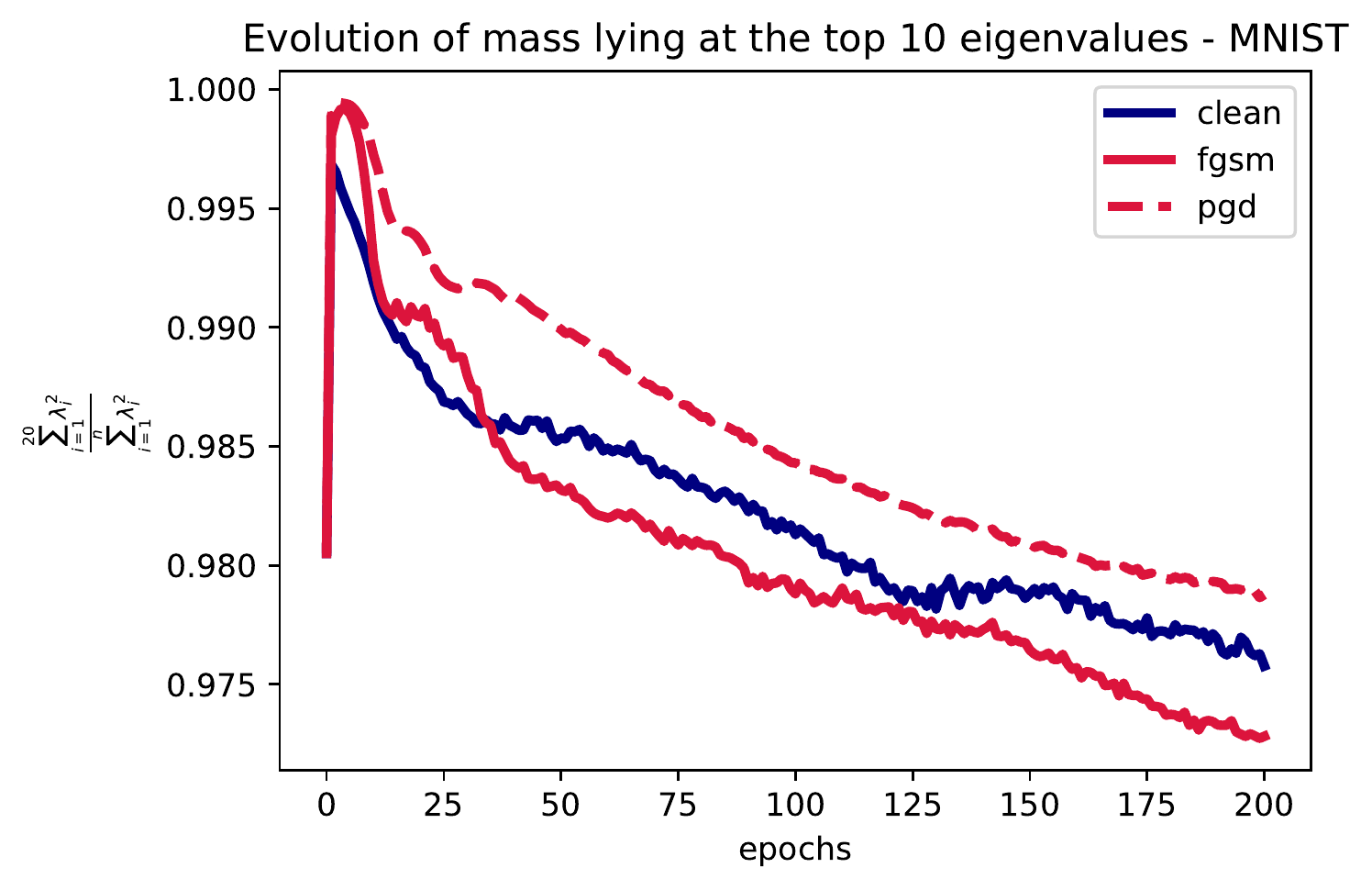}
    \includegraphics[scale=0.3]{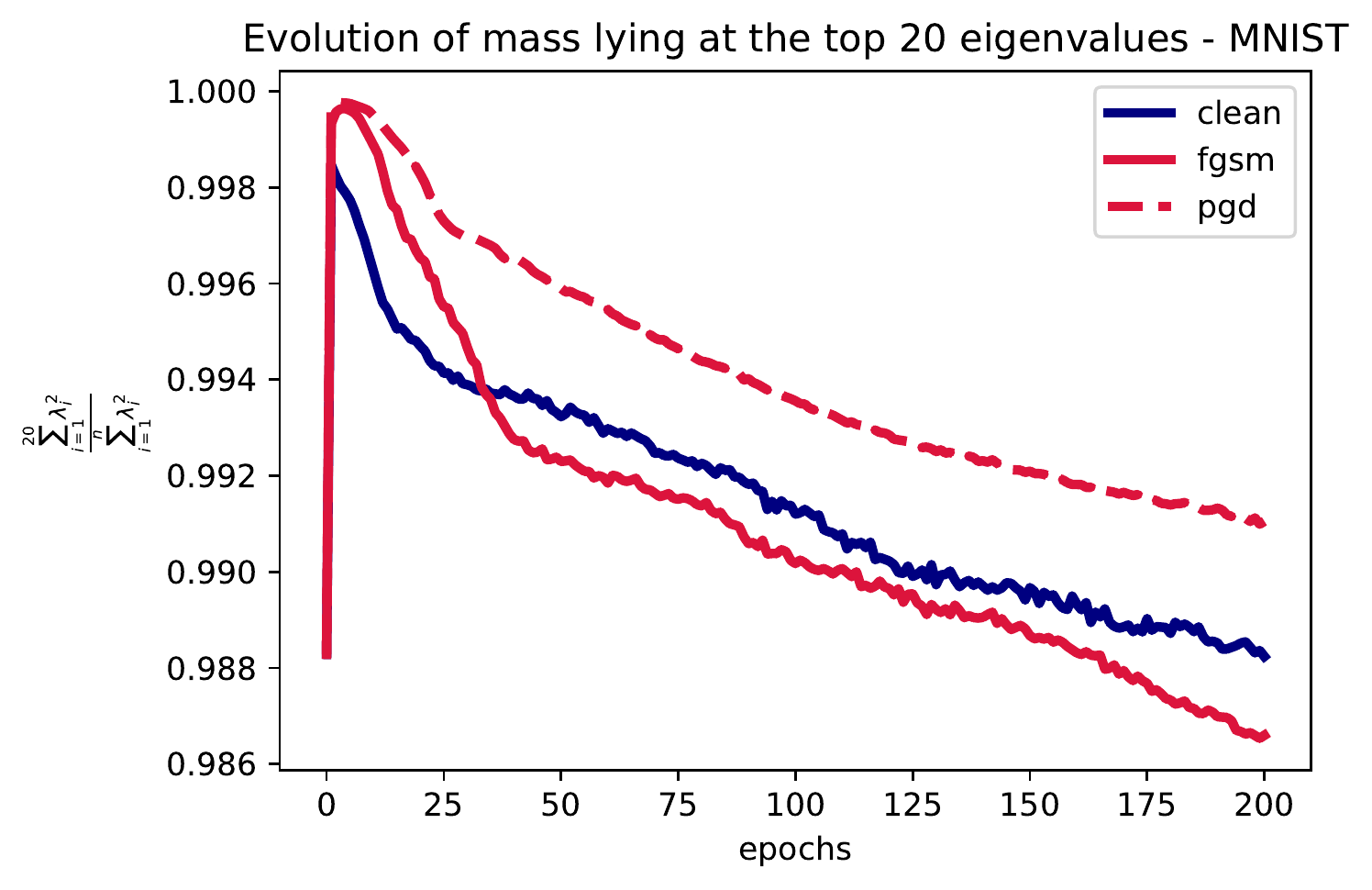}
    \includegraphics[scale=0.3]{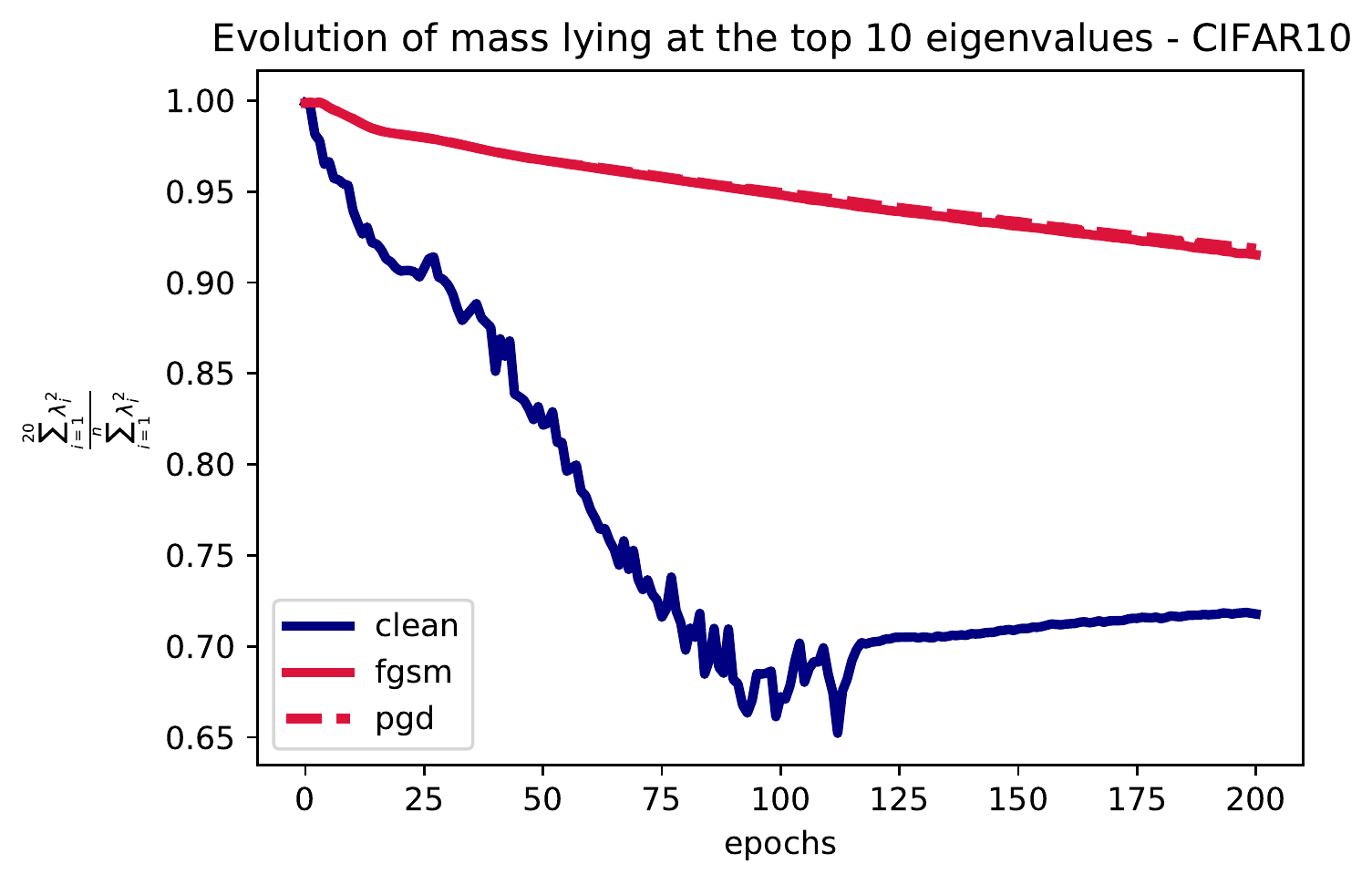}
    \caption{Concentration of norm during standard training vs adversarial training. \textbf{Left:} Concentration on top 10 (MNIST). \textbf{Middle:} Concentration on top 20 (MNIST). \textbf{Right.} Concentration on top 10 (CIFAR-10) (Fig. \ref{fig:kernels} in Sec. 5 shows Concentration on top 20 for CIFAR-10). For MNIST, we observe that when performing adversarial training with just one-step adversary (FGSM), the mass drops below the level of standard training. This is likely related to a phenomenon called catastrophic overfitting which is widespread in simple FGSM training settings \citep{Won+20}.}
    \label{fig:mass_extra}
\end{figure}

Fig.~\ref{fig:top_space_dynamics} shows the polar dynamics for the top space (top 20 eigenvalues) of the kernel. We observe little to no change for adversarial training from Fig. \ref{fig:polarrotation} in the main text that showed the same information for the entire space, though for standard training there is less rotation in the top space. We entertain this as an indication that adversarial training modifies the ``robust'' (top) features of the kernel more than standard training.

\begin{figure}[h]
    \centering
    \includegraphics[scale=0.3]{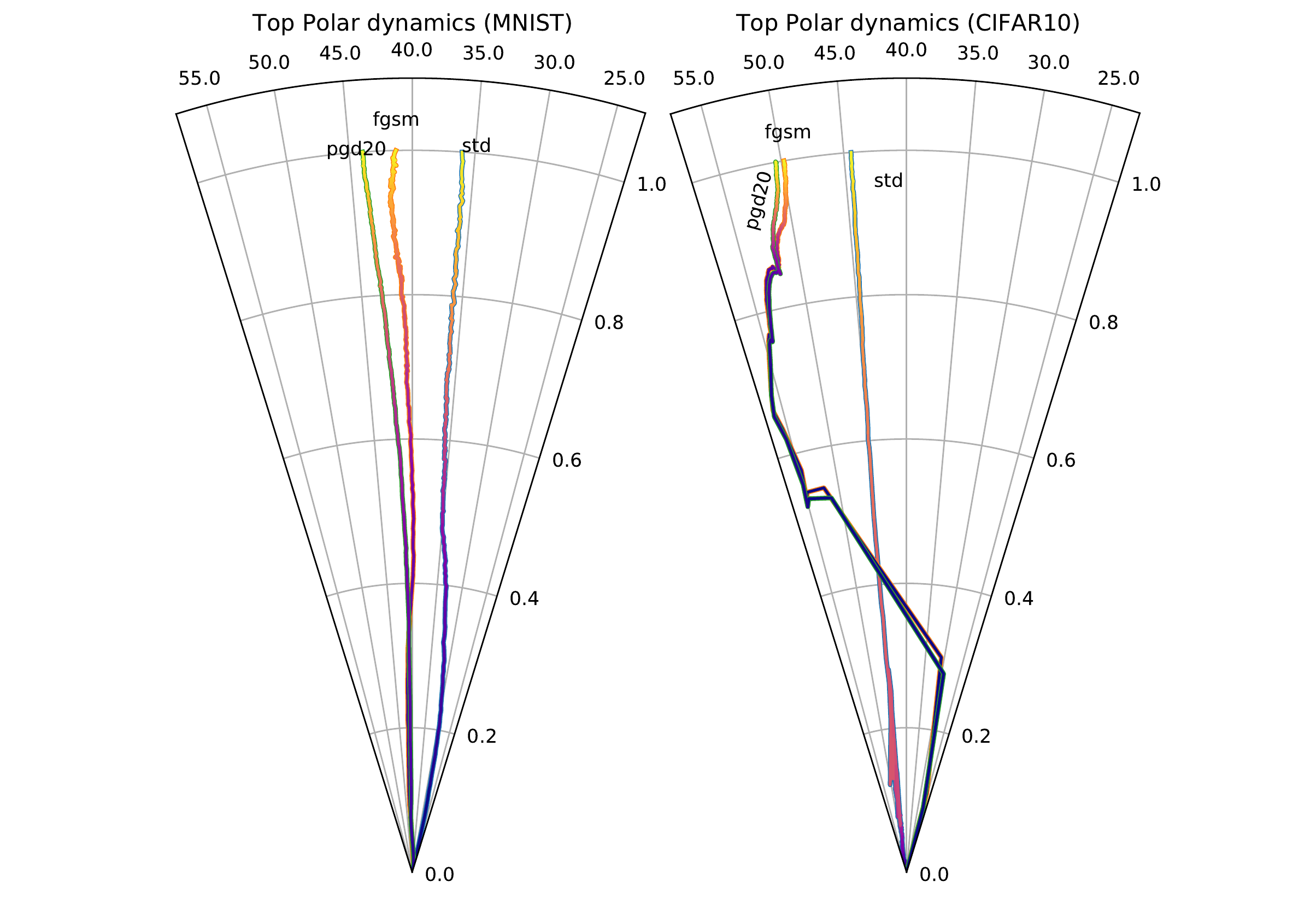}
    \caption{Top-20 dynamics on polar space.}
    \label{fig:top_space_dynamics}
\end{figure}

Finally, Fig. \ref{fig:kernel_images_mnist} shows the values within the kernel matrices before and after training for MNIST for standard and adversarial training. We draw the same conclusions as the main text, namely the ``standard'' kernel has significantly larger values than the ``adversarial'' one.

\begin{figure}[h]
    \centering
    \includegraphics[scale=0.25]{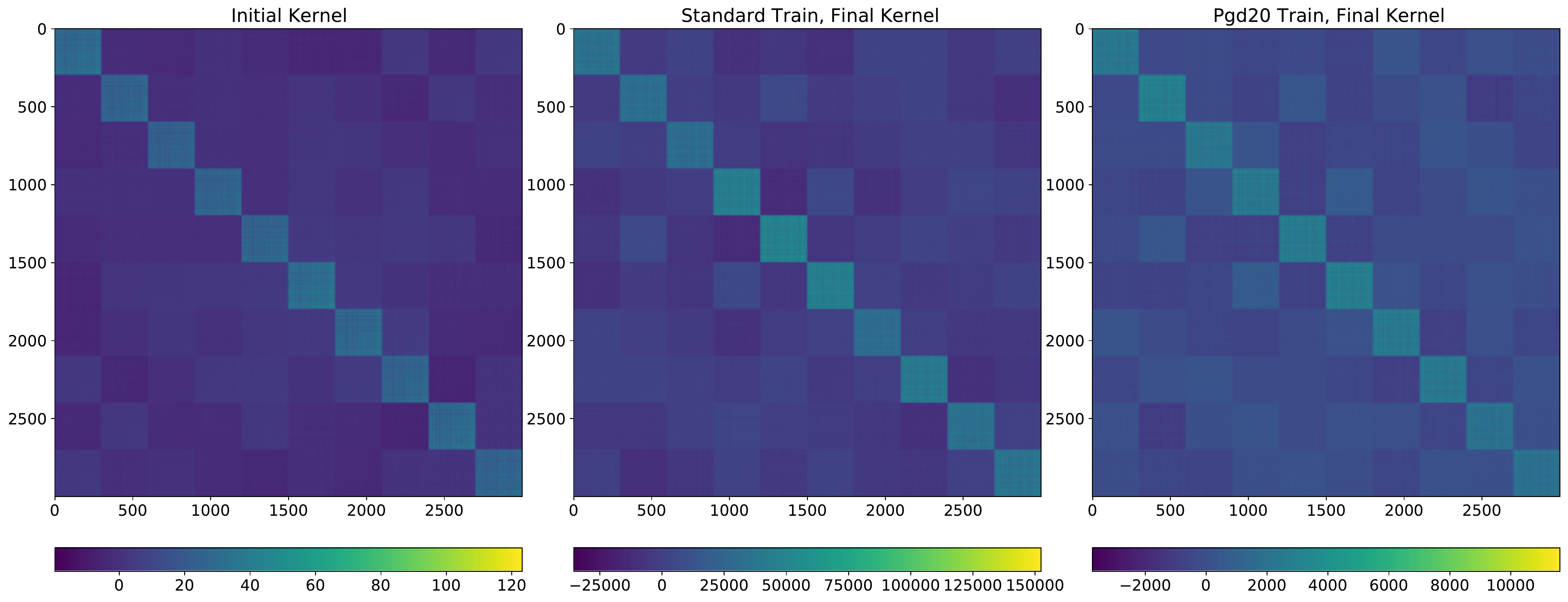}
    \caption{Kernel images for MNIST. Left to Right: Kernel at initialization, Kernel after standard training, Kernel after adversarial training (20 pgd steps). Notice that during training the values increase, but they do substantially more for standard training. Also, observe that for adversarial training there is more spread between different classes. Each little square in the diagonal corresponds to a different class.}
    \label{fig:kernel_images_mnist}
\end{figure}


\subsection{Linearized Adversarial Training}
\label{App:linear_advtrain}

 Motivated by the apparent laziness of the kernel during adversarial training and the findings of prior works \citep{Gei+19,Fort+20} that considered linearization (with respect to the parameters) of the model after some epochs, we do the same for adversarial training.


\begin{figure}[ht]
    \centering
    \includegraphics[scale=0.27]{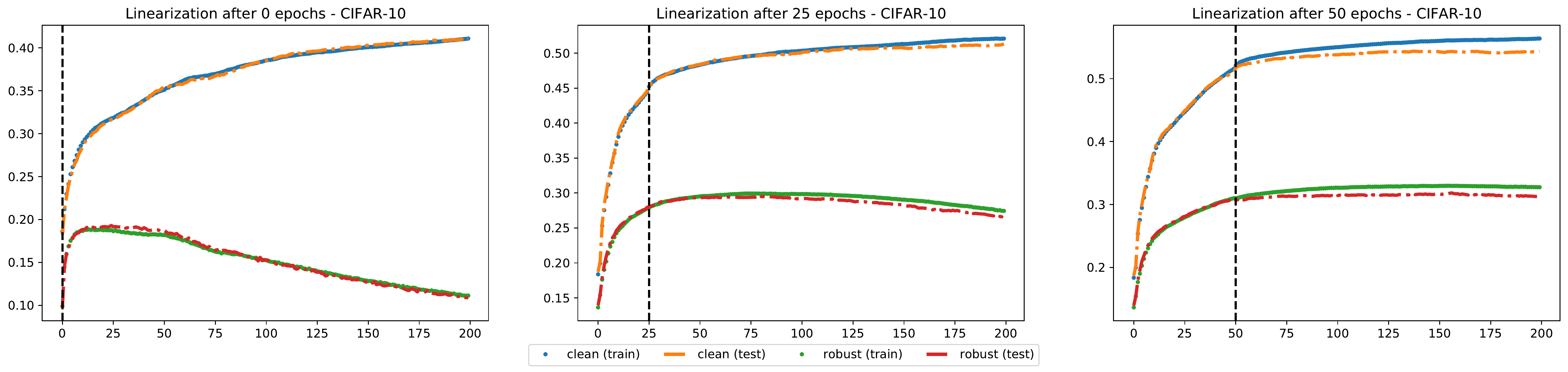}
    \caption{Linearized Adversarial Training on CIFAR-10. \textbf{Left} Linearized after 0 epochs. \textbf{Middle} Linearized after 25 epochs. \textbf{Right} Linearized after 50 epochs. Y-axis has range $(0, 1)$.}
    \label{fig:lin_adv_train}
\end{figure}

We include a small study that linearizes the kernel after a certain number of epochs.
In particular, Fig. \ref{fig:lin_adv_train} shows the training behavior after linearizing the CIFAR-10 model after 25 and 50 epochs, and also at initialization. After linearization, we continue adversarial training in this  simple linearized model (meaning we generate adversarial examples from the linear model). We observe that adversarial training continues, without a collapse of the training method. In comparison to non-linearized training (Fig. \ref{fig:train_curves}), training seems to stagnate. We also observe that the earlier we linearize, the greater the gap is between standard and robust performance. We leave the investigation of this intriguing phenomenon and a detailed comparison to standard training to future work. 

\end{document}